\documentclass{article}
\usepackage[preprint]{preprint}
\usepackage{times}
\usepackage{graphicx}
\usepackage{listings}
\usepackage{subcaption}
\usepackage{csquotes}
\usepackage{amsthm} 
\usepackage{verbatim}
\usepackage{wrapfig}
\usepackage{placeins}

\UseRawInputEncoding

\usepackage{amsmath,amsfonts,bm}

\def\eqref#1{equation~\ref{#1}}

\def\1{\bm{1}}

\DeclareMathAlphabet{\mathsfit}{\encodingdefault}{\sfdefault}{m}{sl}
\SetMathAlphabet{\mathsfit}{bold}{\encodingdefault}{\sfdefault}{bx}{n}

\newcommand{\E}{\mathbb{E}}

\usepackage{xparse}
\usepackage{textcomp}
\usepackage{amsthm}
\usepackage{threeparttable}
\usepackage{pifont}
\usepackage{csquotes}
\usepackage{dashrule}
\usepackage{needspace}
\usepackage[most]{tcolorbox}
\ExplSyntaxOn

\newtheorem{theorem}{Theorem}[section]

\theoremstyle{definition}
\newtheorem{definition}[theorem]{Definition}
\newtheorem{assumption}[theorem]{Assumption}

\theoremstyle{remark}

\prop_new:N \g_model_full_prop
\prop_new:N \g_model_short_prop
\prop_new:N \g_model_api_prop
\bool_new:N \g_model_use_full_bool
\bool_gset_true:N \g_model_use_full_bool 

\NewDocumentCommand{\defmodel}{mmmm}{
  \prop_gput:Nnn \g_model_full_prop  {#1}{#2}
  \prop_gput:Nnn \g_model_short_prop {#1}{#3}
  \prop_gput:Nnn \g_model_api_prop   {#1}{#4}
}

\cs_new:Npn \model_full:n  #1 { \prop_item:Nn \g_model_full_prop  {#1} }
\cs_new:Npn \model_short:n #1 { \prop_item:Nn \g_model_short_prop {#1} }
\cs_new:Npn \model_api:n   #1 { \prop_item:Nn \g_model_api_prop   {#1} }

\NewDocumentCommand{\modelfull}{m}{\model_full:n{#1}}
\NewDocumentCommand{\modelshort}{m}{\model_short:n{#1}}
\NewDocumentCommand{\modelapi}{m}{\texttt{\model_api:n{#1}}}

\NewDocumentCommand{\modellongdefault}{}{\bool_gset_true:N  \g_model_use_full_bool}
\NewDocumentCommand{\modelshortdefault}{}{\bool_gset_false:N \g_model_use_full_bool}
\NewDocumentCommand{\model}{m}{
  \bool_if:NTF \g_model_use_full_bool {\model_full:n{#1}}{\model_short:n{#1}}
}

\newcommand{\tightparagraph}[1]{%
  \textbf{#1}\quad
}

\ExplSyntaxOff

\defmodel{o3}{OpenAI~o3}{O3}{openai/o3}
\defmodel{gpt-4o}{OpenAI~GPT-4o}{GPT-4o}{openai/gpt-4o}
\defmodel{gpt-5}{OpenAI~GPT-4o}{GPT-5}{openai/gpt-5}
\defmodel{o4-mini}{OpenAI~o4~Mini}{O4~Mini}{openai/o4-mini}
\defmodel{o3-mini}{OpenAI~o3~Mini}{O3~Mini}{openai/o3-mini}
\defmodel{gpt-4.1}{OpenAI~GPT-4.1}{GPT-4.1}{openai/gpt-4.1}

\defmodel{llama-4-scout}{Meta~Llama~4~Scout}{Llama~4~Scout}{meta-llama/llama-4-scout}
\defmodel{llama-4-maverick}{Meta~Llama~4~Maverick}{Llama~4~Maverick}{meta-llama/llama-4-maverick}

\defmodel{gemini-2.5-pro}{Google~Gemini~2.5~Pro}{Gemini~2.5~Pro}{google/gemini-2.5-pro}
\defmodel{gemini-2.5-flash}{Google~Gemini~2.5~Flash}{Gemini~2.5~Flash}{google/gemini-2.5-flash}

\defmodel{grok-4}{xAI~Grok~4}{Grok~4}{xai/grok-4}
\defmodel{grok-3-mini}{xAI~Grok~3~Mini}{Grok~3~Mini}{xai/grok-3-mini}

\defmodel{claude-sonnet-4}{Anthropic~Claude~Sonnet~4}{Claude~Sonnet~4}{anthropic/claude-sonnet-4}
\defmodel{claude-opus-4.1}{Anthropic~Claude~Opus~4.1}{Claude~Opus~4.1}{anthropic/claude-opus-4.1}

\defmodel{kimi-k2}{Moonshot~Kimi~K2}{Kimi~K2}{moonshot/kimi-k2}

\defmodel{deepseek-chat-v3}{DeepSeek~Chat~V3}{DeepSeek~Chat~V3}{deepseek/deepseek-chat-v3}
\defmodel{deepseek-r1}{DeepSeek~R1}{DeepSeek~R1}{deepseek/deepseek-r1}

\defmodel{qwen3-235b}{Alibaba~Qwen~3~235B}{Qwen~3~235B}{qwen/qwen-3-235b}

\definecolor{myblue}{HTML}{6397C7}
\definecolor{myyellow}{HTML}{C79363}
\newtcolorbox{caveatbox}[2][]{
    enhanced,
    halign title=flush left,
    left*=0pt, right*=0pt,
    boxsep=2pt, left=5pt, right=5pt,
    skin first=enhanced,
    skin middle=enhanced,
    skin last=enhanced,
    colframe = myblue!100,
    colback  = myblue!10,
    fonttitle=\fontfamily{ppl}\selectfont\bfseries, 
    title={\strut #2},
    #1
}

\newtcolorbox{remarkbox}[2][]{
    enhanced,
    parbox=false,
    halign title=flush left,
    left*=0pt, right*=0pt,
    boxsep=2pt, left=5pt, right=5pt,
    skin first=enhanced,
    skin middle=enhanced,
    skin last=enhanced,
    colframe = myyellow!100,
    colback  = myyellow!10,
    fonttitle=\fontfamily{ppl}\selectfont\bfseries, 
    title={\strut #2},
    #1
}

\usepackage{hyperref}
\usepackage{makecell}
\usepackage{xurl}
\usepackage{booktabs}
\usepackage[capitalise]{cleveref}

\crefname{assumption}{assumption}{assumptions}
\Crefname{assumption}{Assumption}{Assumptions}

\newtheorem{fact}{Fact} 
\usepackage{algorithm}
\usepackage{algorithmic}
\usepackage{amsmath, amssymb}
\usepackage[table,dvipsnames]{xcolor}
\usepackage[most]{tcolorbox}

\newtcolorbox{promptbox}[1]{%
  colframe=black,
  colback=gray!5,
  coltitle=white,
  colbacktitle=black,
  fonttitle=\ttfamily\bfseries,
  title=#1,
  boxrule=0.5mm,
  enhanced jigsaw,
  sharp corners,
  breakable,
  enhanced,
  fontupper=\ttfamily
}

\usepackage{makecell}
\usepackage{multirow}
\usepackage{enumitem}

\newcommand{\prophet}{\texttt{Prophet Arena}}

\newcommand{\titleprophet}{  Prophet Arena \,\raisebox{-0.05em}{\includegraphics[height=.8em]{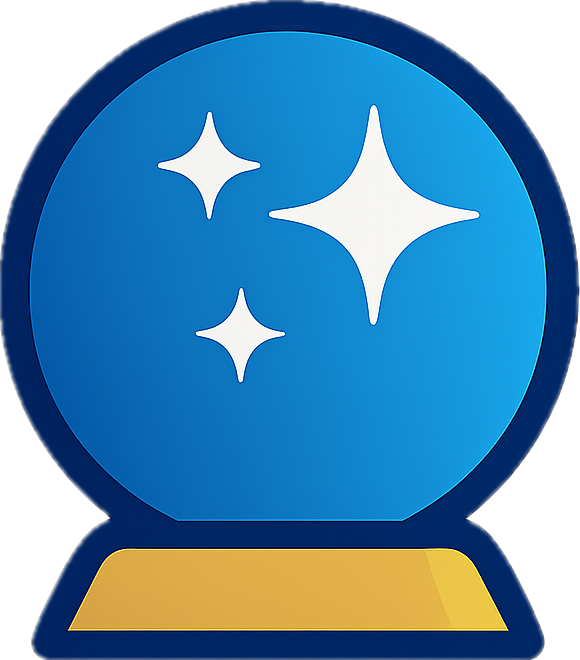}}}

\definecolor{lightblue}{RGB}{220,230,250}  
\definecolor{darkgreen}{RGB}{44,95,52}

\hypersetup{
  colorlinks = true,
  linkcolor  = MidnightBlue,
  citecolor  = ForestGreen,
  urlcolor   = BrickRed
}

\crefname{section}{\S\!}{\S\!}
\crefname{subsection}{\S\!}{\S\!}

\title{
LLM-as-a-Prophet: Understanding Predictive Intelligence with \titleprophet \,
\thanks{Correspondence to \texttt{contact@prophetarena.co}. \\ Part of the work by Qingchuan Yang and Jibang Wu are done while they are at the University of Chicago. We thank Alex Gu, Chaplin Huang and Lucien Liu for help at the early stage of this work.}
}

\newcommand{\affmark}[1]{\textsuperscript{#1}} 

\author{%
Qingchuan Yang \affmark{1}\thanks{Equal contribution.},\;
Simon Mahns \affmark{2}\footnotemark[2],\;
Sida Li \affmark{3}\footnotemark[2],\;
Anri Gu \affmark{3}\footnotemark[2],\;
Jibang Wu \affmark{4}\thanks{Equal advising.},\;
Haifeng Xu \affmark{3}\footnotemark[3]
\\[0.6ex]
\small
\affmark{1}\,University of Southern California \quad
\affmark{2}\,Meta \quad
\affmark{3}\,The University of Chicago \quad
\affmark{4}\,New York University
}

\begin{document}

\maketitle

\begin{abstract}
 Forecasting is not only a fundamental intellectual pursuit  but also is of significant importance to societal systems such as finance and economics. With the rapid advances of large language models (LLMs) trained on  Internet-scale data, it raises the promise of employing LLMs to forecast real-world  future events, an emerging paradigm we call ``LLM-as-a-Prophet''. This paper  systematically investigates such  predictive intelligence of   LLMs. To this end, we build \prophet, a general evaluation benchmark that continuously collects live forecasting tasks and decomposes each task into distinct pipeline stages, in order to support  our controlled and large-scale experimentation. Our comprehensive evaluation reveals that many LLMs already exhibit impressive forecasting capabilities, reflected in, e.g., their small calibration errors, consistent prediction confidence and promising market returns. However, we also uncover key bottlenecks towards achieving superior predictive intelligence via  LLM-as-a-Prophet, such as LLMs' inaccurate event recalls, misunderstanding of data sources and slower information aggregation compared to markets when resolution nears.
\end{abstract}

\section{Introduction}

Forecasting is a fundamental intellectual pursuit that has shaped human progress from the earliest scientific inquiries  to modern economics and finance.  In machine learning, forecasting has also been a central theme, with rich traditions ranging from time-series analysis \citep{box1976time,lim2021time}, online learning \citep{foster1999regret}, to conformal prediction \citep{barber2023conformal}. Yet, somewhat surprisingly, the challenge of open-domain forecasting, producing accurate predictions across a wide range of topics without domain-specific tuning or specialized datasets, remains largely unexplored. Achieving reliable foresight in this setting would represent a qualitative leap in AI capability, with far-reaching societal implications, from enhancing market efficiency to guiding high-stakes policy decisions \citep{arrow2008promise}. 

At its core, forecasting is the process of connecting present knowledge to anticipate future outcomes. Large language models (LLMs) seem natural candidates for this role. Trained on massive corpora of human knowledge through the seemingly narrow objective of next-word prediction, LLMs have developed emergent capabilities that extend far beyond their training objective \citep{bubeck2023sparks}. This motivates the prospect that the ability to predict the next word may also give rise to the ability to predict the next event. Indeed,  recent works by \citet{zou2022forecasting,halawi2024approaching} have already made promising progress in employing language models for forecasting tasks.  We envision that, with  growing interest and continued progress, this would position LLMs not only as repositories of human knowledge but also as instruments of reliable foresight, leading to the prospect of \textbf{LLM-as-a-Prophet}: \textit{Can AI systems reliably predict the future by connecting the dots across existing real-world information?}

In this paper, we seek to systematically examine the prospects and challenges of building general-purpose systems for open-domain forecasting. On one hand, forecasting is a natural next pursuit given the rapid progress of AI, as it draws on a combination of advanced capabilities that current models are only beginning to demonstrate: information retrieval, complex reasoning and data analysis. Moreover, as many established benchmarks are approaching saturation and are increasingly prone to training-data contamination \citep{deng2024investigating}, open-domain forecasting provides a forward-looking and contamination-free setting with objectively measurable outcomes, making it a rigorous testbed for evaluating advanced model intelligence \citep{zou2022forecasting,karger2024forecastbench}. On the other hand, we observe that current LLMs often struggle with key requirements for reliable foresight, including calibrated uncertainty estimation~\citep{geng2023survey} and robust reasoning~\citep{zhou2024can} in the presence of noisy or incomplete evidence. As a result, their forecasting results may at times resemble guesswork rather than deliberated prediction, raising the possibility that fundamental barriers must be addressed before such evaluations can serve as a meaningful benchmark at the present time~\citep{paleka2025evaluating}. Toward this end, we introduce \prophet, a general framework for evaluating LLMs on live, real-world forecasting questions in a controlled and extensible way. Our goal is not only to assess the current forecasting performance of LLMs, but also to use forecasting as a lens for studying core components of intelligence, including reasoning, calibration, evidence aggregation. By doing so, we aim to identify which capabilities are emerging, which remain limited, and how forecasting evaluation can guide the development of more reliable predictive intelligence.

\textbf{Organization of this paper.} In~\cref{sec:arena}, we present our design of \prophet, a live and extensible benchmark for studying predictive intelligence. We define the key concepts and notation used throughout, detail the modular forecasting pipeline -- from event and market extraction to probabilistic prediction and evaluation -- and discuss the design principles that differentiate \prophet\ from prior forecasting benchmarks. In~\cref{sec:evals}, we introduce our three primary evaluation dimensions with their formal metrics and practical motivations, providing the foundation for rigorous and multi-faceted model assessment. In~\cref{sec:exps}, we conduct an in-depth examination of LLM-as-a-Prophet, including mechanistic studies on knowledge internalization, context construction, and reasoning synthesis. We conclude in~\cref{sec:conclusion} with a synthesis of key findings, limitations, and future research directions toward more reliable predictive intelligence.

\textbf{Summary of Our Contributions. }
We summarize the main contributions of this work as follows:
\begin{itemize}[leftmargin=2em]
    \item We introduce the notion of \emph{LLM-as-a-Prophet}, framing open-domain forecasting of real-world events as the next grand challenge of language models' intelligence.  
    \item We develop \prophet, a live and extensible benchmark that continuously collects live real-world forecasting tasks across diverse domains. The framework decomposes forecasting into distinct stages in order to support controlled evaluations and incorporates  multiple scoring metrics, capturing statistical accuracy, calibration, and economic value, in order to provide a comprehensive view of forecasts' quality.
    \item We conduct thorough evaluations on state-of-the-art LLMs across 1300+ resolved real-world events, and perform comprehensive experiments to highlight several caveats in LLM forecasting that may have been overlooked in previous studies. Our results   reveal both emerging strengths -- such as calibrated uncertainty and reasoning alignment -- as well as persistent bottlenecks in information aggregation and foresight near event resolution. 
\end{itemize}

\subsection{Connection and Comparison to Previous Works. } 

\textbf{Understanding and advancing LLMs' forecasting capabilities. } A key goal of our work is to understand the novel paradigm of \emph{LLM-as-a-Prophet} and analyze how different capabilities (e.g., knowledge internalization, source usage, etc.) affect  LLMs'   predictive intelligence. This research goal shares similarities to a few recent works on understanding and diagnosing \emph{special aspects} of LLMs' forecasting capabilities. For instance, \citet{daillms,chenghaozhu2025your} study LLMs'  \emph{temporal generalization} capability by challenging LLMs  to forecast future events curated from news articles.  Both works show that LLMs' forecasting accuracy degrades over time, even when they are armed with retrieval augmented generation (RAG). \citet{paleka2025consistency} studies whether LLMs can make consistent forecasts; for example, a logical AI should not predict  that both the Democratic and Republican parties have a 60\% chance of winning the 2024 US presidential election. Towards that end, they  build a   proper-scoring-rule forecasting benchmark to measure   the consistency of LLMs' predictions. Similar to both works' research methodology, we also build a  benchmark platform \prophet\ to study our research question.  However, to our knowledge, our work is the first to investigate general predictive intelligence within the LLM-as-a-Prophet paradigm. Finally, we also note the  recent line of works advancing the forecasting capabilities of AI systems using language models \citep{zou2022forecasting,halawi2024approaching}. Though our goal of benchmarking is different, we envision that insights from our analysis about the strengths and limitations of current LLMs for forecasting could  benefit future research for advancing AI's forecasting capabilities.

\noindent \textbf{Forecasting Benchmarks. }  Forecasting has recently become a popular challenge for benchmarking LLMs' capabilities. Besides being a real challenge to LLMs, 
it also avoids the thorny issue of benchmark contamination due to the evaluating LLMs on ``future'' events \citep{daillms,karger2024forecastbench}. To  our knowledge,  \citep{jin2021forecastqa} is perhaps one of the earliest to test such forecasting capabilities of language models. They introduced ForecastQA, an evaluation dataset consisting of 10,392 crowdsourced  multiple-choice questions, and the performance of language models at the time (mainly BERT models)
still significantly lags behind human performance. Since \citet{jin2021forecastqa}, there has been a progressive line of work  developing more and more challenging  forecasting benchmarks for more advanced models, by integrating prediction market events \citep{zou2022forecasting}, extracting events from  news articles \citep{zhang2024analyzing,wang2025openforecast}, curating  future-oriented questions from   websites \citep{wildman2025bench}, using open-ended queries \citep{guan2024openep}, and lately  developing dynamic benchmarks with live event and leaderboard updates \citep{zeng2025futurex,karger2024forecastbench,FutureBench}.  

We also develop a benchmark, \prophet, which we intentionally ground on thousands of prediction market events, leveraging their incentive-aligned participation and standardized resolution criteria to obtain a more rigorous testbed that mitigates selection bias from question design. However, our goal goes beyond ranking models: we aim to understand the \emph{LLM-as-a-Prophet} paradigm and analyze how different model capabilities affect forecasting ability. We build \prophet\ primarily for this analytical purpose, similar to \citet{daillms,paleka2025consistency}. 

Towards this end, our work departs from previous benchmarks in several ways. First, we provide a comprehensive evaluation of LLMs' predictive intelligence across multiple dimensions, including Brier score, calibration, and economic value. In Section~\ref{subsec:accuracy}, we illustrate how these metrics differ and when each should be preferred. In contrast, most prior work primarily focuses on a single metric, such as Brier score \citep{halawi2024approaching,karger2024forecastbench}, calibration \citep{zou2022forecasting}, or accuracy \citep{zeng2025futurex,daillms}, which does not fully capture forecast quality (as we further demonstrate in Section~\ref{sec:evals}). Second, as our goal is to better understand frontier LLMs' forecasting capabilities, we require a natural baseline forecaster. By grounding \prophet\ in prediction market events, we can use market-implied probabilities as a cost-effective and theoretically justified baseline \citep{arrow2008promise}. Finally, beyond reporting evaluation metrics, we further analyze the LLM-as-a-Prophet paradigm and provide insights into how forecasts are shaped by specific capabilities such as knowledge internalization and source usage.

\section{Prophet Arena: A Live Benchmark for Predictive Intelligence}
\label{sec:arena}
\subsection{Key Definitions and Notations}

\textit{\small (\textbf{Note:} To aid exposition, we borrow the terminology of prediction markets for the definitions below. The general concept of forecasting, however, extends beyond market-based contexts to any form of future prediction.)}

\textbf{Event. } Let  $\{E_i\}_{i=1}^K$ denote the set of forecasting \emph{events}. An \textit{event} is the overarching question or subject concerning a future real-world occurrence. It serves as a high-level container for one or more tradable \textit{markets}. We consider an event \textbf{not} as a tradable asset; rather, it sets the context, scope, and resolution criteria for the markets that fall under it. \par 

\begin{itemize}[leftmargin=1em]
    \item \textbf{Example $E_1$:} ``Who will win the 2025-26 NBA Championship?"
    \item \textbf{Example $E_2$:} ``Which individuals will President Trump officially meet in 2025?"
\end{itemize}

\textbf{Market. } Each event $E_i$ contains one ore more (binary) \emph{markets} $\{M_{ij}\}_{j=1}^{N_i}$. A \textit{market} is a specific, tradable proposition under an event that will ultimately resolve to either \texttt{Yes} (True) or \texttt{No} (False). It represents a verifiable and unambiguous instantiation of the event’s overarching question. 

\begin{itemize}[leftmargin=1em]
    \item \textbf{Example market under $E_1$:} ``The Boston Celtics will win the 2025 NBA Championship.''
    \item \textbf{Example market under $E_2$:} ``President Trump will officially meet with Emmanuel Macron.''
\end{itemize}

Markets within the same event may be either \textit{mutually exclusive} (e.g., only one NBA team can win the championship) or \textit{non-exclusive} (e.g., a president may meet multiple individuals, or none). The only requirement is that each market be binary, thereby defining a well-posed resolution outcome.

\textbf{Event Resolution. }
An event $E_i$ is said to \textit{resolve} at time $\tau_i$, when all of its markets ${M_{ij}}$ have their outcomes $o_{ij} \in \{0,1\}$ determined, where $o_{ij}=1$ indicates that market $M_{ij}$ resolved to \texttt{Yes}.\footnote{In practice, different markets within an event may resolve at different times; we take $\tau_i$ to be the latest resolution time among them.}

\textbf{Contract. }
For a given market $M_{ij}$, a \textit{\texttt{Yes} contract} $q^Y_{ij}$ is a binary random variable that pays out 1 if $o_{ij} = 1$ and 0 otherwise. A corresponding \textit{\texttt{No} contract} $q^N_{ij}$ is defined analogously, paying out when the \texttt{Yes} contract does not. Each contract comes with a price/value $q_{ij} \in [0,1]$, representing its market price or implied probability; this will be formally defined in~\cref{sec:evals}.

\textit{(Hereafter, when the event index $i$ is clear from context, we drop the subscript $i$ and write $M_j, o_j$.)}

\subsection{The Full Prophet Arena Pipeline}

\begin{figure}[!t]
    \vspace{-3mm}
    \centering
    \includegraphics[width=0.9\linewidth]{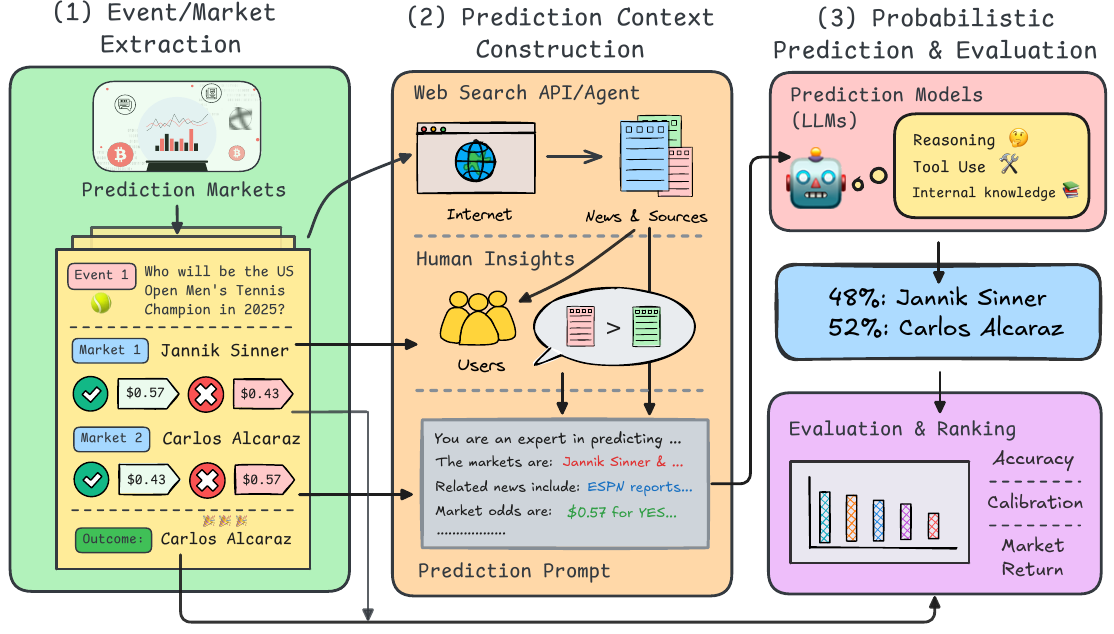}
    \caption{\textbf{Overview of the \prophet\ pipeline.} The outcomes of events are unknown at extraction time. LLM predictions are generated at pre-scheduled forecast times, and evaluation occurs only after the events resolve, comparing model forecasts against the realized outcomes.}
    \label{fig:flowchart}
    \vspace{-3mm}
\end{figure}

\prophet\ is implemented as a live, continuously updated pipeline for evaluating forecasts on real-world events. As illustrated in~\cref{fig:flowchart}, it consists of three main stages -- (1) event and market extraction, (2) prediction context construction, and (3) probabilistic forecasting with evaluation -- which together form an end-to-end workflow for assessing predictive intelligence at scale.

\textbf{(1) \: Event and Market Extraction. }
We continuously collect forecasting events from active prediction markets across diverse domains such as politics, economics, sports, entertainment, and science. Events are selected based on their popularity, diversity, and recurrence patterns, ensuring both temporal and topical coverage. Each day, 20 new events are added to our event pool. For the evaluations in this paper, we apply a cutoff date of October~11,~2025, and include only events that resolved prior to this date. An additional filter is applied to rule out predictions that are too close to the event resolution time (we explain the filtering detail in~\cref{subsec:performance}). The resulting dataset comprises 1{,}367 resolved events encompassing 72{,}136 markets. The use of live markets guarantees that all questions pertain to genuine future outcomes—free from training-data contamination—and that their resolutions can be objectively verified once finalized.

\textbf{(2) \: Prediction Context Construction. }
 Once an event $E_i$ is extracted, \prophet\ schedules a series of \textit{prediction times} $T_i = \{t_i^{(1)}, t_i^{(2)}, \ldots\}$ before the event resolves, enabling models to update their forecasts as market conditions and public information evolve. Accordingly, both the model-predicted probabilities and contract prices are time-dependent (i.e. our notation is actually time-depedent, e.g. $p_{ij} \equiv p_{ij}^{(t)}$). For clarity of exposition, however, we assume in the main text that each model issues a single forecast per event (i.e. $|T_i| = 1$); details on the full scheduling algorithm and multi-horizon aggregation are provided in~\cref{app:prediction_context}.

For each event, \prophet\ constructs a unified \textit{prediction context} $C_i$ that all models receive identically. This context contains:
\begin{itemize}[leftmargin=2em]
\item \textbf{Relevant information sources} retrieved by an LLM-based search agent using web queries that collect titles, snippets, timestamps, and URLs of recent news or reports; and
\item \textbf{Market snapshots} including the latest \texttt{Yes/No} contract prices and trading volumes, from which implied probabilities (i.e. market contract prices $q_{ij}\in[0,1]$) are derived.
\end{itemize}
Providing identical contexts isolates differences in models’ reasoning and calibration rather than their retrieval capabilities. The search component in \prophet\ is fully \textit{searcher-agnostic}: new search agents can be added or replaced without altering the forecasting protocol or evaluation procedure. In all experiments presented in this paper, we use a single LLM-based searcher instantiated with \modelshort{gpt-4o} equipped with web access.

\textbf{(3) \: Probabilistic Forecasting and Evaluation. }
Given an event $E_i$ and its constructed context $C_i$, each model produces a probabilistic forecast for every market $M_{ij}$ within the event. The model outputs a predicted probability $p_{ij}\in[0,1]$, interpreted as its belief that the market will resolve to \texttt{Yes}, together with a short natural-language rationale. All forecasts are logged for later analysis but only the probabilities are used for quantitative evaluation. Once the event resolves and the true outcomes $o_{ij}$ become available, \prophet\ evaluates the forecasts along multiple complementary dimensions -- which are detailed in~\cref{sec:evals}. Together with the event extraction and context construction stages, this final step completes the \prophet\ pipeline -- turning live, real-world questions into a continuous, contamination-resistant benchmark for predictive intelligence.

\textbf{(*) \: Bonus feature: Market Baseline. }
To establish a fair and interpretable anchor for comparing different LLM forecasters, we include a \textit{Market Baseline} -- a synthetic forecaster whose prediction ($p_{ij}^{\mathrm{mb}}$) is defined as the market-consensus probability that market ($M_{ij}$) will resolve to \texttt{Yes}. In practice, this probability is inferred from the normalized contract prices. For instance, if the \texttt{Yes} and \texttt{No} contract prices are ($q^{Y}_{ij} \equiv q_{ij} = 0.8$) and ($q^{N}_{ij} \equiv 1 - q_{ij} = 0.2$), respectively, then the market baseline assigns an 80\% probability to the \texttt{Yes} outcome (i.e. $p_{ij}^{\mathrm{mb}} \equiv q^Y_{ij}$).\footnote{In practice, contract prices may not sum exactly to one because of exchange transaction fees, requiring slight normalization.} As we will demonstrate in~\cref{sec:evals}, this market baseline serves as an informative benchmark for assessing \emph{forecast difficulty} and contextualizing model performance. When LLMs outperform the market baseline, they demonstrate genuine predictive advantage over the aggregated human consensus in the real-world market.

\subsection{Design Choices of \:\prophet\ and Differences from Other Forecasting Benchmarks}
Because our goal is to thoroughly evaluate what capabilities affect LLMs' predictive intelligence, \prophet\ is built with different design choices  compared to several recent forecasting benchmarks, which are designed primarily to rank LLMs or agents based solely on forecasting accuracy. To our best knowledge, \prophet\ is the first large-scale benchmark that continuously runs evaluation on real-market events under a multi-horizon protocol and within a modularized forecasting pipeline.  We further illustrate these differences in \cref{tab:related}, whereas  highlight the four key design choices of \prophet\ below.   

\begin{enumerate}[leftmargin=1.2em,itemsep=2pt,topsep=2pt]
\item \textbf{Probabilistic forecasts:} Future events are intrinsically random, so we elicit each market's probabilistic forecast   from each model, rather than a single choice of the most likely outcome. Notably, probabilistic forecasting is also  the  standard for real-world  forecasting platforms in general, including leading ones such as   \cite{Metaculus} and       \cite{GoodJudge}.        
\item \textbf{Multi-horizon protocol:} We implement a multi-horizon forecasting protocol that sets a schedule for the models to make prediction across various timestamps before the event resolves; we defer details on the scheduling algorithm and multi-horizon aggregation to~\cref{app:prediction_context}. 
This enables our temporal analyses of how models update their forecasts as market conditions and public information evolve across the full lifecycle of an event.
\item \textbf{Modularized forecasting pipeline:} The process of prediction is broken down from source collection to probability elicitation to market actionability. This enables a comprehensive and controlled study of LLMs' forecasting performance at different modules.
\item \textbf{Market return metrics:}   \prophet~allows the evaluation of market profitability of each forecast, measuring models' \emph{relative} advantage over market consensus.   
\end{enumerate}

\begin{table}[ht]
\centering
\begin{tabular}{lccccc}
\toprule
\textbf{Benchmark} & Live events & Probabilistic & Multi-horizon & Modularized & Return metrics \\
\midrule
\textsc{MIRAI}         & --         & --           & \checkmark    & --          & -- \\
\textsc{ForecastBench} & \checkmark & \checkmark   & \checkmark    & --          & -- \\
\textsc{FutureBench}   & \checkmark & --           & --            & --          & -- \\
\textsc{FutureX}       & \checkmark & --           & --            & --          & -- \\
\textbf{\prophet}      & \checkmark & \textbf{\checkmark} & \textbf{\checkmark} & \textbf{\checkmark} & \textbf{\checkmark} \\
\bottomrule
\end{tabular}
\vspace{3mm}
\caption[Comparison between \prophet\ and related forecasting benchmarks]{%
\textbf{Comparison between \prophet\ and related forecasting benchmarks.}
We compare against MIRAI~\citep{DBLP:journals/corr/abs-2407-01231}, ForecastBench~\citep{karger2024forecastbench}, FutureBench~\citep{FutureBench}, and FutureX~\citep{zeng2025futurex}.
}
\vspace{-3mm}
\label{tab:related}
\end{table}
\section{Evaluating Forecasts: Caveats, Metrics, and Results}
\label{sec:evals}

The golden metric in machine learning for evaluating prediction of binary (or multiary) outcomes is the \emph{accuracy} measure (i.e., 0-1 loss). Indeed, several recent studies on benchmarking LLMs'  forecasting capabilities (e.g., \citet{zeng2025futurex,karger2024forecastbench,wildman2025bench})  have used models' accuracy as an indicator of their  forecasting  capabilities. These studies ask LLMs to predict a deterministic outcome and then evaluate this prediction's accuracy (i.e., 0-1 loss). As a first caveat, we point out the limitation of measuring    probabilistic forecasts via the accuracy metric. This limitation is  due to the  random nature of future events, which make them intrinsically different from typical classification tasks (e.g., recognizing whether a cat is in an image or not) with a deterministic answer. This is also why real-world  forecasting platforms, including leading ones such as   \cite{Metaculus} and       \cite{GoodJudge}, almost always elicit probabilistic forecasting.\footnote{Some classification algorithms, such as logistic regression, also compute probabilities. However, unlike forecasting, such probabilities are  not a requirement on the output but rather for the sake of algorithm design. Indeed, these algorithms often convert the probability to a deterministic output as its final prediction.}
 
\begin{caveatbox}{\textbf{Caveat 1:} Accuracy cannot fully measure the quality of probabilistic forecasts.} 
To see this, consider a binary random event $E$ with groundtruth probability $p^* = 0.6$. Suppose Alice has a perfect forecast of the probability as $p^A = 0.6$ whereas Bob's forecast is $p^B=1$. To evaluate their predictions' accuracy, each forecaster must give a predicted \emph{outcome}  $o\in \{0,1 \}$. In this case, it is natural for both to predict $o = 1$ (as they both believe $o=1$ is more probable), which yields the same accuracy $\E_{o\sim p^*} \Pr(o=1) = 0.6$. Therefore, the accuracy metric fails to distinguish   Alice's perfect forecast of the  event's probability from Bob's forecast, which is extremal and not precise. The intrinsic reason of this failure is that using accuracy alone ignores prediction confidence, which is reflected in forecasted probabilities. In this example, Bob is overly confident  about the outcome $o=1$ compared to Alice, despite having the same accuracy as her. Nevertheless, correctly measuring the confidence of a forecast is  crucial for downstream decision making tasks, such as risk control in financial markets \citep{rigotti2005uncertainty}.       
\end{caveatbox}

\subsection{Evaluating Forecasting from Three Dimensions. }\label{subsec:metrics}
Given the \textbf{Caveat 1} above, we identify three different dimensions of forecasting evaluation  in order to obtain a comprehensive understanding of LLMs' forecasting capabilities. These three    dimensions are (1) \emph{forecasting loss} \citep{gneiting2007strictly, glenn1950verification} which measures the \emph{absolute} quality of a probabilistic prediction; (2) \emph{calibration error} \citep{degroot1983comparison,guo2017calibration}  which measures reliability (formally, statistical consistency) of a probabilistic prediction; and (3) \emph{market return} \citep{mallikarjuna2019evaluation} which measures the \emph{relative} advantage over current market's consensus.\footnote{Market consensus about an event's forecast is generally difficult to obtain. However, a good proxy of such data is available for \prophet~as it fetches forecasting events from prediction markets, which are widely believed to offer good approximation of the market's consensus on the event's probability \citep{arrow2008promise,berg2008results}.} As we illustrate below,  these dimensions  capture fundamentally different aspects of a forecast, enabling us to build a  comprehensive understanding of LLMs predictive intelligence. For each dimension,  we choose the most standard metrics.

\subsubsection{Scoring Rules to Measure Forecasting Loss}
\label{subsec:accuracy}
       
The standard approach for evaluating probabilistic forecasts is through  \textit{proper scoring rules} \citep{gneiting2007strictly}, which quantifies the discrepancy between predicted probabilities and realized outcomes. A scoring rule is \emph{proper} if it ensures that the perfect probabilistic forecast achieves the optimal expected score. 
\prophet~adopts one of the most popular proper scoring rules,   the \textbf{Brier score}~\citep{glenn1950verification}, 
which is defined for each event $E_i$ as\footnote{Careful readers might notice that definition given here has some subtle difference from a ``textbook version.'' We delay a detailed explanation to~\cref{app:brier-details}.}
\begin{equation}
\label{eq:brier}
BS_i := \frac{1}{m_i} \sum_{j=1}^{m_i} \bigl(p_{ij} - o_{ij}\bigr)^2. 
\end{equation}
Therefore, the overall Brier score of an LLM is the average   across all $n$ events, $BS := \frac{1}{n} \sum_{i=1}^n BS_i.$ 
The Brier score is the mean squared error for probabilistic forecasts, hence   smaller scores indicate better forecasts. It is easy to verify that Brier score separates  the quality of Alice's and Bob's   forecasts  in the above example:  Bob's expected Brier score  is $\E_{o \sim p^*} (p^B - o)^2 = 0.4$, whereas Alice's   is $\E_{o \sim p^*} (p^A - o)^2 = 0.24$.

\subsubsection{Calibration Errors to Measure Reliability} \label{subsec:calibration}

In addition to achieving low Brier scores, a good forecast should also be \emph{reliable} -- that is,  when it assigns probability $\Tilde{p}$ to \texttt{Yes}, the event should indeed occur about $\Tilde{p}$ fraction of the time. This notion of \emph{reliability} of a forecaster dates back to   \citep{murphy1973new}  and has been conventionally studied in the machine learning literature as \textit{calibration} \citep{degroot1983comparison,guo2017calibration,kalai2024calibrated} \footnote{We note that recent works on  forecast evaluations have separately considered Brier scores (e.g., \citep{halawi2024approaching,karger2024forecastbench}) and calibration errors (e.g., \citep{zou2022forecasting}), though no work has conducted comprehensive evaluation of all these metrics to our knowledge, especially  models' market returns.}.

Formally, in our binary market setting, let  $M = \{1,\cdots, m\}$ denote the set of markets. Suppose a forecaster provides predicted probabilities $\tilde{p}_k$ for all markets in a set $M_k \subset M$ ($m_k = |M_k|$ denotes its cardinality). Then the \textbf{expected calibration error (ECE)} of this forecaster is defined as
\begin{equation}
\label{eq:perfect-calibration}
   ECE = \frac{1}{m} \sum_k \bigg| \sum_{j \in M_k} \mathbb{P}(o_j = 1| p_j=\Tilde{p_k}) - m_i \Tilde{p}_k \bigg| .
\end{equation} 
Intuitively, ECE captures how much a probabilistic forecast differs from the  real averaged probability, given this forecast level. A lower ECE means the forecast is more reliable, though it does not necessarily imply small forecasting loss. This difference is reflected in our LLM evaluations below, whereas an illustrative mathematical example can be found in~\cref{app:ece-differ-brier}. 

 In practice, however, exactly repeated predictions are rare, so computing $\mathbb{P}(o_j = 1| p_j=\Tilde{p}_i)$ in~\cref{eq:perfect-calibration} is infeasible. A standard solution -- used throughout the applied literature -- is to approximate ECE by partitioning predictions into probability bins, and comparing each bin’s empirical accuracy to its average predicted probability. This binned version is what most prior work simply calls ``ECE.'' Since the implementation is routine and well-known, we defer the details to \cref{app:ece-details}.

\subsubsection{Market Return to Measure Economic Values}
\label{subsec:market-return}
 \prophet~ uses events from real-world prediction markets, hence allows us to evaluate the economic value of a forecast based on current market prices. We thus introduce  \emph{Average Return}  as a natural metric to capture \emph{how profitable it would be to trade in the market using LLM forecasts?}
Formally, consider a (binary) market with  market price $q^{Y}$  per share  for the \texttt{Yes} contract and $q^{N}$ for   the \texttt{No} contract. Given forecasted probability $p$ for an \texttt{Yes} outcome, we use the following natural betting strategy: allocate a unit budget (\$1) to buy $1/q_Y $ shares of \texttt{Yes} contracts if $ \frac{p}{q^Y} \geq \frac{1-p}{q^N}$, or to buy $1/q^N$ shares of \texttt{No} contracts otherwise.\footnote{This betting strategy  can be shown to maximize expected return under risk neutrality. We defer its proof to~\cref{app:unified-betting}, describing a   general  framework for optimizing betting strategies under different risk-preferences.}  After the event resolution, if the bought contracts match  the event outcome, the return is   simply the number of these contracts; otherwise, the return is $0$. The \textbf{Average Return} of an LLM is then defined as the average of these  returns across all markets, under a unit budget allocated to each market.

\begin{remarkbox}{\textbf{Remark:} Intrinsic Differences among the Three  Metrics.} 
\textbf{First},   the  difference between the Brier score and ECE is well-known \citep{murphy1973new}.\footnote{\cite{murphy1973new} shows that the Brier score can be decomposed into three terms: the ECE   (also coined   it ``reliability score'' by \cite{murphy1973new}), the \emph{uncertainty} score that captures inherent randomness of the to-be-forecasted event, and the \emph{resolution} score that captures forecasts' variance.} A forecast with better/smaller ECE could have worse Brier score (see an example in \cref{app:ece-differ-brier}). In practice,  ECE is crucial when decision makers have \emph{risk preferences} since smaller ECE ensures that the risk encoded in the forecasted probabilities is more reliably captured, leading to better risk-preference-adjusted decisions -- even when the forecasted  probabilities have worse Brier scores. In \cref{app:ece-differ-brier}, we illustrate this insight with a simple example in which a forecaster with better ECE but much worse Brier score could lead to better  risk-adjusted  utility.

\textbf{Second}, while both market return and Brier score assess forecast quality, they differ fundamentally. The Brier score is an \emph{absolute} metric, measuring a forecast’s closeness to the ground truth, independent of external factors like market prices. In contrast, market return is a \emph{relative} metric, capturing how much a forecast outperforms the market’s current belief (interpreted as the contract price~\citep{wolfers2006interpreting}). This distinction is already evident in our evaluations. In \cref{app:brier-vs-return}, we provide a concrete example to illustrate that forecasts with worse Brier scores can achieve higher market returns.

\textbf{Third}, calibration is not related to the total market return, but is indicative about how balanced the market return is from betting on the \texttt{Yes} contracts and \texttt{No} contracts.  To formalize this intuition,  we prove in \cref{app:calibration-and-return}  that a well-calibrated and symmetric forecaster -- intuitively, one that is not systematically more aggressive or conservative than the market -- will have balanced returns from both contract types. 
\end{remarkbox}

\subsection{Evaluation Results and Analysis Across Different Dimensions }
\label{subsec:performance}
\subsubsection{Caveats from Temporal Analysis of Forecasting Results}

\begin{figure}[ht!]
    \begin{minipage}[c]{0.5\textwidth}
        \includegraphics[width=\linewidth]{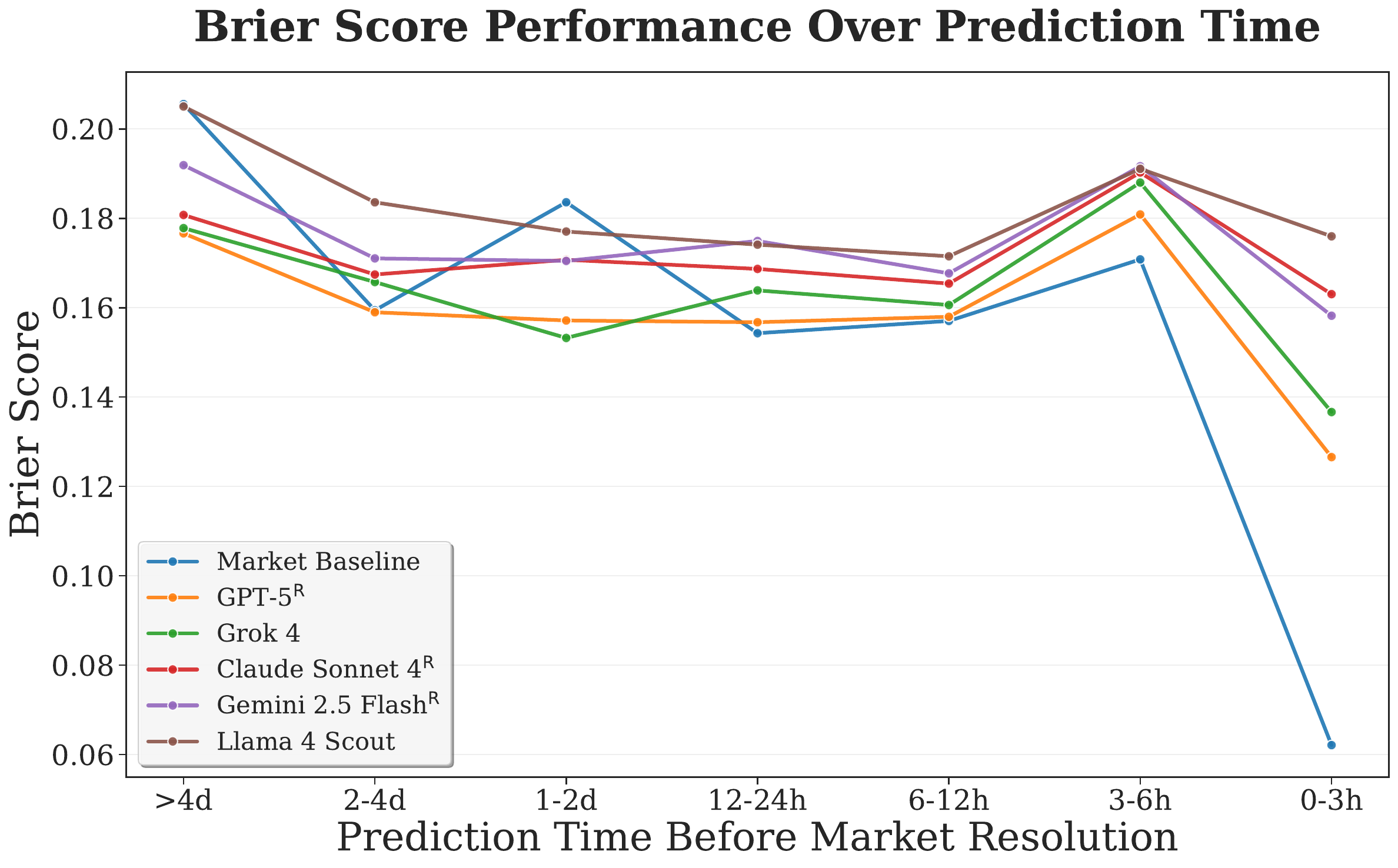}
    \end{minipage}
    \hfill
    \begin{minipage}[c]{0.45\textwidth}
    \vspace{3mm}
    \caption{\textbf{Brier score as a function of prediction time before resolution.} 
    Predictions made closer to the market resolution time tend to be better in Brier score, as additional information becomes available. While some LLMs outperform the market baseline at longer horizons, the market adjusts rapidly to new signals and achieves the highest short-term accuracy near event resolution.}
    \label{fig:brier-over-time}
    \end{minipage}
    \vspace{-3mm}
\end{figure}

Before presenting our main evaluation results, we first examine how predictive accuracy varies with the time remaining before event resolution. As shown in~\cref{fig:brier-over-time}, the average Brier score is plotted across lead-time intervals, where, for instance, the ``0-3 h'' bin represents predictions made within three hours before resolution. For both the market baseline and representative LLMs, accuracy generally improves as the resolution time approaches -- reflecting that additional information becomes available and predictive signals strengthen.\footnote{In certain domains, such as live sports, markets may remain open during the event itself, allowing traders to incorporate real-time developments as outcomes unfold.} Interestingly, the market baseline lags behind several frontier LLMs when predictions are made far in advance, suggesting that LLMs can effectively synthesize broader prior knowledge and reason under noisy settings at long horizons. However, as resolution nears, markets incorporate breaking information and news updates more rapidly than LLMs, quickly surpassing LLMs in short-term accuracy. 
\emph{This observation motivates two methodological takeaways that inform our evaluation design:}

 \begin{caveatbox}{\textbf{Caveat 2:} Raw Brier scores can be misleading, and having a baseline is instrumental in understanding predictive intelligence.}
Predictive difficulty varies drastically with lead time: the same event may be highly uncertain a week in advance but almost deterministic hours before resolution. Hence, comparing raw Brier scores across horizons conflates forecasting skill with intrinsic event difficulty. Including the \textit{market baseline} as a dynamic reference forecaster provides a meaningful normalization, offering a heuristic measurement of ``how predictable an event is'' at each point in time.
 \end{caveatbox}

\begin{caveatbox}{\textbf{Caveat 3:} Predictions too close to resolution should be excluded.}
Near-resolution forecasts are dominated by real-time information access rather than reasoning ability. Because \prophet\ holds the retrieval component fixed across models, such predictions no longer reflect intrinsic model capabilities. Consequently, we exclude all forecasts made within three hours before event resolution from subsequent evaluations. 
\end{caveatbox}

\subsubsection{Evaluation Results and Discussions} 
Throughout the main body of the paper, we highlight (the same) five representative LLMs\footnote{Models are chosen to span proprietary and open-source families, reasoning and non-reasoning variants, and a range of performance levels in the full ranking.} out of the 23 evaluated in total (the full results are available in~\cref{tab:full-eval-results} of \cref{app:full-eval}).  As shown in \cref{tab:main-evaluation-result},  frontier proprietary models demonstrate similar Brier score performances as the the \textit{Market Baseline}, and consistently outperform the later in terms of calibration and market return. 

\begin{table}[!th]
\centering
\begin{threeparttable}
\setlength{\tabcolsep}{6pt}
\renewcommand{\arraystretch}{1.15}
\begin{tabular}{lcccccc}
\toprule
& \multicolumn{2}{c}{\textbf{Forecasting Loss}} & \multicolumn{2}{c}{\textbf{Calibration Error}} & \multicolumn{2}{c}{\textbf{Market Return}} \\
\cmidrule(lr){2-3}\cmidrule(lr){4-5}\cmidrule(lr){6-7}
\textbf{LLM} & {$\downarrow$ \textbf{Brier} (95\% CI)} & {\textbf{Rank}} & {$\downarrow$ \textbf{ECE}} & {\textbf{Rank}} & {$\uparrow$ \textbf{Average} (95\% CI)} & {\textbf{Rank}} \\
\midrule
GPT-5$^\texttt{R}$~$\triangle$
  & 0.184 ($\pm$ 0.006) & {\large \ding{172}} & 0.042 & {\large \ding{173}} & 0.943 ($\pm$ 0.042) & {\large \ding{172}}\\
Grok-4$^\texttt{R}$~$\triangle$ 
  & 0.189 ($\pm$ 0.005) & {\large \ding{173}} & 0.043 & {\large \ding{174}} & 0.864 ($\pm$ 0.052) & {\large \ding{175}}\\
Claude Sonnet 4$^\texttt{R}$~$\triangle$
  & 0.194 ($\pm$ 0.006) & {\large \ding{174}} & 0.041 & {\large \ding{172}} & 0.909 ($\pm$ 0.101) & {\large \ding{173}}\\
Gemini 2.5 Flash$^\texttt{R}$~$\triangle$  
  & 0.197 ($\pm$ 0.007) & {\large \ding{175}} & 0.067 & {\large \ding{176}} & 0.883 ($\pm$ 0.053) & {\large \ding{174}}\\
Llama-4-Scout~$\triangle$
  & 0.219 ($\pm$ 0.008) & {\large \ding{176}} & 0.060 & {\large \ding{175}} & 0.805 ($\pm$ 0.040) & {\large \ding{176}}\\
\midrule
Market Baseline & 0.187 ($\pm$ 0.006) & N/A & 0.069 & N/A & 0.899 ($\pm$ 0.043) & N/A \\
\bottomrule
\end{tabular}
\end{threeparttable}
\vspace{2mm}
\caption{\textbf{Evaluation of five representative LLMs.} For Brier and Average Return, bootstrapped $95\%$ confidence intervals are reported. Superscript $^\texttt{R}$ is used to denote a reasoning model, with its reasoning configuration in~\cref{app:full-llm-list}. The full results for all 23 LLMs are provided in~\cref{tab:full-eval-results}. Although our benchmark updates in real time, for the purpose of writing this paper we need to fix a dataset. Our evaluation is conducted on 1{,}367 events that were resolved before October 11, 2025.} \label{tab:main-evaluation-result}
\vspace{-3mm}
\end{table}
Notably, the relative rankings of models differ depending on which evaluation metric is used, illustrating the complementary perspectives offered by accuracy, calibration, and profitability. Concretely, Brier scores fall in a narrow band $[0.18, 0.22]$ (for reference, pure random guess has expected Brier score $0.25$). By contrast, calibration differences are more pronounced: strong models typically achieve $\text{ECE} \leq 0.05$, whereas weaker ones fall in the $[0.06, 0.2]$ range. Nevertheless, all the selected LLMs demonstrate better calibration than the market baseline. For market performance, even \modelshort{gpt-5}$^\texttt{R}$, the top-ranked model, fails to reach break-even (Average Return $< 1$), and most models fall below $0.9$. Since event-level payoffs depend heavily on market-implied probabilities, the resulting returns exhibit substantial variance, as evidenced by the wide confidence intervals. In~\cref{app:sharpe-ratio}, we further discuss the \textit{Sharpe ratio}~\citep{sharpe1998sharpe} of our betting strategy, which normalizes Average Return by volatility, providing a more stable comparison of models' economic performance. 

Overall, our results suggest that absolute forecasting skill and relative profitability against prediction markets are still challenging for today’s LLMs. A more fine-grained investigation could help  deepen our understanding about what could make  an LLM   a good prophet.  \cref{fig:eval-results} draws  for the best and worst models (Left \& right) their  reliability diagrams regarding calibration errors \citep{guo2025deepseek}, where predicted probability ($x$-axis) is compared against realized frequency ($y$-axis) at different probability bins. While their calibration is similar in intermediate ranges, the stronger model -- \model{o3} -- performs much better in the extreme bins ($0$-$0.1$ and $0.9$-$1.0$), where it almost always predicts correctly. Because such extreme forecasts occur frequently, this advantage helps explain the gap in both Brier score and market return.

The empirical results reported in this paper are based on evaluations conducted over 1,367 events resolved prior to October 11, 2025. These events were collected via our real-time Kalshi data pipeline, and consequently the distribution of event categories reflects the underlying composition: 76\% Sports, 8\% Entertainment, 7\% Politics, and 9\% Other. To ensure that the results remain robust to event composition, we additionally verify that these performance patterns are robust to reweighting the dataset toward a more balanced distribution across categories. Full results from these balanced-subset evaluations are reported in Appendix~\ref{app:category-rebalancing}. Regardless, we plan to continue expanding and diversifying the event pool in future releases of \prophet.

\begin{figure}[ht!]
  \centering
  \begin{minipage}[c]{0.65\textwidth} 
    \begin{minipage}[c]{0.49\textwidth}
      \includegraphics[width=\linewidth]{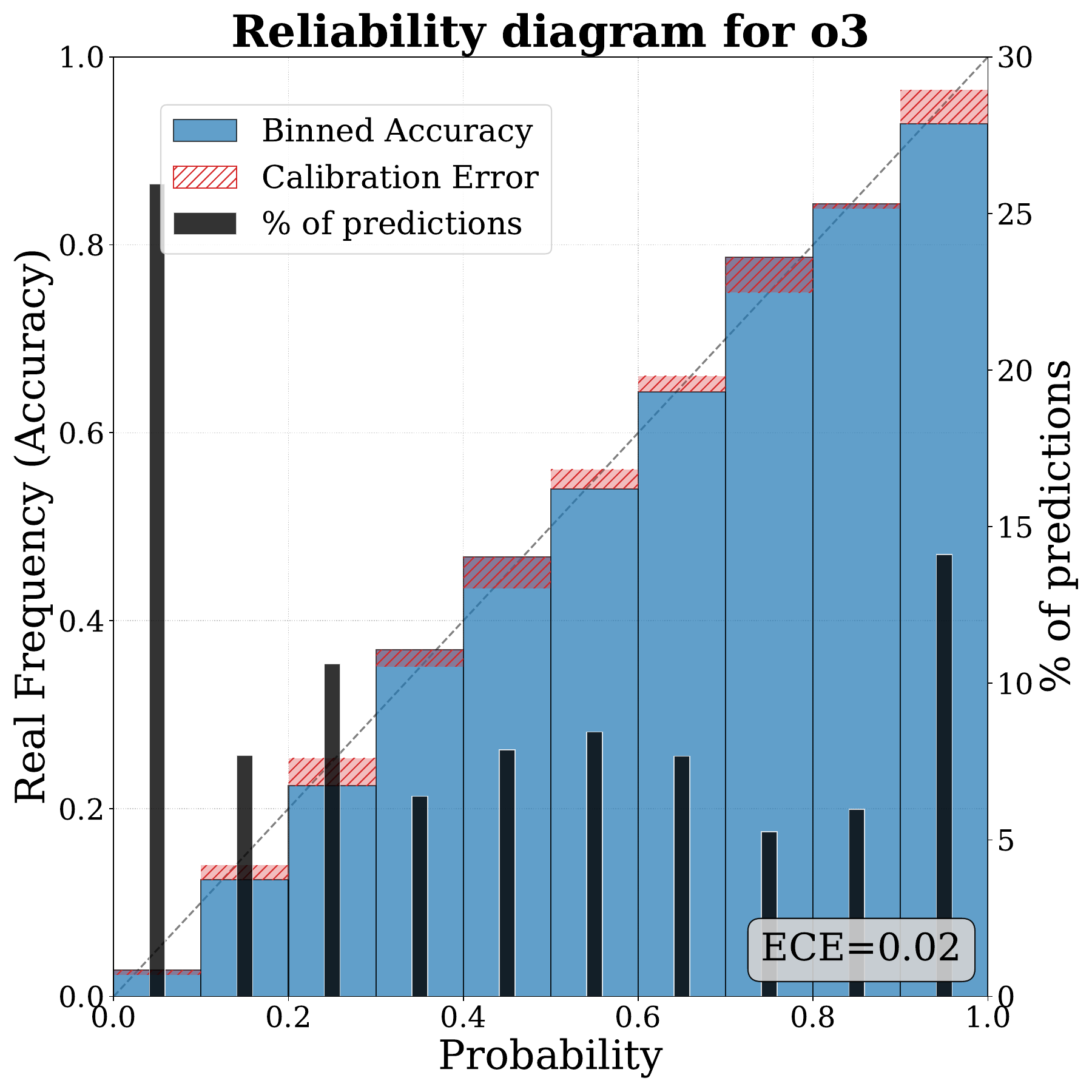}
    \end{minipage}\hfill
    \begin{minipage}[c]{0.49\textwidth}
      \includegraphics[width=\linewidth]{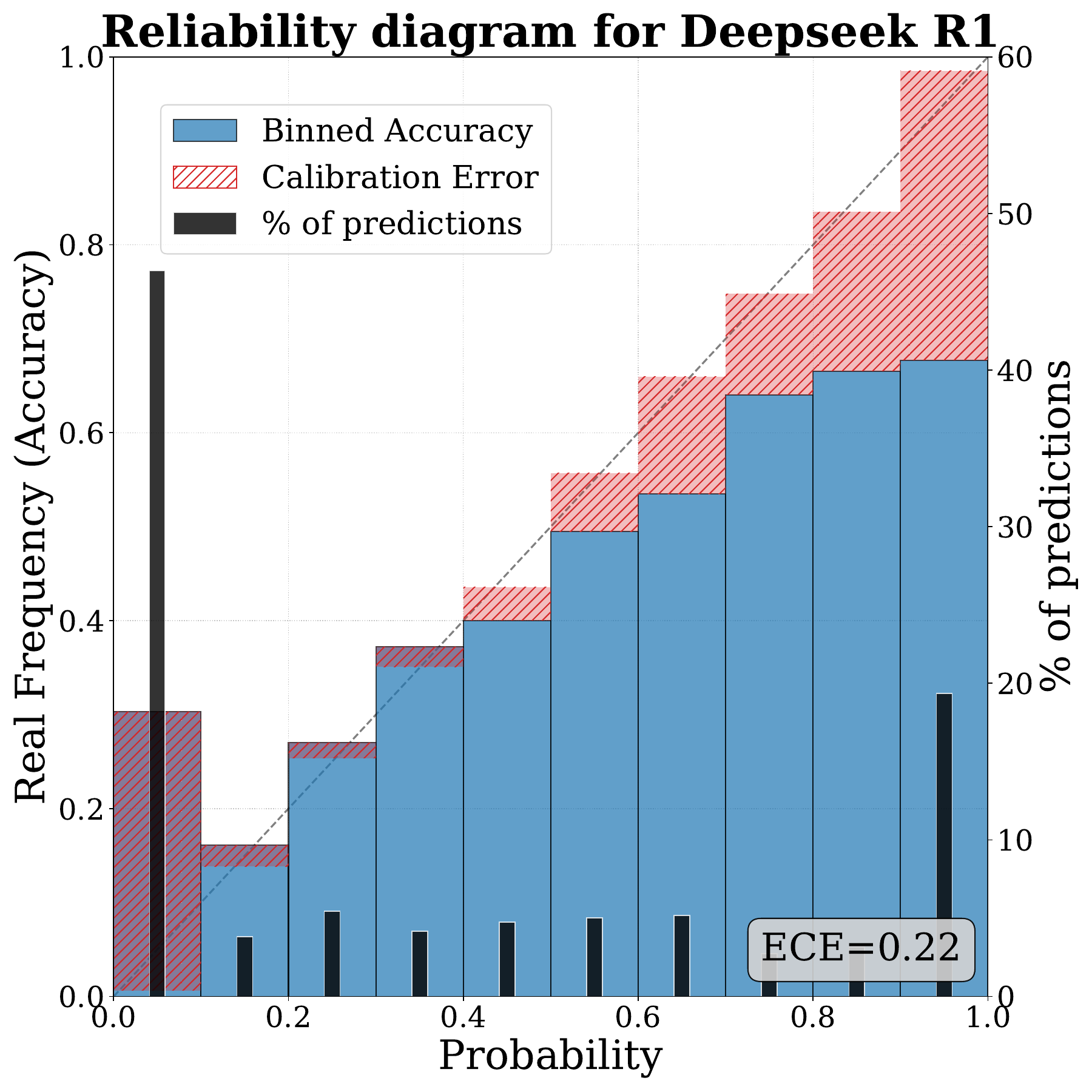}
    \end{minipage}
  \end{minipage}%
  \hfill
  \begin{minipage}[c]{0.33\textwidth}
    \caption{\small \textbf{Reliability diagrams for the best and worst LLMs ranked by calibration score (ECE).} 
    The \textbf{black} histogram indicates fraction of predicted probabilities in each bin. 
    The calibration error within each bin appears as the height of the \textcolor{red}{red} rectangle, i.e., the gap between accuracy and confidence. 
    The reported ECE score (bottom right) corresponds to a weighted sum of these errors, weighted by the distribution. 
    Superscript $^\texttt{R}$ denotes a reasoning model.} 
    \label{fig:eval-results}
  \end{minipage}
\vspace{-5mm}
\end{figure}



\section{An In-Depth Analysis of LLM-as-a-Prophet}
\label{sec:exps}

In this section, we deploy our tailored \prophet\ framework to perform a systematic investigation into the emerging paradigm of LLM-as-a-Prophet. Due to the resource constraint, our experiment results in this section uses a subset of 100 events sampled uniformly from our full benchmark.\footnote{We make this dataset publicly available at \url{https://huggingface.co/datasets/prophetarena/Prophet-Arena-Subset-100}.} For clarity, this section prioritizes the presentation of interesting     findings and novel insights, during which we shall also point the reader to corresponding appendices for more thorough evaluations and discussions.

We begin with a series of sanity checks to ensure that existing LLMs demonstrate reasonable level of understanding and reasoning capabilities to make forecasts.
Our analysis proceeds from two complementary perspectives. The first adopts a mechanistic approach, designing controlled experiments to uncover the principles and failure modes that govern model behavior. The second follows a granular evaluation procedure, assessing how well current LLMs can reason about future events when embedded in realistic forecasting environments.
Together, these experiments establish an empirical foundation for understanding what it means for a model to ``forecast'', and delineate the boundary of its current predictive intelligence. 

\subsection{Robustness and Consistency Checking}
The basis to accomplish the general forecasting tasks is to have a good understanding of the question as well as the nature of probability. Hence, we examined two fundamental capabilities of existing LLM models: (i) \textit{robustness of probability elicitation}, where calibration remains stable under prompt variations and alternative probability estimation methods, and (ii) \textit{logical consistency}, where most LLMs correctly understand structures such as mutually exclusive or nested markets. For the majority of models, both capabilities appear already reliable and largely mature. Detailed results are provided in \cref{app:probability-elicitation,app:reasoning-consistency}.

\subsection{Mechanistic Analysis of LLM-as-a-Prophet}
We begin by taking a mechanistic view of LLM predictive intelligence, aiming to uncover how it arises from the interaction of distinct model capabilities, in aim to identify the bottlenecks and causal links that determine their effectiveness. To this end, we design experiments that examine the key factors behind strong forecasting performance, starting from a model's internal knowledge (section \ref{subsubsec:internalization}), to the quality and accessibility of external sources (section \ref{subsubsec:context}), to its ability to integrate those sources effectively (section \ref{sec:modelsConservative}).

\subsubsection{Can Internalized Knowledge Become Foresight?}\label{subsubsec:internalization}

Forecasting begins with what a model already knows --- the \textit{internalized} knowledge.
A key question is whether models possess accurate representations of past outcomes and can effectively leverage this knowledge to inform present forecasts.
To investigate this, we examine how well models recall and interpret historical events, where we retrieve 100 past events from Kalshi that occurred prior to each model’s knowledge cutoff date and evaluate their responses under the recall prompt described in~\cref{app:mem_prompts}.

\begin{figure}[!ht]
\centering
\begin{minipage}[c]{0.5\textwidth}
  \includegraphics[width=\linewidth]{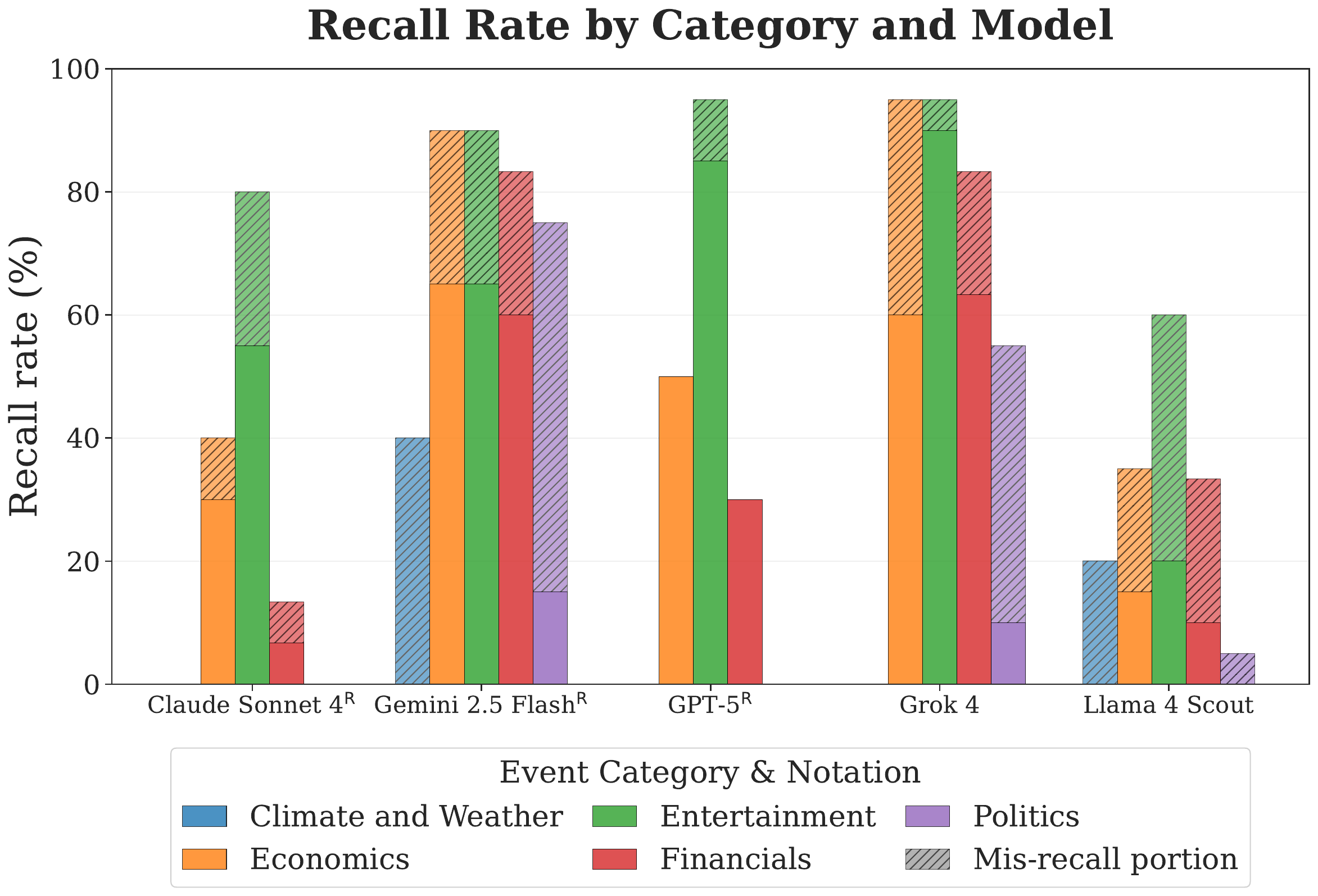}
\end{minipage}
\, \quad 
\begin{minipage}[c]{0.4\textwidth}
  \caption{\textbf{Recall rate by event category and model using the knowledge internalization recall prompt (\cref{app:recall_prompt}).} Shaded region represents the percentage of events reported to be recognized by the LLM but having a wrong recall.}
  \label{fig:recall_rate}
\end{minipage}
\end{figure}

\paragraph{Event recall varies by topics and models.} In~\cref{fig:recall_rate}, we observe that models most reliably recall events in \textit{Entertainment}. By contrast, \textit{Climate and Weather} and \textit{Politics} display low recall and frequent mis-recall. Two factors likely contribute. First, \textit{Weather} prompts often require fine-grained, date-stamped facts (e.g., “\textit{Highest temperature in Miami on Aug 29, 2023?}”), which are less likely to be memorized. Second, due to 2023 regulatory constraints on Kalshi's election markets, \textit{Politics} in our dataset skews toward politics-adjacent indicators (e.g., ``\textit{Biden 538 approval rating on Aug 30, 2023?}”), which varies daily and may be sparsely represented in training corpora (\cite{kalshi}). \par 
Despite these broad patterns, models differ in recall accuracy. For \textit{Economics} and \textit{Politics}, \modelshort{gpt-5} (High) correctly answered all events it claimed to recall. In contrast, models like \modelshort{llama-4-scout} and \modelshort{gemini-2.5-flash} (Reasoning) reported recognizing events in all categories. However, for categories like \textit{Climate and Weather} as well as \textit{Politics}, almost all of their recalled events were false.

\textbf{Event recall is approximate, not precise.} Consider the event ``\textit{Billboard Hot 100 \#1, Jul 13, 2023?}". \modelshort{gemini-2.5-flash} recalled (\cref{app:event_recall}) that Olivia Rodrigo’s \textit{Vampire} displaced Morgan Wallen’s \textit{Last Night} at number one, which is correct. However, it aligned the answer with the chart dated July 15, 2023, treating it as equivalent to July 13, 2023. Thus, while the model demonstrated knowledge of the outcome, it failed to recall the precise alignment between dates and chart releases. This illustrates that event recall in LLMs is approximate: models often retain coarse associations (the correct song and transition) but lack fidelity on exact temporal details. An intriguing open question is to quantify the effect of inaccurate past knowledge internalization on the model's forecasting capability in future events.

\subsubsection{How Do Contexts Shape Forecasts?}
\label{subsubsec:context}
Besides the internalized knowledge, making good forecasts also requires good context information --- \prophet\ retrieves relevant news sources from the Internet and provide the live market data from Kalshi. 
We first analyze how these different information sources affect model performance. \cref{fig:new-sources-category} shows the average Brier scores across all tested LLMs under four conditions: access to both market data and news sources, market data only, news sources only, and none.
The results demonstrate a clear performance hierarchy; as expected, models with access to both market data and news sources output the best predictions, while those without access to either source of information exhibit the poorest performance. 

Beyond performance differences, the variance patterns reveal deeper insights. Interestingly, models using only market data perform only slightly worse, on average, than those with both market data and news sources; however, the key difference lies in the variability of predictions. Combining multiple high-quality sources substantially reduces variance in prediction quality, suggesting that sources still offer valuable signals and perspectives that enable more consistent forecasting. Thus, while market data is powerful precisely because it aggregates information from a plethora of sources and trends, adding a few carefully chosen high-quality sources can still help stabilize and refine the signals the information provides. \cref{app:llm_finding_srcs} includes a deeper analysis of LLMs’ abilities to find and utilize high-quality sources.

\begin{figure}[ht!]
    \centering
    \includegraphics[width=0.9\linewidth]{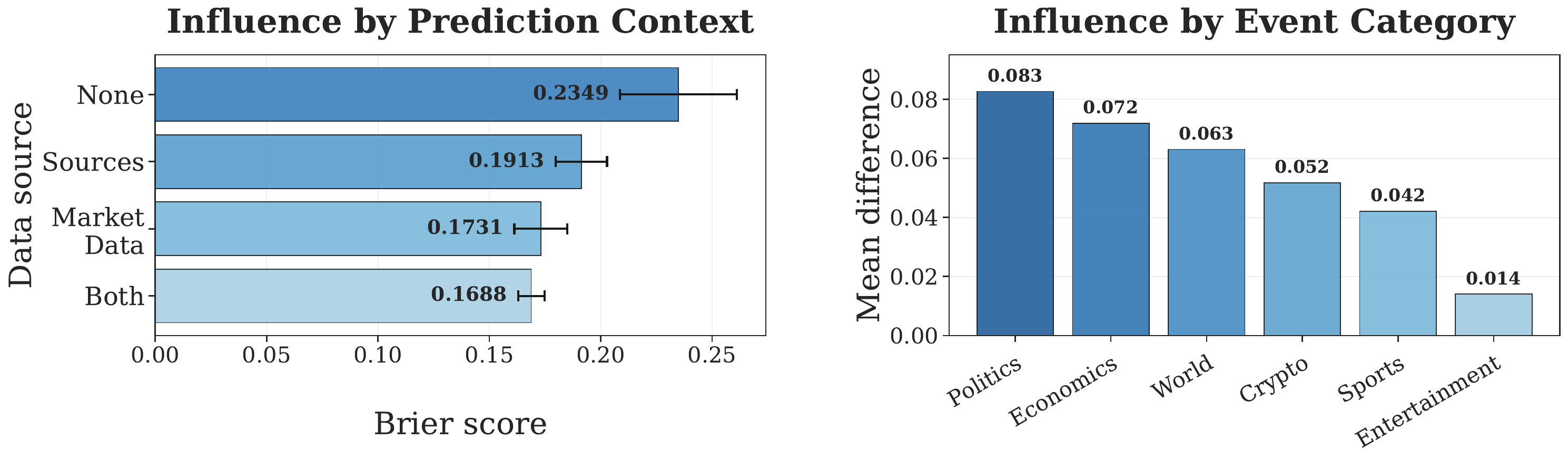}
    \vspace{-2mm}
    \caption{\textbf{Prediction quality across different contexts and event categories.} Left panel shows average Brier scores across evaluated LLMs for prompts with varying information availability (none, sources only, market data only, or both). Error bars represent the interquartile range (IQR) of model performance. Right panel displays mean Brier scores of events in each category.}
    \label{fig:new-sources-category}
\end{figure}

\textbf{Sources can clarify or confound predictions.}
As shown in ~\cref{fig:new-sources-category}, on average, adding sources improves mean Brier score. However, the effect is heterogeneous: not only does the magnitude of the effect vary based on event category, it also does not necessarily strictly improve prediction quality, as shown in the case study in~\cref{app:bitcoin}. Regarding category differences, the benefit of adding sources is not uniform. In areas like politics, where events can be interpreted through multiple perspectives, incorporating information from varied outlets and institutions appears to add useful context. In contrast, in domains such as entertainment or sports, the marginal value of additional sources seems smaller. This highlights that more information is not necessarily better; the effectiveness of sources depends on their relevance to the prediction task.

\subsubsection{How Models Engage with Sources}
\label{sec:modelsConservative}
Having established that access to contextual information improves forecasting performance, we next examine how models actually engage with those sources. In particular, we study how models calibrate their beliefs in light of new evidence: do they tend to keep their prior belief or open to adjust in response to external signals? Do they amplify confidence when multiple sources are consistent, or moderate it when signals diverge?
 
\paragraph{Sources drive LLMs' forecasts to be more conservative than markets.}
\cref{fig:scout_vs_gpt5} compares model prediction probabilities to the market baseline's probabilities on those markets resolved to \texttt{Yes} (i.e., the ``winner''). Since market data is included in the LLM prompts -- most models generally align closely with market predictions. However, across a large majority of events, LLMs consistently output more conservative probabilities. A representative example is \modelshort{llama-4-scout}, shown in \cref{fig:llama-scout-parity}: even when markets assign near-certain probabilities (close to one), the model remains hesitant, rarely producing equally extreme predictions. This reflects a systematic conservatism  where the model, when being fed  with additional external  sources, tend to  underweight outcomes the market views as almost certain. 
As illustrated in \cref{fig:gpt5-parity} \& \cref{fig:grok4-parity}, although both \modelshort{gpt-5} and \modelshort{grok-4} adopt a cautious approach in their probability assignments, they generally track the market more closely across the full probability range and avoid the same degree of reluctance. \modelshort{claude-sonnet-4} exhibits a slightly more assertive pattern: in the mid-probability range (around 0.5-0.6), it occasionally assigns probabilities slightly above the market-implied probabilities. Nevertheless, across most events, it displays significant conservatism, and particularly at the higher end of the probability spectrum, \modelshort{claude-sonnet-4} also is reluctant to place extreme probability predictions. Overall, while conservatism is a common trait across models, the extent of hesitation varies, with some models exhibiting considerably stronger reluctance than others. Similar trends have appeared in other models as well.

\begin{figure}[htbp]
  \centering
  \begin{subfigure}[t]{0.35\linewidth}
    \centering
    \includegraphics[width=\linewidth]{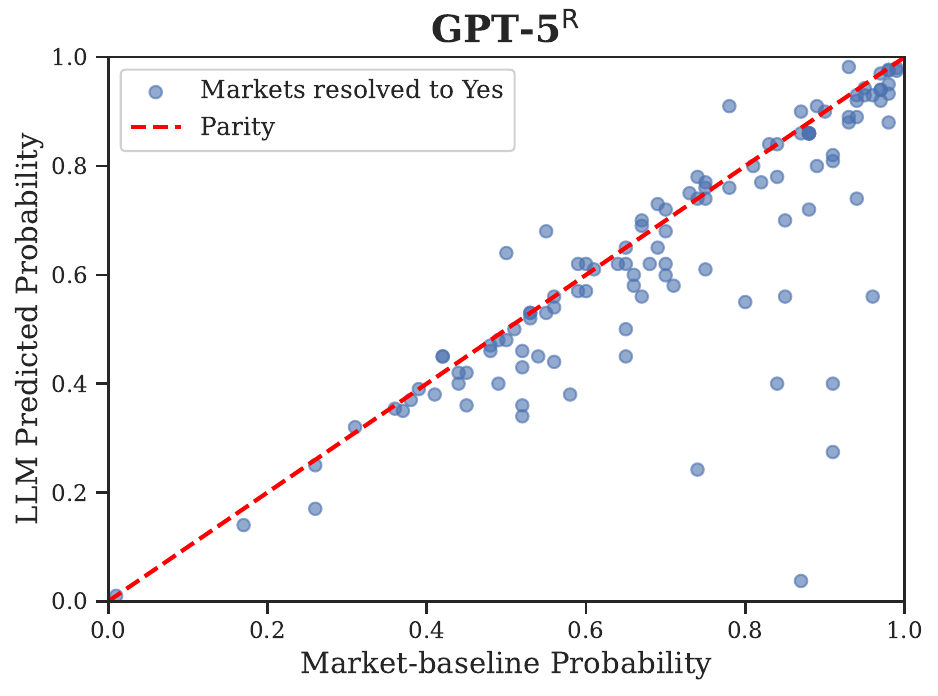}
    \phantomcaption
    \label{fig:gpt5-parity}
  \end{subfigure}
  \hspace{0.05\linewidth}
  \begin{subfigure}[t]{0.35\linewidth}
    \centering
    \includegraphics[width=\linewidth]{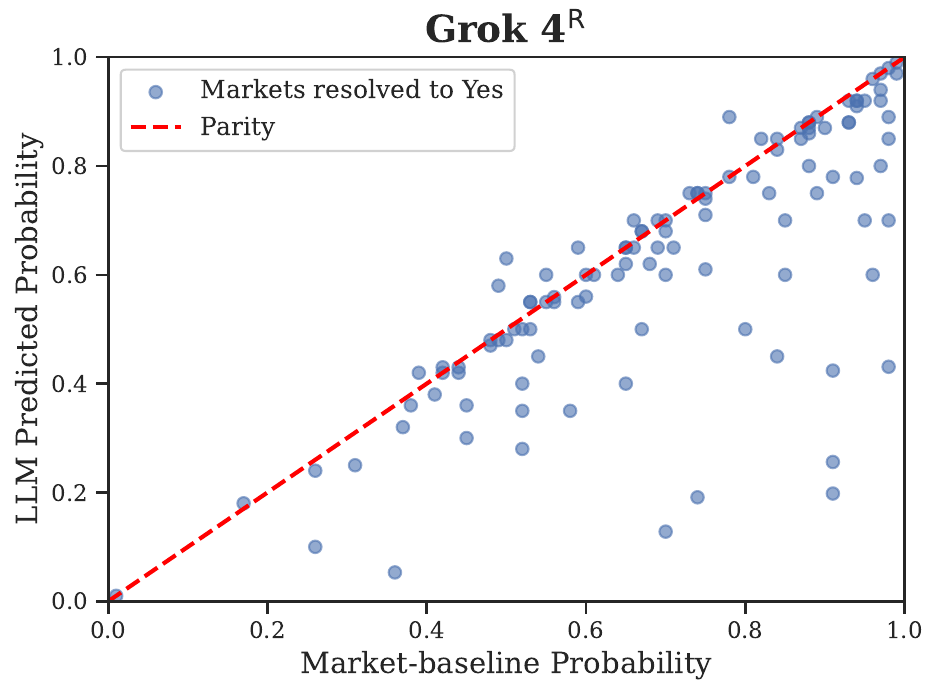}
    \phantomcaption
    \label{fig:grok4-parity}
  \end{subfigure}
  \vspace{0.8em}
  \begin{subfigure}[t]{0.35\linewidth}
    \centering
    \includegraphics[width=\linewidth]{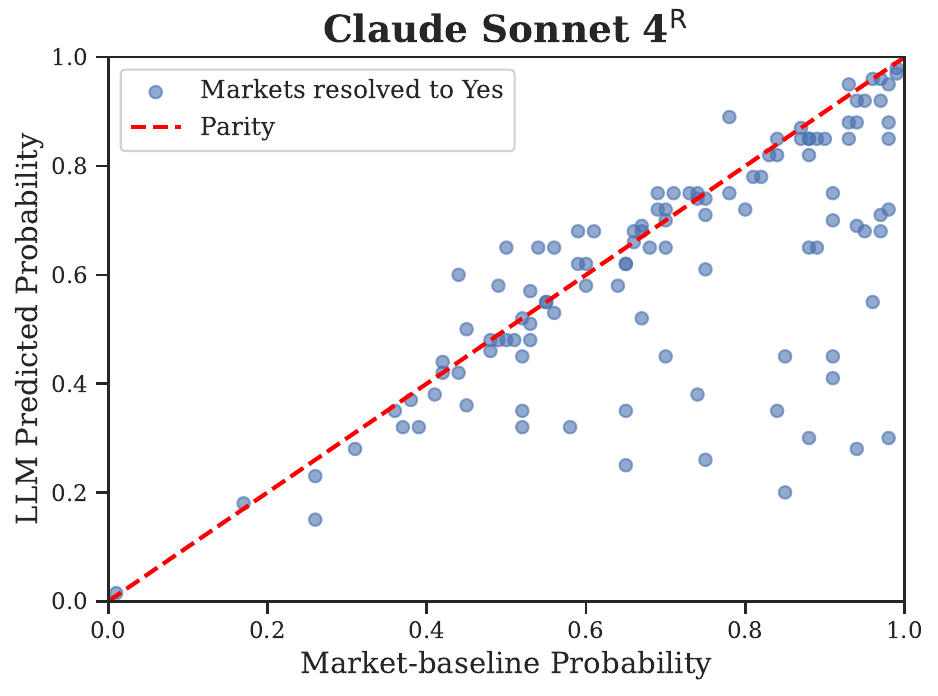}
    \phantomcaption
    \label{fig:claude4-parity}
  \end{subfigure}
  \hspace{0.05\linewidth}
  \begin{subfigure}[t]{0.35\linewidth}
    \centering
    \includegraphics[width=\linewidth]{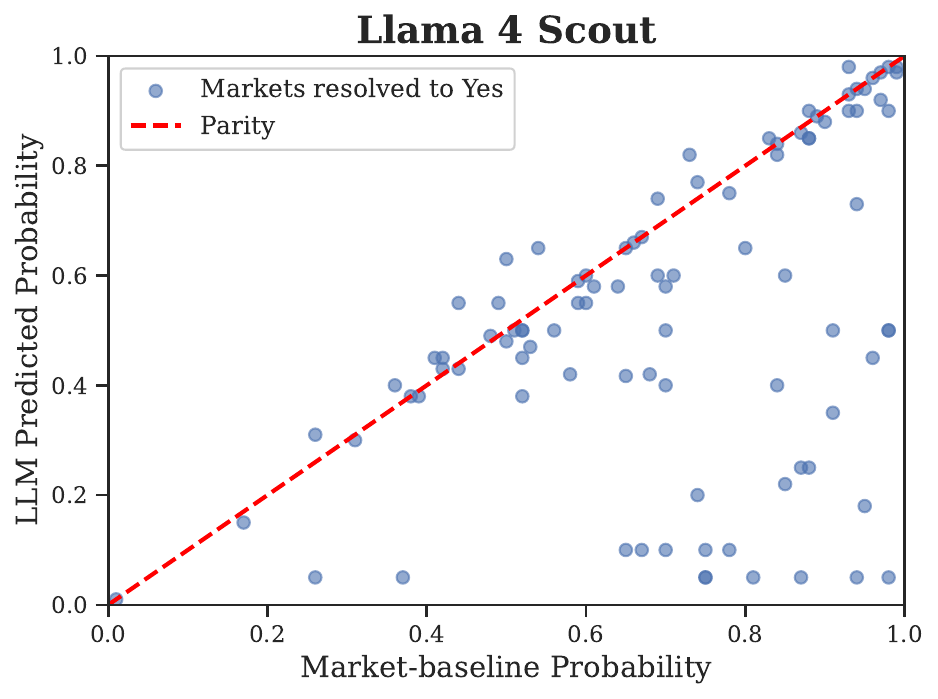}
    \phantomcaption
    \label{fig:llama-scout-parity}
  \end{subfigure}
  \vspace{-5mm}
  \caption{\textbf{LLM-predicted probabilities versus market baselines for markets resolved to \texttt{Yes}.} Each scatter plot compares model-predicted probabilities with market-implied probabilities. Points along the diagonal indicate cases where the model’s forecast exactly matches the market’s implied probability of a \texttt{Yes} resolution.}
  \label{fig:scout_vs_gpt5}
  \vspace{-3mm}
\end{figure}

\subsection{Granular Analysis of LLM-as-a-Prophet}
\label{sec:granularAnalysis}
The reported probabilities of LLMs compress a rich decision-making process into a single number: two models may yield similar predicted probabilities and scores while relying on dramatically different reasoning processes. To address this, we open the black box and evaluate the process underlying each forecast. Specifically, we adopt an LLM-as-a-judge framework \citep{zheng2023judging} to assess the soundness of reasoning along five critical dimensions: source selection, evidence extraction, reasoning synthesis, reasoning-to-prediction alignment, and recognition of prediction uncertainty. Full prompts and additional details of the evaluation framework are provided in~\cref{app:reasoning}. 

To assess the reliability of the LLM-as-a-judge framework, we conducted a blinded human evaluation on a random subset of 170 predictions made by the LLMs. Human raters scored the same five dimensions using the identical rubric provided to the LLM judge. Overall, human and LLM ratings were closely aligned (mean difference $< 0.5$ out of a maximum of 4, with standard deviation $\approx 0.3$ across all dimensions). Full details of the study design and results are provided in \cref{app:reasoning:rater}.

As shown in \cref{tab:llm_judge_eval}, the models demonstrate broadly comparable performance in source utilization, evidence extraction, and uncertainty analysis. However, substantial disparities emerge in the reasoning synthesis and reasoning-prediction alignment categories. For example, when comparing \modelshort{gpt-5} with \modelshort{gemini-2.5-flash}, the differences in reasoning synthesis (0.95) and reasoning-to-prediction alignment (0.30) are markedly larger than the relatively minor gaps in source use (0.12), evidence extraction (0.00), and uncertainty analysis (0.20). Due to the significant deficiencies in those two key categories, the other models demonstrate a significant gap in prediction quality relative to \modelshort{gpt-5}, supporting the findings in \cref{tab:main-evaluation-result}. These findings indicate a potential ceiling effect: once models attain proficiency in retrieval and evidence extraction, further performance gains depend primarily on advances in higher-order reasoning rather than incremental improvements in information access.

\begin{table}[h]
\centering
\begin{tabular}{lcccccc}
\toprule
\textbf{LLM} & \textbf{Sources} & \textbf{Evidence} & \textbf{Reas. Synth.} & \textbf{Align.} & \textbf{Uncert.} & \textbf{Average Score} \\
\midrule
\modelshort{gpt-5}$^\texttt{R}$      & \textbf{3.69} & \textbf{3.66} & \textbf{4.14} & \textbf{3.97} & \textbf{3.94} & \textbf{3.88} \\
\modelshort{gemini-2.5-flash}$^\texttt{R}$      & 3.57 & \textbf{3.66} & 3.19 & 3.67 & 3.74 & 3.57 \\
\modelshort{grok-4}          & 3.40 & 3.51 & 3.33 & 3.48 & 3.66 & 3.48 \\
\modelshort{claude-sonnet-4}$^\texttt{R}$     & 3.53 & 3.47 & 2.93 & 3.39 & 3.75 & 3.41 \\
\modelshort{llama-4-scout}     & 2.97 & 2.88 & 2.29 & 2.37 & 2.87 & 2.68 \\
\bottomrule
\end{tabular}
\vspace{2mm}
\caption{\textbf{LLM performance on reasoning evaluation criteria across dataset events.} Each dimension is scored on a standardized 5-point scale, where 1 and 5 indicate poor and excellent performance, respectively.  Average scores are presented for each model, with \textbf{bold} values indicating the best-performing model for each criterion. Models are ordered by descending overall average score. 
}
\label{tab:llm_judge_eval}
\end{table}

\enlargethispage{\baselineskip}
\section{Conclusion}
\label{sec:conclusion}
This paper systematically evaluates the prospects and challenges of using LLMs to forecast future events, a paradigm coined LLM-as-a-Prophet. Towards that end, we build \prophet, a benchmark that allows modularized analysis about various aspects of existing LLMs' predictive intelligence. Our thorough experiments demonstrates the promise of LLM-as-a-Prophet reflected in the small forecasting loss, calibration error, strong reasoning synthesis and alignment of frontier models. However, we also identify key bottlenecks and highlight avenues for further progresses, such as better curation of context sources, more accurate internalization of knowledge and improving forecasts near events' resolution.

\bibliography{preprint}
\bibliographystyle{iclr2026_conference}

\newpage
\appendix
\section*{\centering \Large $\clubsuit$ Appendix: Table of Contents}
\vspace{2mm}
\begin{enumerate}[leftmargin=*, itemsep=3mm, label=\textcolor{MidnightBlue}{\textbf{\Alph*.}}]
    \item \textbf{\hyperref[app:pipeline]{Prophet Arena Pipeline Details}} \dotfill Page \pageref{app:pipeline}
    \item \textbf{\hyperref[app:technical-details]{Technical Details}} \dotfill Page \pageref{app:technical-details}
    \item \textbf{\hyperref[app:additional-experiments]{Additional Experiment Results}} \dotfill Page \pageref{app:additional-experiments}
    \item \textbf{\hyperref[app:case_studies]{Case Studies}} \dotfill Page \pageref{app:case_studies}
    \item \textbf{\hyperref[app:prompts]{Prompts}} \dotfill Page \pageref{app:prompts}
\end{enumerate}

\newpage 
\section{Definitions and Prophet Arena Pipeline Details}
\label{app:pipeline}
\subsection{Event Extraction}
\label{app:event_extraction}
\prophet\, sources unresolved events from Kalshi, a live prediction platform with events spanning across finance, sports, politics, sports, and entertainment. 
To ensure informativeness and comparability, we filter by \emph{Popularity} (volume/liquidity/volatility), \emph{Diversity} (domain balance), and \emph{Recurrence} (repeated formats). \prophet\, periodically retrieves 20 unresolved events each day at 12 AM (UTC).  \par 
\paragraph{Event.} Defined formally, let  $\{E_i\}_{i\in[K]}$ denote the set of evaluated \emph{events}. An \textit{event} is the overarching question or subject concerning a future real-world occurrence. It serves as a high-level container for one or more tradable \textit{markets}. In many prediction markets, an event itself is \textbf{not} a tradable asset; rather, it sets the context, scope, and resolution criteria for the markets that fall under it. \par 

\begin{itemize}[leftmargin=2em]
    \item \textbf{Example 1:} ``Who will win the 2025-26 NBA Championship?"
    \item \textbf{Example 2:} ``Which individuals will President Trump officially meet in 2025?"
\end{itemize}

\paragraph{Market.} Each event $E_i$ contains \emph{markets} $\{M_{ij}\}_{j\in[N_i]}$. A \textit{market} is a specific, tradable proposition under an event that resolves to a definitive \texttt{Yes} or \texttt{No} outcome. Each market represents a potential, verifiable answer to the event's overarching question. For a given market, a \textit{\texttt{Yes} contract} is a 0-1 random variable that achieves value 1 if the \texttt{Yes} outcome is realized, and 0 otherwise. A \textit{``NO'' contract} is defined similarly, and always pays out in the opposite direction as the ``YES'' contract. \par 

\begin{itemize}[leftmargin=2em]
    \item \textbf{Under Event 1, a market could be:} 
    
    ``The Boston Celtics will win the 2025 NBA Championship."
    \item \textbf{Under Event 2, a market could be:} 
    
    ``President Trump will officially meet with Emmanuel Macron in 2025."
\end{itemize}

\paragraph{Event Resolution.} An event $E_i$ resolves at time $\tau_{i}$ with realized outcome indicator $o_{ij}\in\{0,1\}$, where $o_{ij}=1$ means the corresponding market $M_{ij}$ of the event is realized (\texttt{Yes}). When the event index $i$ is clear from context, we often drop the subscript $i$ and write $M_j, o_j$.

\subsection{Prediction Context}
\label{app:prediction_context}
\paragraph{Prediction scheduling.} For each event $E_i$, the benchmark specifies a finite set of pre-resolution forecast times (horizons) $\mathcal{T}_{i}\subset(-\infty,\tau_{i})$. 
We construct $\mathcal{T}_i$ by placing each subsequent forecast halfway between the current forecast time and the event close time $\tau_i$.
Let the first forecast time be $t_i^{(0)}<\tau_i$ and define the initial gap $\Delta_i^{(0)} := \tau_i - t_i^{(0)} > 0$.
For $k\ge 0$,
\[
t_i^{(k+1)} \;=\; \frac{t_i^{(k)} + \tau_i}{2}
\quad\Longleftrightarrow\quad
t_i^{(k)} \;=\; \tau_i - 2^{-k}\,\Delta_i^{(0)},
\;\;\;
\Delta_i^{(k+1)} \;=\; \frac{\Delta_i^{(k)}}{2}.
\]
To avoid excessive clustering near $\tau_i$, we enforce a minimum time gap $\delta_{\min}>0$ between the last forecast and the close time.

\paragraph{Context construction.} For each event $E_i$ and forecasting time $t \in \mathcal{T}_{i}$ we construct a curated context $C_{i,t}$ shared across models; this isolates forecasting ability in \(p_{ij,t}\) from retrieval variability. The context $C_{i,t}$ consists of two components: relevant news sources and market data. The relevant news sources are retrieved by LLM searchers. The prompt for search is shown below. The market snapshot is fetched from Kalshi API, and contains three fields for each market: \texttt{last\_price} (price of the last transaction), \texttt{yes\_ask} (asking price for buying \texttt{Yes}), \texttt{no\_ask} (asking price for buying \texttt{No}). From that snapshot we extract the \emph{implied probability} $q_{ij,t}\in[0,1]$ for a \texttt{Yes} contract at time $t$ (with \texttt{No} priced at $1-q_{ij,t}$).\footnote{Implied probabilities are reverse-engineered from contract prices; transaction fees may cause \texttt{YES/NO} prices not to sum to 1.} \par 

At each forecasting time $t_i \in \mathcal{T}_i$, LLM searchers will be dispatched and live market snapshot will be retrieved. Together, relevant news sources from LLM searchers and market snapshots from Kalshi will serve as the prediction context.\par 

Importantly, \prophet\, is \emph{searcher-agnostic}: the search component is an independent, pluggable module and adding new searchers does not change the forecasting protocol or the scoring of $p_{ij,t}$. \prophet\ actively updates to include new LLM searchers. In this paper, we instantiate a single LLM searcher using \modelshort{gpt-4o} with web search enabled and all experiments use that configuration (\cref{app:search_prompt}).

\subsection{Probabilistic Forecasting to Account for Uncertainty}
\label{app:prediction}
\paragraph{LLM-predicted Probabilities.} Given $(E_i,M_{ij},C_{i,t},t)$, a model must output an \emph{LLM probability},
\[
p_{ij,t}\in[0,1],
\]
interpreted as its belief that $M_{ij}$ will realize \texttt{Yes} at time $t$, accompanied by a natural-language rationale (logged for analysis but not used in scoring). The prediction prompts are documented in \cref{app:prediction_prompt}.

\paragraph{Implied Probabilities.} At each $t\in\mathcal{T}_{i}$, a \texttt{Yes} (resp. \texttt{No}) contract is valued at $q_{ij,t}$ (resp.\ $1-q_{ij,t}$). The \textit{implied probability} $q_{ij}$ represents the (human) market-consensus belief that the \texttt{Yes} outcome will come true. \par 

\subsection{Resolution and Evaluation}
\label{app:resolution}

\paragraph{Edge and Utility.} We denote a (yes) \textit{edge} $e_{ij} := \frac{p_{ij}}{q_{ij}}$ as the likelihood-ratio between the LLM-predicted and implied probabilities. A larger edge signals the LLM to be more confident (than the market) that a market will be realized. Similarly, we define the (no) edge to be $\tilde{e}_{ij} := \frac{1-p_{ij}}{1-q_{ij}}$. \par We assume that the \textit{price} of a single \texttt{Yes} contract simply equals the implied probability $q_{ij}$, and the price of a single \texttt{No} contract is thus $1 - q_{ij}$.\footnote{In practice, (1) we actually ``reverse-engineer'' the implied probabilities from market contract prices, and (2) the prices of \texttt{Yes/No} contracts might not sum to 1 due to transaction fees taken by the exchange.} In the sequel, our strategy is limited to buying contracts (taking a long position). But all contracts can be purchased in fractional amounts. \par Our simulated trading policies (used only for \emph{relative} metrics) considers the utility function with risk-aversion hyperparameter $\gamma \in [0, 1]$. It maps any wealth (i.e. payoff) to a utility.

\paragraph{Prediction Evaluation.} After the close time $\tau_i$, we will retrieve the outcomes of event $E_i$ from Kalshi. 
For each market $M_{ij}$ and each horizon $t\in\mathcal{T}_{i}$, we score $(p_{ij,t},o_{ij})$ with \emph{absolute} proper scoring rules (e.g., Brier-based) and, separately, evaluate \emph{relative} Average Return by simulating trades against $q_{ij,t}$ under $U_\gamma$. Scores are then aggregated across $j$, $i$, with the specific calculations detailed in the subsequent evaluation section. 

The below table summarizes the notations that will appear in the later math expressions:

\bgroup
\def\arraystretch{1.5}
\begin{tabular}{p{1in}p{3.75in}}
$\displaystyle E_i,$ & $i \in [K]$, the $i$-th event in our evaluation set.\\
$\displaystyle M_{ij} \:(M_j)$ & $j \in [N_i]$, the $j$-th market of the $i$-th event, we drop the $i$ subscript when the market is obvious (same for below). \\
$\displaystyle p_{ij} \:(p_j)$ & the LLM-predicted probability that $M_{ij}$ will realize.\\
$\displaystyle q_{ij} \:(q_j)$ & the market implied probability that $M_{ij}$ will realize.\\
$\displaystyle e_{ij}/ \tilde{e}_{ij} \:(e_j/\tilde{e}_j)$ & the (yes/no) edge of $M_{ij}$.\\
$\displaystyle o_{ij} \:(o_j)$ & the indicator of whether $M_{ij}$ is realized.\\
$\displaystyle U_\gamma(w)$ & the utility function with risk-aversion hyperparameter $\gamma \in [0, 1]$. It maps any wealth (i.e. payoff) to a utility.

\end{tabular}
\egroup

\subsection{Comprehensive List of Evaluated LLMs}
\label{app:full-llm-list}

\begin{table}[!th]
    \centering
    \setlength{\tabcolsep}{6pt}
    \renewcommand{\arraystretch}{1.15}
    \begin{tabular}{l l c c}
    \hline
    \textbf{LLM} & \textbf{Citation} & \textbf{Open Weight?} & \textbf{Default Reasoning} \\
    \hline
    GPT-5 & \citet{gpt5} & No & High \\
    o3 & \citet{o3} & No & High \\
    o3-Mini & \citet{o3} & No & Medium \\
    o4-Mini & \citet{o3} & No & High \\
    Gemini 2.5 Pro & \citet{comanici2025gemini} & No & Enabled \\
    Gemini 2.5 Flash & \citet{comanici2025gemini} & No & Enabled \\
    Grok-4 & \citet{xai-grok4} & No & Enabled \\
    Grok-3-Mini & \citet{xai-grok3} & No & Enabled \\
    GPT-4.1 & \citet{gpt41} & No & N/A \\
    Claude Sonnet 4 & \citet{claude4} & No & Enabled \\
    Kimi-K2 & \citet{team2025kimi} & \href{https://huggingface.co/moonshotai/Kimi-K2-Instruct}{Yes} & N/A \\
    GPT-4o & \citet{gpt4o} & No & N/A \\
    Llama 4 Maverick & \citet{meta2025llama4} & \href{https://huggingface.co/meta-llama/Llama-4-Scout-17B-16E-Instruct}{Yes} & N/A \\
    Llama 4 Scout & \citet{meta2025llama4} & \href{https://huggingface.co/meta-llama/Llama-4-Maverick-17B-128E-Instruct}{Yes} & N/A\\
    DeepSeek-V3 & \citet{liu2024deepseek} & \href{https://huggingface.co/deepseek-ai/DeepSeek-V3}{Yes} & N/A \\
    DeepSeek-R1 & \citet{guo2025deepseek} & \href{https://huggingface.co/deepseek-ai/DeepSeek-R1}{Yes} & Enabled \\
    Qwen3-235B & \citet{yang2025qwen3} & \href{https://huggingface.co/Qwen/Qwen3-235B-A22B-Thinking-2507}{Yes} & N/A \\
    Gemini 2.0 Flash & \citet{gemini20} & No & N/A\\
    Gemini 2.0 Flash (Lite) & \citet{gemini20} & No & N/A\\
    \hline
    \end{tabular}
    \vspace{3mm}
    \caption{\textbf{Comprehensive list of LLMs evaluated in} \prophet~(as of submission time).}
\end{table}

\textbf{Remark.} The column \textbf{Default Reasoning} specifies the reasoning configuration used when a model is referenced without qualifiers (e.g. \modelshort{gpt-5}, or when we write \modelshort{gpt-5}$^\texttt{R}$ in~\cref{tab:main-evaluation-result}). In the full evaluation table (\cref{tab:full-eval-results}), some models appear multiple times under different reasoning settings; each such variant is treated as a distinct model (e.g., GPT-5 (Minimal)).

Reasoning configuration is inherently model-dependent. For ``hybrid reasoning'' models that allow toggling between thinking and non-thinking modes, the configuration is specified as either \emph{enabled} or \emph{disabled}. For models that expose explicit control over reasoning effort, the configuration is expressed in levels (e.g., \emph{minimal}, \emph{medium}, \emph{high}).
\newpage
\section{Technical Details}
\label{app:technical-details}
\subsection{Further Intuitions behind Brier Score} \label{app:brier-details}

In \prophet, the set of markets $\{M_{ij}\}_{j=1}^{m_i}$ under an event $E_i$ need not be mutually exclusive. This has two implications: (1) the realized outcome vector $\mathbf{o}_i = (o_{i1}, \ldots, o_{im_i})'$ may contain multiple ones rather than being strictly one-hot, and (2) the predicted probabilities $\{p_{ij}\}_j$ may sum to more than one. Our Brier score formulation remains robust to these cases because the score is always evaluated at the \emph{market level}, where each individual market is binary -- resolving either to \texttt{Yes} or \texttt{No}, but not both. If we pool all markets across all events and relabel them $M_1, \ldots, M_k$, the standard binary Brier score is simply
\begin{equation}
\label{eq:brier-og}
BS = \frac{1}{k} \sum_{j=1}^k (p_j - o_j)^2.
\end{equation}

Our event-level definition in \cref{subsec:accuracy} can be viewed as a \emph{weighted} version of \cref{eq:brier-og}. Since events vary greatly in the number of associated markets, directly pooling them would let large events dominate the metric. To mitigate this, we assign each market a weight $w_{ij} = 1/m_i$, inversely proportional to the number of markets in its event. This yields the final form:
\begin{equation*}
BS = \frac{1}{\sum_i \sum_j w_{ij}} \sum_{i=1}^N \sum_{j=1}^{m_i} w_{ij} (p_{ij} - o_{ij})^2
= \frac{1}{N} \sum_{i=1}^N \left( \frac{1}{m_i} \sum_{j=1}^{m_i} (p_{ij} - o_{ij})^2 \right).
\end{equation*}

This weighting ensures that each event contributes equally, regardless of its size, while still respecting the binary resolution of individual markets.


\subsection{Empirical Estimation of Calibration Error}
\label{app:ece-details}

In~\cref{subsec:calibration}, we introduced the formal definition of expected calibration error (ECE) in terms of true conditional probabilities. In practice, this definition cannot be computed exactly, as the conditional terms $\mathbb{P}(o_j = 1 \mid p_j = \tilde{p}_i)$ are not directly observable. Instead, the standard approach in the applied literature is to approximate ECE via a binned empirical estimate.

Formally, let $\{(p_j, o_j)\}_{j=1}^m$ denote a set of predicted probabilities and realized outcomes. We first partition the unit interval $[0,1]$ into $B$ disjoint bins $\{I_b\}_{b=1}^B$, and assign each prediction $p_j$ to its corresponding bin. Let $M_b = \{j : p_j \in I_b\}$ be the index set of predictions falling into bin $b$, and $m_b = |M_b|$ its size. Define

\begin{equation*}
\hat{p}_b = \frac{1}{m_b}\sum_{j \in M_b} p_j, 
\quad 
\hat{o}_b = \frac{1}{m_b}\sum_{j \in M_b} o_j ,    
\end{equation*}

as the average predicted probability and empirical frequency of \texttt{Yes} outcomes in bin $b$, respectively. The \textbf{empirical expected calibration error} is then given by
\begin{equation}
\widehat{ECE} = \frac{1}{m} \sum_{b=1}^B m_b \cdot\big| \hat{o}_b - \hat{p}_b \big| .
\end{equation}

Intuitively, $\widehat{ECE}$ measures the weighted average discrepancy between empirical accuracy and average predicted probability across bins, with weights proportional to bin counts. Throughout our experiments, all reported calibration results correspond to this empirical version (with $B = 10$).

\subsection{Brier Score vs ECE.}
\label{app:ece-differ-brier}
As discussed in~\cref{subsec:metrics}, calibration errors and Brier scores measure different aspects of forecasts. We start with a simple example to illustrate that a better-calibrated forecast  can have a worse  Brier score.  Suppose there are two markets $M_1, M_2$ with ground-truth probability $p^*_1 = 0.9, p^*_2 = 0.1$. Alice's prediction is $p^A_1 = 1, p^A_2 = 0$ whereas Bob's prediction is $p^B= 0.5$ for both markets. It is not difficult to see that Alice has a better/smaller Brier score, but she turns out to have a worse/larger ECE. Specifically, since Alice has two different probability predictions, her ECE is $\frac{1}{2}(|p^*_1 - p^A_1| + |p^*_2 - p^A_2|) = 0.1$ whereas Bob has a single probability prediction with the ECE equaling $\frac{1}{2}\big|p^*_1  + p^*_2 - 2 p^B| \big|  = 0$. That is, Bob's prediction of $0.5$ probability  is perfectly calibrated and reliable as, conditioned on this probability, events' average probability is indeed $0.5$.

\textbf{A forecaster  with smaller ECE but  larger Brier score induces better risk-adjusted decision making.} Consider an event $E$ with binary outcome $o\in \{0, 1\}$ and a ground-truth probability $p^* = 0.5$ to have $o=1$. Suppose forecaster $A$ predicts probability $p_A:= \Pr(o = 1)$ which correlates with the realized outcome $o$ in the following way
\[
(p_A, o)) = \begin{cases}
    (1, 1) , & \text{with joint probability } 0.5(1-\epsilon)   \\
    (0, 0) ,  & \text{with joint probability }   0.5(1-\epsilon) \\
    (1, 0) , & \text{with joint probability } 0.5 \epsilon    \\
    (0, 1) , & \text{with joint probability }   0.5 \epsilon  \\  
\end{cases} 
\] 
where $\epsilon$ is a small value. In other words, forecaster $A$ predicts $p_A = 1$ or $p_A = 0$, each with probability $0.5$ and with $\epsilon$ fraction of errors. These are   extremal and uncalibrated forecasts but is almost always correct. It is easy to verify that the Brier score of $p_A$ is $0.5 \epsilon  \times (1-0)^2 + 0.5 \epsilon  \times (1-0)^2 = \epsilon$, whereas expected calibration error (ECE) is $0.5 \times   [1 - (1-\epsilon)]^2 + 0.5 \times   [1 - (1-\epsilon)]^2  = 0.5 \epsilon^2$. 

Next consider forecaster $B$ predicts probability $p_B:= \Pr(o = 1)$ which correlates with the realized outcome $o$ in the following way
\[
(p_B, o)) = \begin{cases}
    (2/3, 1) , & \text{with joint probability } 0.5    \\
    (2/3, 0) ,  & \text{with joint probability }   0.25 \\
    (0, 0) , & \text{with joint probability }   0.25  \\  
\end{cases} 
\] 
That is, forecaster $B$ predicts $p_B = 2/3$ for $75\%$ of the time, and predicts $p_B = 0$ otherwise. It is easy to see that $p_B$ is perfectly calibrated with $ECE = 0$, but has  Brier score equal $0.75 \times [\frac{2}{3} (2/3 - 1)^2 + \frac{1}{3} (2/3 - 0)^2] = 1/6$. 

Clearly, $p_B$ is more well-calibrated than $p_A$ but has much worse Brier score, especially  when $\epsilon$ is small. However, we show that $p_B$ is better for risk-averse decision making,  than $p_A$. Suppose a decision maker has $1$ unit of fund and would like to buy contracts for   event $E$'s outcome. Suppose the current market price for  the contract $o=1$ and $o=0$ are both $0.5$ per contract.   Adopting a classic risk-averse utility function, we assume her utility $U(x) = \log (x)$ (i.e., $x$ amount of return means $U(x)$ utility to her). 

If the decision maker adopts forecaster $A$, since $A$ makes extremal forecasts, even under risk-averse utility $\log(x)$ her strategy is to spend all the  $1$ unit of fund to buy two units of contract   $o=1$ whenever $p_A= 1$, and buy $2$ units of $o=0$ otherwise. With $\epsilon$ probability of losing all her fund, her expected \emph{risk-adjusted} utility is thus $ (1-\epsilon) U(2) + \epsilon U(0)$, which goes to $-\infty$ when $U(x) = 0$, regardless how small $\epsilon$ is. 

If the decision maker adopts the perfectly calibrated forecaster $B$, she would spend all fund to buy $o=0$ contract when $p_B = 0$, but split her $1$ unit of fund to buy $2B_1$ of $o=1$ contracts and $2B_0$ of $o=0$ contracts, so as to maximize her risk-adjusted utility as $ \max_{B_0 + B_1 = 1} 2/3 U(2B_1) + 1/3 U(2B_0)$. Simple calculation shows that optimal $B_1 = 2/3$, leading to optimal risk-adjust utility $2/3 U(4/3) + 1/3 U(2/3) = \log(\frac{32}{27})/3$, which is strictly larger than her initial utility $\log 1 = 0$.    

\textbf{Remarks. } A few remarks are worthwhile to mention about this example. First, forecaster $p_A$ with a small $\epsilon$ will be better for a   risk-neural decision maker. However, in this case, while the decision maker will have larger expected utility, she will have to bear the possibility of losing all her $1$ unit of fund, though also enjoy the possibility of doubling her fund to $2$ units. This is often not desirable in real-world applications, such as financial decisions in which risk control is extremely important. In such cases, more well-calibrated forecaters are preferred. Second,  our example has extremely negative utility   $-\infty = U(0)$. This helps us to   cleanly highlight the underlying conceptual message, but simple tweak to the utility function can remove that extreme situation yet arrive at the same conclusion.

\subsection{Differences between Brier Scores and Market Returns}
\label{app:brier-vs-return} 
We provide a simple example to prove the following fact. 
 \begin{fact}\label{fact:brier-vs-return} \it   There exist a binary prediction market with two forecaster $A,B$ such that $A$ has strictly higher market return than $B$ but has strictly worse/higher Brier score.    
\end{fact}
\begin{proof}

 Suppose we bet on a single event with binary outcomes, with ground-truth probability of \texttt{Yes} being $0.6$ (for the event to be realized), and prediction market price $0.5$. Consider two different probabilistic predictions of this event’s \texttt{Yes} realization: model A predicts $0.45$ and model B predicts $0.9$.

The expected Brier Score of A is:
\begin{align*}
1 - \big[ 0.6 \cdot (0.45 - 1)^2 + 0.4 \cdot (0.45 - 0)^2 \big] 
&= 1 - 0.2625 = 0.7375
\end{align*}

The expected Brier Score of B is:
\begin{align*}
1 - \big[ 0.6 \cdot (0.9 - 1)^2 + 0.4 \cdot (0.9 - 0)^2 \big] 
&= 1 - 0.33 = 0.67
\end{align*}

So A has a higher Brier Score. However, because A predicts $0.45$, lower than the $0.5$ prediction market price, A will short \texttt{Yes} (or equivalently, buy \texttt{No}) at $0.5$. Meanwhile, B predicts $0.9$, much higher than the prediction market price, so B will buy \texttt{Yes}. Respectively, A and B’s expected returns will be:
\begin{align*}
0.6 \cdot (-1 + 0.5) + 0.4 \cdot (0.5) &= -0.1 \\
0.6 \cdot (1 - 0.5) + 0.4 \cdot (-0.5) &= 0.1
\end{align*}
Therefore, A has a higher Brier score, but lower returns.
    
\end{proof}
 This example uncovers a key difference between the two metrics. The Brier Score measures how close a prediction is to the ground truth and, importantly, has nothing to do with market prices. Since A’s prediction above is closer to the ground truth, it receives a higher Brier Score. However, returns on the market are not only driven by the true probability but also by the market price. Therefore, even though B’s prediction is exaggerated, it lies on the correct side of the market mispricing (buying ``Yes'' when the outcome is more likely than price suggests), thereby achieving higher returns.

\subsection{A Unified Framework for Utility-Maximizing Betting Strategy}
\label{app:unified-betting}

In \cref{subsec:market-return}, we introduced \textbf{Average Return} under a simple strategy: bet the full budget on \texttt{Yes} whenever $p > q$, and on \texttt{No} otherwise, where $p$ is the model’s predicted probability and $q$ is the market price of a \texttt{Yes} contract. Here we generalize this idea by developing a unified framework for betting strategies in binary markets. This framework serves two purposes:
\begin{enumerate}[leftmargin=2em]
    \item It formalizes betting as a utility-maximization problem, allowing us to flexibly encode different risk preferences.
    \item It shows that the simple strategy used in the main text is a special case of this unified framework, corresponding to the risk-neutral setting.
\end{enumerate}

\textbf{Contracts and market prices.} Each binary market resolves to either \texttt{Yes} or \texttt{No}. A share of a \texttt{Yes} contract pays \$1 if the outcome is \texttt{Yes} and \$0 otherwise; the price of this contract is denoted $q$. Symmetrically, a share of a \texttt{No} contract costs $1-q$ and pays \$1 if the outcome is \texttt{No}. These prices are often interpreted as market-implied probabilities~\citep{wolfers2006interpreting}, though our framework does not rely on this interpretation.

\textbf{Utility function.} Let $U: \mathbb{R}_{\ge 0} \to \mathbb{R}$ denote a utility function mapping a payoff $w$ to its perceived value (satisfaction) $U(w)$. We focus on the class of constant relative risk aversion (CRRA) utilities:
\begin{equation}
\label{eq:crra}
    U_\gamma(w)=
\begin{cases}
\dfrac{w^{1-\gamma}}{\,1-\gamma\,}, & 0\le\gamma<1,\\[6pt]
\log w, & \gamma=1
\end{cases}
\end{equation}
where $\gamma \in [0, 1]$ indexes risk aversion. At $\gamma=0$, $U_\gamma$ is linear (risk-neutral); at $\gamma=1$, it reduces to $\log w$ (log utility, risk-averse); and for $\gamma \in (0,1)$, it interpolates smoothly between the two. \cref{fig:crra-utility} visualizes representative cases.

\begin{figure}
    \centering
    \includegraphics[width=0.6\linewidth]{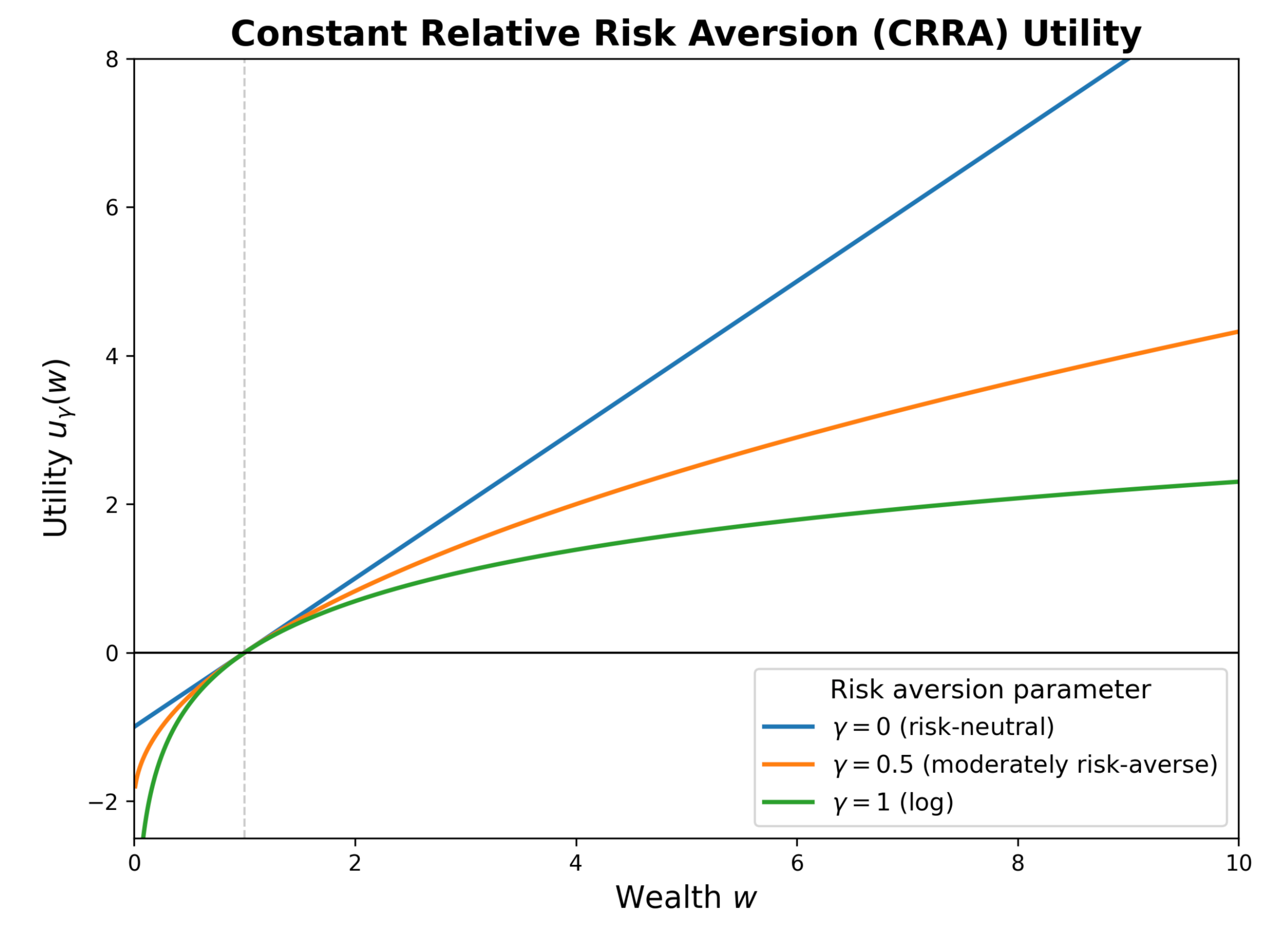}
    \caption{\textbf{Visualization of CRRA utility functions.}}
    \label{fig:crra-utility}
\end{figure}

\textbf{Budget allocation as optimization.} Fix a unit budget of \$1 for each market. Let $a_Y$ and $a_N$ denote the amounts allocated to \texttt{Yes} and \texttt{No} contracts, with $a_Y + a_N = 1$. Purchasing $a_Y$ dollars of \texttt{Yes} contracts yields $a_Y/q$ shares (i.e., the pay off is $a_Y/q$ if the outcome is \texttt{Yes}); similarly, $a_N/(1-q)$ shares of \texttt{No} pay off if the outcome is \texttt{No}. Given belief (predicted probability) $p$ and risk preference $\gamma$, the optimal allocation is the solution to
\begin{equation}
\label{eq:budget-optim}
    \max_{a_{Y}, a_N \geq 0 \: :\:  a_Y + a_N = 1} \left[ p \cdot U_\gamma\left(\frac{a_Y}{q}\right) + (1-p)\cdot U_\gamma\left(\frac{a_N}{1-q}\right) \right].
\end{equation}
This objective is the expected utility of betting under $p$.

\textbf{Closed-form solutions.} For $\gamma > 0$, the optimization admits a closed-form solution:
\begin{align} \label{eq:budget-solution} \quad a_Y^* \equiv a_Y^*(\gamma) = \frac{q^{\,1-\frac1\gamma}\;p^{\,\frac1\gamma}}{q^{\,1-\frac1\gamma}\;p^{\,\frac1\gamma} + (1-q)^{\,1-\frac1\gamma}\;(1-p)^{\,\frac1\gamma}} ,\quad \quad a_N^* \equiv a_N^*(\gamma) = 1 - a_Y^*, 
\end{align}
At $\gamma = 0$, we take the limit $\gamma \to 0^+$, yielding
\begin{equation*}
a_Y^*(0) \equiv \lim_{\gamma \to 0} a_Y^*(\gamma) = 
\begin{cases}
1, & \text{if} \:\: p > q,\\
0, & \text{if} \:\: p \leq q
\end{cases}.
\end{equation*}
which recovers the strategy used in the main text (\cref{subsec:market-return}). At the other extreme, when $\gamma = 1$, the optimal strategy is to allocate $a_Y^* = p$ and $a_N^* = 1-p$, i.e., to bet proportionally to one’s own probabilities, independent of the market price. Overall, this unified framework highlights how different betting strategies arise from different risk preferences. The ``all-in'' rule in the main text is simply the risk-neutral optimum.

\subsection{Market Returns for a Calibrated Predictor}
\label{app:calibration-and-return}

In~\cref{sec:evals}, we introduced the evaluation pipeline in \prophet, including how market return is calculated based on market decisions (buying either the YES or NO contract), which are derived from predicted and market-implied probabilities ($p_{j}$, $q_{j}$) as well as the realized outcome $o_{j}$ under the following decision rule: \textit{buy YES contract whenever $p_{j} > q_{j}$, and buy NO otherwise.}

Here we further make the connection between a (perfectly) calibrated predictor and \textbf{market returns conditioning on the type of decision made}. Before stating the main result, the following mild assumptions and definitions are needed.

\begin{assumption}[Data Generation]\label{ass:data-generation}
    We assume that the triples $(o_j, p_j, q_j)$ are drawn i.i.d. from some joint probability distribution $\mathcal{D}$. The expectation $\mathbb{E}[\cdot]$ below is taken over $\mathcal{D}$.
\end{assumption}

\begin{assumption}[Sufficiency of Calibrated Predictors] 
\label{ass:sufficiency}
Perfect calibration implies that $\forall c \in [0, 1]$,
\begin{equation*}
    \mathbb{E}[o_j \mid p_j = c] = 1 \cdot \mathbb{P}[o_j = 1\mid p_j = c] + 0 \cdot \mathbb{P}[o_j = 0\mid p_j = c] = c.
\end{equation*}
We futher assume that a calibrated predictor is \textbf{sufficient} for the market outcome: the outcome $o_j$ is conditionally independent of the market price $q_j$ given the forecast $p_j$, i.e. $\mathbb{E}[o_j \mid p_j, q_j] \equiv \mathbb{E}[o_j \mid p_j]$
\end{assumption}

\begin{definition}[Symmetric Disagreement]
    We say a predictor is \textbf{symmetric} (against the market) if the probability distribution of $d_j := p_j - q_j$ is \textbf{symmetric about zero}. We let $f(x)$ denote the p.d.f. of $d_j$ and the definition implies (1) $f(x) = f(-x)$ for all $x \in \mathbb{R}$, and (2) $\mathbb{P}[p_j > q_j] = \mathbb{P}[p_j \leq q_j]$. In other words, $d_j$ represents the ``disagreement'' that the predictor perceives over the market, and a \textbf{symmetric} predictor is not consistently more conservative or more aggressive.
\end{definition}

\begin{theorem}[Symmetric Expected Returns for Calibrated Predictors]
\label{thm:prophet-indifference}
Let $R_Y = o_j - q_j$ be the return on a YES contract and $R_N = (1-o_j) - (1-q_j) = q_j - o_j$ be the return on a NO contract. For a predictor that is both \textbf{perfectly calibrated} and \textbf{symmetric} against the market, the expected returns conditioned on the betting decision are equal, i.e.
\begin{equation}
    \mathbb{E}[R_Y \mid p_j > q_j] = \mathbb{E}[R_N \mid p_j < q_j].
\end{equation}
\end{theorem}
\begin{proof}
Let $E_Y = \mathbb{E}[o_j - q_j \mid p_j > q_j]$ and $E_N = \mathbb{E}[q_j - o_j \mid p_j < q_j]$ denote the average returns on YES and NO contracts, respectively. We begin by analyzing the expected outcome $o_j$ under each condition. By the Law of Iterated Expectations and the sufficiency of the calibrated predictor (Assumption~\ref{ass:sufficiency}), we can replace the outcome $o_j$ with the prediction $p_j$ inside the expectation:
\begin{align*}
    \mathbb{E}[o_j \mid p_j > q_j] &= \mathbb{E}\left[\mathbb{E}[o_j \mid p_j, q_j] \mid p_j > q_j\right] = \mathbb{E}[p_j \mid p_j > q_j], \\
    \mathbb{E}[o_j \mid p_j < q_j] &= \mathbb{E}\left[\mathbb{E}[o_j \mid p_j, q_j] \mid p_j < q_j\right] = \mathbb{E}[p_j \mid p_j < q_j].
\end{align*}
Substituting these back into the expressions for $E_Y$ and $E_N$:
\begin{align*}
    E_Y &= \mathbb{E}[p_j \mid p_j > q_j] - \mathbb{E}[q_j \mid p_j > q_j] = \mathbb{E}[p_j - q_j \mid p_j > q_j], \\
    E_N &= \mathbb{E}[q_j \mid p_j < q_j] - \mathbb{E}[p_j \mid p_j < q_j] = \mathbb{E}[q_j - p_j \mid p_j < q_j].
\end{align*}
The theorem is proven if we can show that $\mathbb{E}[p_j - q_j \mid p_j > q_j] = \mathbb{E}[q_j - p_j \mid p_j < q_j]$. Let $d_j = p_j - q_j$ be the disagreement variable. The equality becomes:
$$
\mathbb{E}[d_j \mid d_j > 0] = \mathbb{E}[-d_j \mid d_j < 0] 
$$
We show this using the integral definition of conditional expectation. The symmetric disagreement assumption implies $f(x) = f(-x)$ and $\mathbb{P}(d_j > 0) = \mathbb{P}(d_j < 0)$. Consider the right-hand side:
$$ 
\mathbb{E}[-d_j | d_j < 0] = \frac{\int_{-\infty}^0 -x f(x) \,dx}{\mathbb{P}(d_j < 0)}. 
$$
Applying the change of variable $u = -x$, we have $x = -u$ and $dx = -du$. Therefore
$$ 
\frac{\int_{\infty}^0 -(-u) f(-u) (-du)}{\mathbb{P}(d_j < 0)} = \frac{-\int_{\infty}^0 u f(-u) du}{\mathbb{P}(d_j < 0)} = \frac{\int_0^{\infty} u f(-u) du}{\mathbb{P}(d_j < 0)}. 
$$
Using the symmetry properties $f(-u)=f(u)$ and $\mathbb{P}(d_j < 0)=\mathbb{P}(d_j > 0)$, this is equal to:
$$ 
\frac{\int_0^{\infty} u f(u) du}{\mathbb{P}(d_j > 0)} = \mathbb{E}[d_j | d_j > 0]. 
$$
Thus, we have shown $\mathbb{E}[-d_j | d_j < 0] = \mathbb{E}[d_j | d_j > 0]$, which concludes the proof.
\end{proof}

\subsection{Bootstrap Confidence Intervals} 
Most metrics in this paper can be summarized by the following two steps:
\begin{enumerate}
    \item We first calculate a collection of scores $s_1, \ldots, s_N$ at either the event level or the market level, where $N$ denotes the total number of events or markets.
    \item We then aggregate these scores into a single statistic via a mapping $T:\mathbb{R}^N \to \mathbb{R}$ (typically the arithmetic mean).
\end{enumerate}
To assess the uncertainty of the resulting point estimate, we construct a nonparametric bootstrap confidence interval (CI) at level $(1-\alpha)\%$~\citep{diciccio1996bootstrap}. Specifically, we form a symmetric interval around the point estimate using bootstrap resampling. A pseudo-code for our implementation is given below in~\cref{algo:symmetric-ci}.

\begin{algorithm}
\caption{Symmetric Nonparametric Bootstrap CI Centered at Point Estimate $\hat{\theta}$}
\label{algo:symmetric-ci}
\begin{algorithmic}[1]
\REQUIRE Data $s_1,\ldots,s_N$; statistic $T(\cdot)$; number of bootstrap replicates $B$; target level $1-\alpha$; 
\ENSURE Symmetric CI $[\hat{\theta}-h,\ \hat{\theta}+h]$ and achieved bootstrap mass $k/B$.

\STATE Compute point estimate $\hat{\theta} \leftarrow T(s_1,\ldots,s_N)$.
\FOR{$b = 1$ \TO $B$}
  \STATE Draw a bootstrap sample $\{s^{*(b)}_1,\ldots,s^{*(b)}_N\}$ by sampling with replacement from $\{s_i\}_{i=1}^N$.
  \STATE Compute bootstrap re-estimate $\theta^{*(b)} \leftarrow T(s^{*(b)}_1,\ldots,s^{*(b)}_N)$.
\ENDFOR

\STATE Compute deviations $d_b \leftarrow \big|\theta^{*(b)} - \hat{\theta}\big|$ for $b=1,\ldots,B$.
\STATE Sort $\{d_b\}$ into order statistics $d_{(1)} \le \cdots \le d_{(B)}$.

\STATE $h \leftarrow d_{(k)}$.
\RETURN Symmetric CI $[\hat{\theta}-h,\ \hat{\theta}+h]$.
\end{algorithmic}
\end{algorithm}

\subsection{Robustness to Event-Category Rebalancing}
\label{app:category-rebalancing}

To examine whether our results are sensitive to the distribution of event categories, we perform robustness analyses using rebalanced subsets of the evaluation data. The main dataset is heavily skewed toward Sports events; to assess the impact of this imbalance, we construct alternative evaluation sets in which the proportion of Sports events is reduced while still retaining a mix of the other categories. One subset retains Sports as the largest category but reduces its dominance to roughly 50\% (Sports: 0.5, Entertainment: 0.15, Politics: 0.15, Other: 0.2). Another subset ensures that all four categories are equally represented (25\% each). All analyses use only submissions on or before 2025--09--10.

Table~\ref{tab:combined-results} reports the Brier scores across the original and rebalanced evaluation sets. The relative ordering of the top five models remain unchanged across all distributions, confirming that their comparative performance is robust to event-category composition. As the proportion of Sports events decreases, overall Brier scores improve slightly, suggesting that Sports events are intrinsically more challenging to predict. 

\begin{table}[h!]
\centering
\begin{tabular}{lccc}
\toprule
Forecaster & Original (N=816) & Moderate (N=250) & Full (N=152) \\
\midrule
GPT-5\textsuperscript{R} & 0.179 $\ (\pm 0.008)$ & 0.146 $\ (\pm 0.013)$ & 0.128 $\ (\pm 0.015)$ \\
Grok 4\textsuperscript{R} & 0.186 $\ (\pm 0.009)$ & 0.156 $\ (\pm 0.014)$ & 0.141 $\ (\pm 0.018)$ \\
Claude Sonnet 4\textsuperscript{R} (Thinking) & 0.190 $\ (\pm 0.009)$ & 0.162 $\ (\pm 0.016)$ & 0.148 $\ (\pm 0.020)$ \\
Gemini 2.5 Flash\textsuperscript{R} (Reasoning) & 0.195 $\ (\pm 0.009)$ & 0.165 $\ (\pm 0.016)$ & 0.157 $\ (\pm 0.021)$ \\
Llama 4 Scout & 0.230 $\ (\pm 0.011)$ & 0.202 $\ (\pm 0.019)$ & 0.210 $\ (\pm 0.025)$ \\
Market Baseline & 0.188 $\ (\pm 0.008)$ & 0.164 $\ (\pm 0.010)$ & 0.149 $\ (\pm 0.013)$ \\
\bottomrule
\end{tabular}
\captionsetup{skip=6pt}

\caption{Brier scores across the original and rebalanced evaluation subsets. Numbers in parentheses denote 95\% confidence intervals. The ``Moderate'' subset contains 50\% Sports events, while the ``Full'' subset has 25\% Sports events.}
\label{tab:combined-results}

\end{table}

As such, these results provide evidence that the conclusions drawn in the main text regarding LLM performance are robust to variations in the event-category distribution, though we intend to expand the diversity of our events in the future.

\newpage
\section{Additional Experiment Results}
\label{app:additional-experiments}

\subsection{Full Evaluation Results}\label{app:full-eval}
\begin{table}[!th]
    \centering
    \begin{threeparttable}
    \setlength{\tabcolsep}{4pt}
    \renewcommand{\arraystretch}{1.15}
    \begin{tabular}{lcccccc}
    \toprule
    & \multicolumn{2}{c}{\textbf{Forecasting Loss}} & \multicolumn{2}{c}{\textbf{Calibration Error}} & \multicolumn{2}{c}{\textbf{Market Return}} \\
    \cmidrule(lr){2-3}\cmidrule(lr){4-5}\cmidrule(lr){6-7}
    \textbf{LLMs} & {$\downarrow$ \textbf{Brier} (95\% CI)} & {\textbf{Rank}} & {$\downarrow$ \textbf{ECE}} & {\textbf{Rank}} & {$\uparrow$ \textbf{Average} (95\% CI)} & {\textbf{Rank}} \\
    \midrule
    GPT-5$^\texttt{R}$ (High)~$\triangle$ & 0.184 ($\pm$ 0.006) & {\large \ding{172}} & 0.042 & {\large \ding{175}} & 0.943 ($\pm$ 0.042) & {\large \ding{177}}\\
    GPT-5$^\texttt{R}$ & 0.187 ($\pm$ 0.005) & {\large \ding{173}} & 0.044 & {\large \ding{177}} & 0.890 ($\pm$ 0.040) & 13\\
    Market Baseline & 0.187 ($\pm$ 0.006) & {\large \ding{174}} & 0.069 & 19 & 0.899 ($\pm$ 0.043) & 10\\
    GPT-5$^\texttt{R}$ (Minimal) & 0.188 ($\pm$ 0.006) & {\large \ding{175}} & 0.036 & {\large \ding{173}} & 0.869 ($\pm$ 0.044) & 19\\
    o3$^\texttt{R}$ & 0.188 ($\pm$ 0.005) & {\large \ding{176}} & 0.030 & {\large \ding{172}} & 0.959 ($\pm$ 0.109) & {\large \ding{176}}\\
    Grok-4$^\texttt{R}$~$\triangle$ & 0.189 ($\pm$ 0.005) & {\large \ding{177}} & 0.043 & {\large \ding{176}} & 0.864 ($\pm$ 0.052) & 21\\
    Grok-3-Mini$^\texttt{R}$ & 0.189 ($\pm$ 0.006) & 7 & 0.046 & 7 & 0.869 ($\pm$ 0.043) & 20\\
    GPT-4.1 & 0.192 ($\pm$ 0.007) & 8 & 0.053 & 10 & 0.907 ($\pm$ 0.035) & 8\\
    Gemini 2.5 Pro$^\texttt{R}$ & 0.193 ($\pm$ 0.007) & 9 & 0.061 & 14 & 0.876 ($\pm$ 0.050) & 17\\
    Claude Opus 4.1$^\texttt{R}$ & 0.193 ($\pm$ 0.018) & 10 & 0.054 & 11 & 0.982 ($\pm$ 0.093) & {\large \ding{172}}\\
    Claude Sonnet 4$^\texttt{R}$~$\triangle$ & 0.194 ($\pm$ 0.006) & 11 & 0.041 & {\large \ding{174}} & 0.909 ($\pm$ 0.101) & 7\\
    o3-Mini$^\texttt{R}$ & 0.195 ($\pm$ 0.005) & 12 & 0.046 & 8 & 0.897 ($\pm$ 0.046) & 12\\
    o4-Mini$^\texttt{R}$ (High) & 0.196 ($\pm$ 0.006) & 13 & 0.062 & 16 & 0.874 ($\pm$ 0.040) & 18\\
    Kimi-K2 & 0.197 ($\pm$ 0.008) & 14 & 0.048 & 9 & 0.966 ($\pm$ 0.124) & {\large \ding{174}}\\
    Gemini 2.5 Flash$^\texttt{R}$~$\triangle$ & 0.197 ($\pm$ 0.007) & 15 & 0.067 & 17 & 0.883 ($\pm$ 0.053) & 15\\
    GPT-4o & 0.198 ($\pm$ 0.007) & 16 & 0.058 & 12 & 0.970 ($\pm$ 0.104) & {\large \ding{173}}\\
    DeepSeek-V3 & 0.201 ($\pm$ 0.008) & 17 & 0.061 & 15 & 0.963 ($\pm$ 0.103) & {\large \ding{175}}\\
    Gemini 2.5 Flash & 0.203 ($\pm$ 0.006) & 18 & 0.073 & 20 & 0.859 ($\pm$ 0.042) & 22\\
    Llama-4-Maverick & 0.208 ($\pm$ 0.008) & 19 & 0.067 & 18 & 0.904 ($\pm$ 0.050) & 9\\
    Llama-4-Scout~$\triangle$ & 0.219 ($\pm$ 0.008) & 20 & 0.060 & 13 & 0.805 ($\pm$ 0.040) & 24\\
    Gemini 2.0 Flash (Lite) & 0.224 ($\pm$ 0.013) & 21 & 0.091 & 22 & 0.855 ($\pm$ 0.074) & 23\\
    Gemini 2.0 Flash & 0.224 ($\pm$ 0.013) & 22 & 0.079 & 21 & 0.876 ($\pm$ 0.078) & 16\\
    Qwen3-235B & 0.234 ($\pm$ 0.007) & 23 & 0.118 & 23 & 0.898 ($\pm$ 0.111) & 11\\
    DeepSeek-R1$^\texttt{R}$ & 0.303 ($\pm$ 0.018) & 24 & 0.165 & 24 & 0.884 ($\pm$ 0.058) & 14\\
    \bottomrule
    \end{tabular}
    \end{threeparttable}
    \vspace{3mm}
    \caption{The full evaluation results for all 23 LLMs (including variants) and the market baseline. $\triangle$ denotes models selected for presentation in the main text.}
    \label{tab:full-eval-results}
\end{table}

\subsection{Sharpe Ratio for Average Return}
\label{app:sharpe-ratio}

As discussed in \cref{subsec:performance} and \cref{tab:full-eval-results}, Average Return often exhibits wide confidence intervals, reflecting the high variance of event-level payoffs. In most events, models earn nothing or only marginal gains, while in rare cases they achieve outsized returns. To account for this variability, we complement Average Return with the \emph{Sharpe ratio}~\citep{sharpe1998sharpe}, a standard metric that normalizes expected returns by their volatility. Formally, for asset returns $R_a$, the Sharpe ratio is defined as
\begin{equation}
\label{eq:sharpe}
S_a = \frac{\mathbb{E}[R_a - R_b]}{\sqrt{\text{Var}[R_a - R_b]}},
\end{equation}
where $R_b$ denotes a reference risk-free return. In the \prophet~setting, $R_a$ corresponds to the payoff from each event under the trading strategy of \cref{subsec:market-return}, and we set $R_b = 1$, the return from abstaining (i.e., keeping the budget unbet). The expectation and variance in \cref{eq:sharpe} are estimated by the sample mean and variance over $n$ events.

Sharpe ratios for all evaluated models are reported in \cref{tab:full-sharpe-ratio}, providing a volatility-adjusted comparison of economic performance.

\begin{table}[!th]
    \centering
    \begin{threeparttable}
    \setlength{\tabcolsep}{6pt}
    \renewcommand{\arraystretch}{1.15}
    \begin{tabular}{lcc|lcc}
    \toprule
    \textbf{LLMs} & {$\uparrow$ \textbf{Sharpe Ratio}} & {\textbf{Rank}} & \textbf{LLMs} & {$\uparrow$ \textbf{Sharpe Ratio}} & {\textbf{Rank}} \\
    \midrule
    o3$^\texttt{R}$ & -0.0131 & {\large \ding{172}} & GPT-4.1 & -0.0707 & 13\\
    GPT-5$^\texttt{R}$ & -0.0212 & {\large \ding{173}} & Claude Opus 4.1$^\texttt{R}$ & -0.0763 & 14\\
    Gemini 2.5 Pro$^\texttt{R}$ & -0.0230 & {\large \ding{174}} & Grok-4$^\texttt{R}$~$\triangle$ & -0.0771 & 15\\
    Gemini 2.5 Flash$^\texttt{R}$~$\triangle$ & -0.0270 & {\large \ding{175}} & Grok-3-Mini$^\texttt{R}$ & -0.0821 & 16\\
    o3-Mini$^\texttt{R}$ & -0.0285 & {\large \ding{176}} & o4-Mini$^\texttt{R}$ (High) & -0.0853 & 17\\
    Kimi-K2 & -0.0373 & {\large \ding{177}} & Market Baseline & -0.0897 & 18\\
    DeepSeek-V3 & -0.0381 & 7 & DeepSeek-R1$^\texttt{R}$ & -0.1199 & 19\\
    GPT-4o & -0.0383 & 8 & GPT-5$^\texttt{R}$ (Minimal) & -0.1400 & 20\\
    Qwen3-235B & -0.0389 & 9 & Gemini 2.5 Flash & -0.1483 & 21\\
    Claude Sonnet 4$^\texttt{R}$~$\triangle$ & -0.0404 & 10 & Gemini 2.0 Flash & -0.1604 & 22\\
    GPT-5$^\texttt{R}$ (High)~$\triangle$ & -0.0435 & 11 & Llama-4-Scout~$\triangle$ & -0.1799 & 23\\
    Llama-4-Maverick & -0.0637 & 12 & Gemini 2.0 Flash (Lite) & -0.1842 & 24\\
    \bottomrule
    \end{tabular}
    \end{threeparttable}
    \vspace{3mm}
    \caption{\textbf{Sharpe ratio performance for all 23 LLMs and market baseline.} $\triangle$ denotes models selected for presentation in the main text.}
    \label{tab:full-sharpe-ratio}
\end{table}

\subsection{Probability Elicitation Methods}
\label{app:probability-elicitation}
\cref{sec:arena} introduced our default approach: directly prompting an LLM to verbalize the probability that a market resolves to \texttt{Yes}. In this experiment, we conduct ablation studies over alternative confidence estimation methods. The goal is twofold: (i) test the robustness of our default prompt against reasonable variations, and (ii) illustrate how \prophet\: can serve as a testbed for comparing black-box confidence estimation methods under forecasting settings.

\tightparagraph{Setup.}
We evaluate the five representative LLMs from main text on the \textsc{Prophet-Arena-Subset-100} dataset. Metrics are Brier and ECE scores. We consider two families of confidence estimation methods (seven variants total):

\begin{enumerate}[leftmargin=1.2em,nosep]
    \item \textbf{Verbalized Probability}~\citep{tian2023just}
    \begin{itemize}[leftmargin=1em]
        \item \textit{Prompt variation}: we modify our default prompt to make it: (A) more concise, (B) more verbose, and (C) rewritten by another LLM (\modelshort{grok-4}). Key logistics and formatting instructions are preserved in all variations. These variation prompts are available in~\cref{app:prompt-variation-prompts}.
        \item \textit{Prompt emsemble}~\citep{wightman2023strength}: for each market, we average the probabilities elicited from the default and all variation prompts (i.e. (A), (B), (C) above).
        \item \textit{Bi-direction*} (ours): in addition to eliciting $p_{ij}$ (“probability of \texttt{Yes}”) using default prompt, we also elicit $p^\circ_{ij}$ (“probability of NO”), and calibrate via $\tfrac{1}{2}(p_{ij} + (1 - p^\circ_{ij}))$.
    \end{itemize}
    \item \textbf{Self-consistency}~\citep{wang2022self}
    \begin{itemize}[leftmargin=1em]
        \item \textit{Unweighted}: instead of directly asking for probability, we repeatedly query the model 10 times for a \texttt{Yes/No} decision; probability is the fraction of \texttt{Yes} answers.
        \item \textit{Weighted}~\citep{taubenfeld2025confidence}: we further supplement each \texttt{Yes/No} with a confidence score.\footnote{This confidence reflects uncertainty about the chosen answer, not a direct market probability (e.g., low confidence in \texttt{Yes} signals indecision, not belief in “NO”).} Probabilities are formed by confidence-weighted aggregation.
    \end{itemize}
\end{enumerate}

\tightparagraph{Results (see~\cref{tab:confidence-estimation-results}).} Among the verbalized probability methods, we observe that \textbf{all models exhibit strong robustness to prompt variations}. For accuracy, Brier scores for all models vary by less than one standard deviation ($\approx 0.01$). As a result, \modelshort{gpt-5} ranks the highest under all methods, and \modelshort{gemini-2.5-flash} (Thinking) consistently achieves second place in $5/6$ cases. Prompt ensemble does not lead to substantial improvement, since elicited probabilities are already similar across variations. For calibration, \modelshort{gpt-5} and \modelshort{llama-4-scout} are consistently the best and the worst models, regardless of the prompting method. Despite slightly larger fluctuations among the ECE scores, \textbf{no single method dominate the others for all models} (i.e. achieves the best scores on all columns). Our original Bi-direction* method improves calibration (over the default) on $4/5$ models, supporting the view that LLMs tend to be overconfident toward \texttt{Yes} outcomes.

\begin{table}[!t]
\centering
\setlength{\tabcolsep}{4pt} 
\begin{tabular}{lccccc}
\toprule
\textbf{Method Type / Name} & \textbf{\modelshort{grok-4}} & \textbf{\modelshort{gemini-2.5-flash}} & \textbf{\modelshort{claude-sonnet-4}} & \textbf{\modelshort{gpt-5}} & \textbf{\modelshort{llama-4-scout}} \\
\midrule
\textit{Verbalized Prob} \\
\cmidrule(r){1-1} 
 Default & 0.186/0.117 & 0.166/0.036 & 0.173/0.046 & \textbf{0.165/0.020} & 0.196/0.153 \\
 Variation A (Concise) & 0.180/0.102 & 0.172/0.036 & 0.169/0.035 & \textbf{0.162/0.021} & \underline{0.192}/0.134 \\
 Variation B (Verbose) & 0.178/0.124 & 0.166/0.039 & 0.172/0.040 & \textbf{0.160/0.028} & 0.199/0.164 \\
 Variation C (Rewrite) & \underline{0.176}/0.123 & 0.167/\underline{0.027} & 0.172/0.039 & \textbf{0.159/\underline{0.016}} & 0.195/0.167 \\
 Ensemble & 0.177/0.117 & 0.165/0.032 & 0.170/0.043 & \textbf{0.160/0.024} & \underline{0.192}/0.142 \\
Bi-direction* & 0.180/0.101 & \underline{0.164}/0.031 & \underline{0.165}/\underline{0.028} & \textbf{\underline{0.158}/0.023} & 0.203/0.140 \\
\midrule
\textit{Self-consistency} \\
\cmidrule(r){1-1} 
 Unweighted & 0.238/0.115 & \textbf{0.231}/0.110 & 0.241/0.071 & 0.239/\textbf{0.071} & 0.267/0.129 \\
 Weighted & 0.205/\underline{0.091} & 0.201/0.067 & 0.189/0.050 & \textbf{0.181/0.046} & 0.214/\underline{0.125} \\
\bottomrule
\end{tabular}
\vspace{3mm}
\caption{\textbf{Evaluation Results for Different Probability Elicitation Methods .} Each cell contains a pair of \emph{Brier$\downarrow$ / ECE$\downarrow$} scores. \textbf{Bold} denotes the best score for each row, \underline{underline} for each column.}
\label{tab:confidence-estimation-results}
\vspace{-3mm}
\end{table}

In contrast, \textbf{self-consistency methods result in significantly lower accuracies and yield mixed calibration benefits}. With 10 rollouts, the unweighted variant produces coarse-grained probabilities at a resolution of 0.1, limiting accuracy despite incurring higher compute cost. The weighted variant partially alleviates this by using finer-grained confidence signals, producing noticeable calibration gains but still trailing verbalized methods in accuracy.

To sum up, these results show that (i) LLMs are generally robust against prompt ablations, justifying our default prompt choice, and (ii) \prophet\: provides a natural benchmark for evaluating and contrasting future confidence estimation/calibration methods. 

\subsection{Reasoning Consistency of LLMs}
\label{app:reasoning-consistency}

In addition to the primary relative (Brier score) and absolute (Average return) metrics, certain types of forecasting events also enable us to evaluate the consistency -- an important component of reasoning -- of LLMs.
These metrics can be calculated \textbf{solely by looking at the potential event outcomes and the LLM probabilities given to them.} Below we give two concrete examples of such consistency metrics.

\paragraph[]{Logical chain score.\footnote{This is a placeholder name. Feel free to suggest a better one.}} Consider the forecasting event \textit{``The bitcoin price by the end of 2026''}, and two of its outcomes are \textit{``(A) The bitcoin price is above \$200,000''} and \textit{``(B) The bitcoin price is above \$220,000''}, respectively. No matter how good an LLM is at predicting the probabilities for (A) \& (B), we know that anyone with \textbf{consistent reasoning} will give $\mathbb{P}[(A)] \geq \mathbb{P}[(B)]$ since the latter outcome logically implies the former. We denote such a relationship as a \textbf{logical chain}, or $(B) \to (A)$. Obviously, this logical chain can contain more than two outcomes, so we call a logical chain $\mathcal{S} = (S_1) \to ... \to (S_T)$ \textbf{maximal} whenever it satisfies both:
\begin{enumerate}
    \item For all $ 1 \leq t < T$, we have $(S_t) \to (S_{t+1})$,
    \item No other outcome $(K) \notin \mathcal{S}$ satisfies $(K) \to (S_1)$ or $(S_T) \to (K)$.
\end{enumerate}

A single event might contain multiple maximal logical chains $\mathcal{S}_1,...,\mathcal{S}_n$ with lengths $T_1,...,T_n$. For an LLM with probability $\mathbb{P}[(A)]$ for outcome $(A)$, its \textbf{logical chain score} for this event is given by $\frac{1}{n}\sum_{i = 1}^n score(\mathcal{S}_i)$, where
\begin{equation}
    score\left(\mathcal{S}_i= (S_{i1} \to ...\to S_{i T_i})\right) := \frac{1}{T_i -1} \sum_{j=1}^{T_i - 1} \mathbf{1}\{\mathbb{P}[(S_j)] \leq \mathbb{P}[(S_{j+1})]\}
\end{equation}
with $\mathbf{1}(\cdot)$ being the indicator function. The final logical chain score is then \textbf{averaged over all events with at least one chain}. We adopt an LLM-judge (\texttt{Gemini-2.5-Flash}) to automatically detect the maximal logical chains for all our events.

\paragraph[]{Mutually exclusive score.} 

In certain forecasting events, the potential outcomes are \textbf{mutually exclusive}, meaning that exactly one outcome can occur. For example, in the event \textit{``Who will win the NBA championship in 2026?''}, the possible outcomes could be the teams, where only one team can win. A \textbf{maximal set of mutually exclusive outcomes} is a set where 
\begin{enumerate}
    \item each outcome is distinct and all outcomes are mutually exclusive,
    \item the event will resolve to one and only one outcome in the set.
\end{enumerate}
If such a maximal set $\mathcal{S} := \{(S_1),...,(S_m)\}$ with size $m$ exists for a forecasting event, we can calculate the \textbf{mutually exclusive score} at this event as:
\begin{equation}
    score_{\text{ME}} = \mathbf{1}\left\{\sum_{i=1}^m \mathbb{P}[(S_i)] = 1\right\}
\end{equation}
(In practice, we allow the sum to deviate slightly from 1 with some tolerance level $\epsilon$).

We evaluate this score over all events where mutually exclusive outcomes are defined. The identification of maximal set is performed using the same LLM-judge (\texttt{Gemini-2.5-Flash}).

\begin{table}[!th]
    \small
    \centering
    \begin{threeparttable}
    \setlength{\tabcolsep}{6pt}
    \renewcommand{\arraystretch}{1.15}
    \begin{tabular}{lcc}
    \toprule
    \textbf{LLMs} & \textbf{Mutually Exclusive Consistency} & \textbf{Logical Chain Consistency} \\
    \midrule
    DeepSeek-R1$^\texttt{R}$ & 0.996 & 0.987 \\
    o4-Mini$^\texttt{R}$ & 0.995 & 0.998 \\
    GPT-5$^\texttt{R}$ (High) & 0.995 & 0.999 \\
    Gemini 2.5 Flash$^\texttt{R}$ & 0.995 & 0.997 \\
    Gemini 2.5 Pro$^\texttt{R}$ & 0.995 & 0.995 \\
    Claude Sonnet 4$^\texttt{R}$ & 0.995 & 0.987 \\
    DeepSeek-V3 & 0.994 & 0.973 \\
    Grok-4$^\texttt{R}$ & 0.994 & 0.994 \\
    GPT-4.1 & 0.994 & 0.973 \\
    Llama 4 Maverick & 0.994 & 0.981 \\
    Llama 4 Scout & 0.994 & 0.979 \\
    Grok-3-Mini$^\texttt{R}$ & 0.994 & 0.901 \\
    Qwen3-235B & 0.994 & 0.988 \\
    GPT-4o & 0.994 & 0.930 \\
    GPT-5$^\texttt{R}$ (Minimal) & 0.994 & 0.962 \\
    GPT-5$^\texttt{R}$ (Medium) & 0.994 & 0.990 \\
    o3-Mini$^\texttt{R}$ & 0.994 & 0.930 \\
    Gemini 2.0 Flash & 0.994 & 0.990 \\
    Kimi-K2 & 0.993 & 0.984 \\
    o3$^\texttt{R}$ & 0.993 & 0.996 \\
    Gemini 2.0 Flash (Lite) & 0.993 & 0.911 \\
    \bottomrule
    \end{tabular}
    \end{threeparttable}
    \vspace{3mm}
    \caption{\textbf{Consistency scores for the LLMs.} Most LLMs evaluated have exhibited excellent performances in both consistency metrics, indicating their mature logical reasoning skills.}
    \label{tab:consistency-22llms-single}
\end{table}

\subsection{Variability in Model Forecasts Despite Identical Inputs}

\begin{figure}[t]
    \centering
    \includegraphics[width=0.4\linewidth]{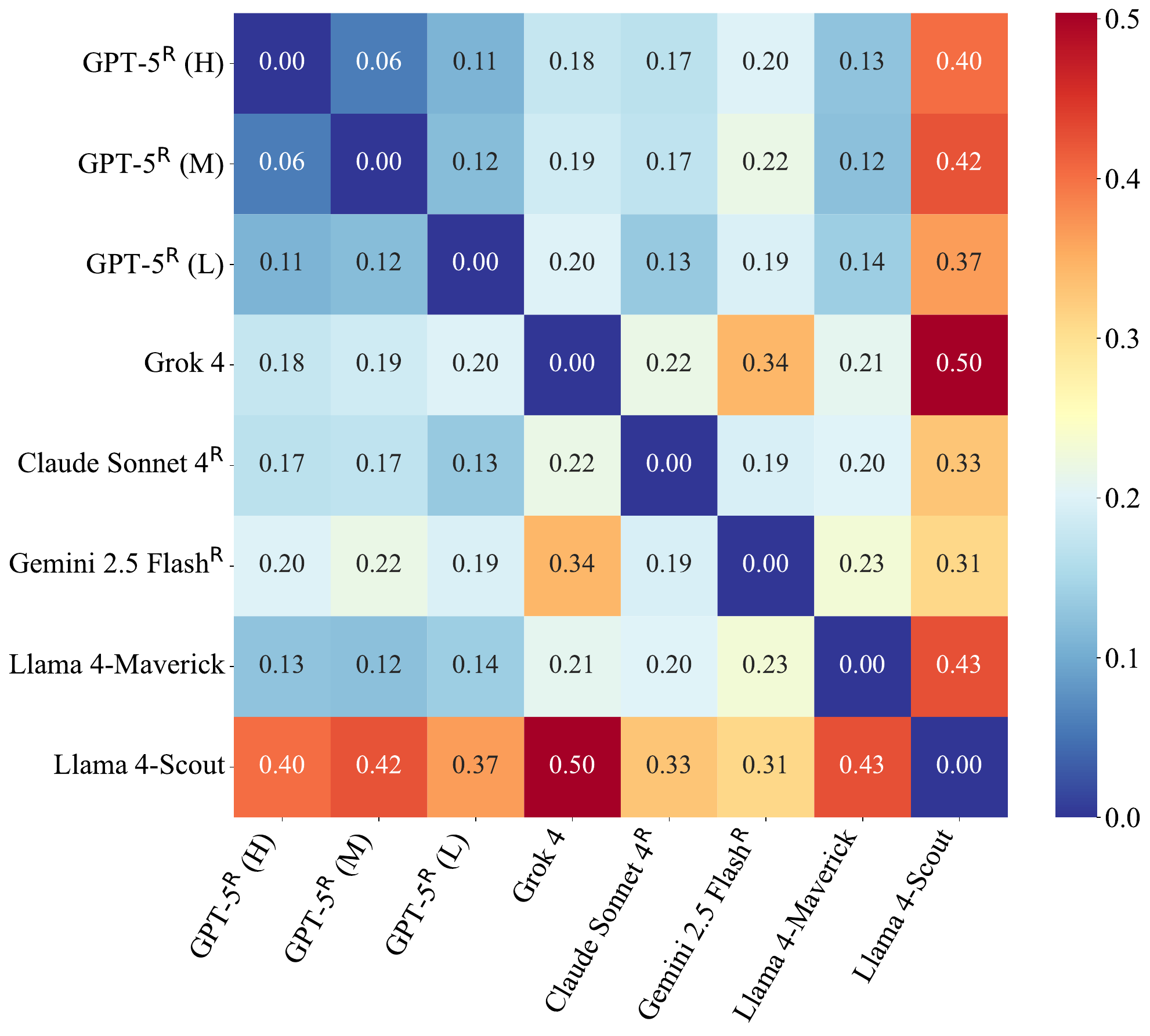}
    \caption{\textbf{Pairwise divergence of model predictions across tasks.} 
    Heatmap shows the average $L_2$ distance between probabilistic predictions of each model pair. 
    \modelshort{gpt-5} H, M, L labels represent high, medium, and minimal reasoning efforts, as explained by \cref{app:full-llm-list}.}
    \label{fig:llmDiffs}
\end{figure}

While all tested LLMs receive the same market data and news sources, they differ substantially in how they combine and weigh each piece of information, resulting in diverse prediction patterns. \cref{fig:llmDiffs} shows averaged pairwise differences in event predictions, and the generally large differences across models suggest that LLMs reason about events in fundamentally different ways, even when given identical inputs.

Interestingly, even within model families, clustering patterns also vary notably. For example, the GPT-5 variants produce relatively homogeneous predictions, reflecting similar training and reasoning capabilities. In contrast, Llama 4 Maverick and Scout, despite belonging to the same model family, exhibit the second largest average pairwise difference (0.43). This divergence illustrates that sharing a model family does not guarantee similar predictions, nor does it imply identical reasoning processes.

\subsection{Evaluating Reasoning using an LLM-as-a-Judge Framework}
\label{app:reasoning}

To directly evaluate LLMs reasoning processes, we employ an LLM-as-a-judge framework \citep{zheng2023judging} to assess the soundness of their reasoning methodologies.  Our evaluation proceeds in two stages: (1) eliciting explicit reasoning from the prediction model, and (2) systematically evaluating this reasoning with an independent evaluator. 

The assessment framework encompasses five critical dimensions, each scored on an 1-5 point scale (where 1 and 5 indicate poor and excellent performance, respectively):

\begin{enumerate}[leftmargin=2em]
    \item \textbf{Source  selection}: Assessment of how effectively models incorporate provided sources and the reliability of their source selection criteria.
    \item \textbf{Evidence extraction from selected sources} : Evaluation of the model's ability to extract relevant evidence from sources and demonstrate sophisticated interpretation beyond surface-level analysis.
    \item \textbf{Reasoning synthesis}: Analysis of how extracted evidence is integrated into coherent justifications, including the model's approach to combining and weighting disparate pieces of evidence.
    \item \textbf{Reasoning-to-prediction}: Assessment of how effectively the reasoning process is translated into the final probabilistic prediction.
    \item \textbf{Recognition of prediction uncertainty}: Examination of the model's capacity to identify and appropriately account for uncertainties and potential counterarguments within their analysis. 
\end{enumerate}

To enhance evaluation reliability, we incorporate a human expert-assessed reference evaluation of an external event not included in the dataset, which serves as a grounding benchmark. Additionally, we lower the temperature setting to 0 for the LLM judge (\modelshort{claude-sonnet-4}) to improve response consistency. We include the full prompt in \cref{app:reasoning_eval_prompt}.

We include the full table of reasoning evaluations with a wider range of LLMs below. The same trends observed in \cref{sec:granularAnalysis} are evident here: models reach near-parity in source utilization, evidence extraction, and uncertainty analysis, while exhibiting substantial disparities in reasoning synthesis and reasoning-to-prediction alignment. These latter dimensions are the primary drivers of differences in overall predictive performance. As such, the findings further suggest that the development of future prediction agents should prioritize advances in higher-order reasoning and the alignment of reasoning with probabilistic forecasts, rather than focusing on marginal improvements in retrieval or evidence handling.

\begin{table}[h]
\centering
\small
\begin{tabular}{lcccccc}
\toprule
\textbf{LLM} & \textbf{Sources} & \textbf{Evidence} & \textbf{Reas. Synth.} & \textbf{Align.} & \textbf{Uncert.} & \textbf{Average Score} \\
\midrule
\modelshort{gpt-5}$^\texttt{R}$ (High)       & 3.69 & 3.66 & \textbf{4.14} & \textbf{3.97} & 3.94 & \textbf{3.88} \\
\modelshort{o3}                & \textbf{3.71} & \textbf{3.74} & 3.93 & 3.78 & 3.87 & 3.81 \\
\modelshort{gemini-2.5-pro}        & 3.70 & 3.69 & 3.39 & 3.92 & \textbf{3.95} & 3.73 \\
\modelshort{gpt-5}$^\texttt{R}$ (Medium)      & 3.69 & 3.65 & 3.69 & 3.66 & 3.94 & 3.73 \\
\modelshort{gpt-5}$^\texttt{R}$ (Minimal)     & 3.69 & 3.64 & 3.26 & 3.58 & 3.90 & 3.61 \\
\modelshort{gemini-2.5-flash}$^\texttt{R}$      & 3.57 & 3.66 & 3.19 & 3.67 & 3.74 & 3.57 \\
\modelshort{grok-4}          & 3.40 & 3.51 & 3.33 & 3.48 & 3.66 & 3.48 \\
\modelshort{claude-sonnet-4}$^\texttt{R}$     & 3.53 & 3.47 & 2.93 & 3.39 & 3.75 & 3.41 \\
\modelshort{llama-4-maverick}   & 3.14 & 3.29 & 2.43 & 2.14 & 3.14 & 2.83 \\
\modelshort{gpt-4o}           & 3.07 & 2.99 & 2.32 & 2.59 & 2.96 & 2.79 \\
\modelshort{llama-4-scout}     & 2.97 & 2.88 & 2.29 & 2.37 & 2.87 & 2.68 \\
\bottomrule
\end{tabular}
\vspace{3mm}
\caption{\textbf{Full table on LLM performance on reasoning evaluation criteria across dataset events.} Each dimension is scored on a standardized 5-point scale, where 1 and 5 indicate poor and excellent performance, respectively. Average scores are presented for each model, with \textbf{bold} values indicating the best-performing model for each criterion. Models are ordered by descending overall average score.}
\label{tab:llm_judge_eval_2}
\end{table}

\subsubsection{Validating LLM-as-a-Judge via Human Ratings} \label{app:reasoning:rater}

To validate the LLM-as-a-judge system and check for alignment bias, we conducted a blinded human adjudication study. We randomly sampled 170 events with LLM-generated reasoning traces and hid both the model identity and LLM judge outputs from human raters. Each reasoning trace was independently scored by humans on a 1-5 scale across the five rubric dimensions (Sources Used, Evidence Extracted, Combination \& Weighting, Uncertainties \& Counterpoints, and Mapping to Final Probabilities). The rubric provided was identical to that given to the LLMs.

Overall, we find a high degree of concordance between human ratings and LLM-as-a-judge scores. As shown in~\cref{tab:mad}, mean absolute differences between human and LLM ratings ranged from 0.37 to 0.44, with an average difference of 0.42, and variances between 0.27 and 0.36. These values are strikingly low given the 1-5 rating scale: all differences are, on average, substantially smaller than 1, meaning that the LLM judge’s score deviates from human judgment by less than half a point. Moreover, across all rubric categories, over 94\% of human ratings were within 1 point of the LLM’s score as seen in~\cref{fig:agreementMatrix}, underscoring the tight alignment between automated and human evaluation.

\begin{table}[h!]
\centering
\begin{tabular}{lc}
\toprule
Rubric Category & Mean Absolute Difference (0-4) (Variance) \\
\midrule
Sources Used & 0.43 (0.30) \\
Evidence Extracted & 0.42 (0.27) \\
Combination Weighting & 0.43 (0.36) \\
Uncertainties Counterpoints & 0.44 (0.28) \\
Mapping To Final Probabilities & 0.37 (0.29) \\
\bottomrule
\end{tabular}
\vspace{3mm}
\caption{Human ratings for rubric categories, with variance shown in parentheses.
Mean absolute difference is calculated as the average $\lvert\,\text{human score} - \text{LLM score}\,\rvert$ across the 170 events. Differences are on a 0-4 scale.}
\label{tab:mad}
\end{table}

\begin{figure}[h!]
    \centering
    \includegraphics[width=0.7\linewidth]{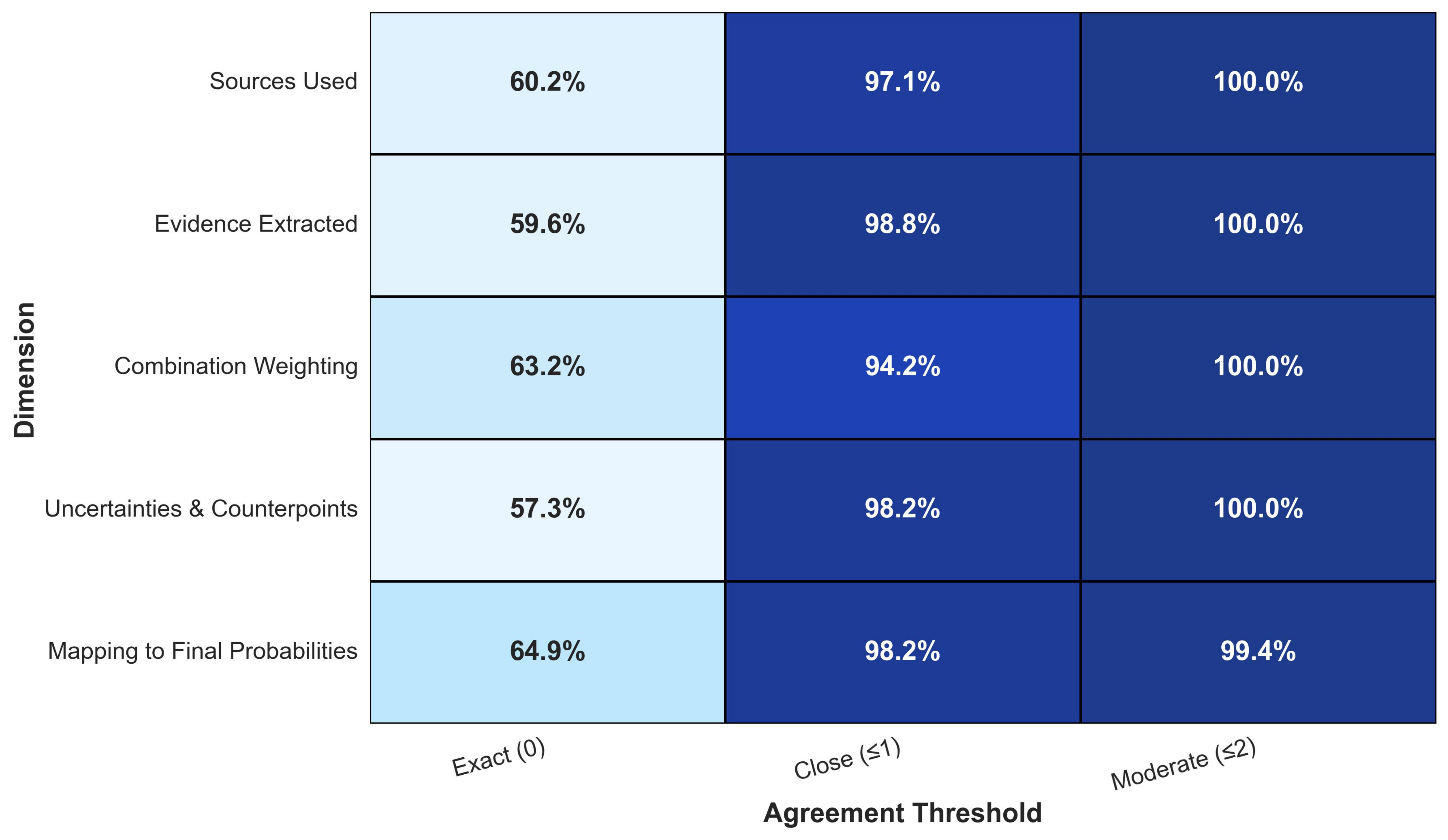}
    \caption{Agreement matrix showing the percentage of human ratings that lie within exact, moderate ($\pm 1$), and larger ($\pm 2$) thresholds of the LLM-judge ratings for each rubric dimension.}
    \label{fig:agreementMatrix}
\end{figure}

\subsubsection{(Lack of) Effect of Scaffolding Prompt on Prediction}
\label{app:scaffoldEffect}

\cref{tab:brier_diff} reports the differences in Brier score between the predictions with the reasoning scaffolding prompt and the predictions with the default, non-scaffolding configurations. Overall, the differences are insignificant for all models, indicating that the probabilistic forecasts remain largely stable regardless of the prompt type. This suggests that the enhanced reasoning elicitation prompt did not meaningfully improve prediction performance, implying that the prompt primarily serves as a structured summary rather than enhancing the models’ underlying predictive capabilities.

\begin{table}[ht!]
\centering
\begin{tabular}{l c}
\toprule
\textbf{Model} & \textbf{Brier Difference} \\
\midrule
GPT-5$^\texttt{R}$ (High) & -0.0012 \\
GPT-5$^\texttt{R}$ (Medium) & 0.0011 \\
GPT-5$^\texttt{R}$ (Minimal) & 0.0017 \\
o3$^\texttt{R}$ & -0.0018 \\
Gemini 2.5 Pro$^\texttt{R}$ & -0.0080 \\
Gemini 2.5 Flash$^\texttt{R}$ & -0.0001 \\
Llama 4 Maverick & 0.0095 \\
Llama 4 Scout & 0.0138 \\
Grok-4$^\texttt{R}$ & -0.0186 \\
Claude Sonnet 4$^\texttt{R}$ & -0.0005 \\
DeepSeek-R1$^\texttt{R}$ & 0.0057 \\
GPT-4o & 0.0053 \\
\bottomrule
\end{tabular}
\vspace{3mm}
\caption{Brier differences between the reasoning scaffolding prompts and non-scaffolding configurations.}
\label{tab:brier_diff}
\end{table}

\subsection{Knowledge Internalization}
\subsubsection{Knowledge Internalization Events}
\label{app:mem_events}
We sample 100 events from Kalshi \citep{kalshi}, with market close time before October 2023 (i.e., before all model knowledge cutoff dates. Note that despite that many popular events on Kalshi are \textit{Sports} events as of August 2025, sports betting was not legal on Kalshi until 2024 \citep{wilmot2025kalshi}. The sampled past events span the categories available on the platform during that period and exhibit differing levels of \emph{temporal granularity}. Some events target a specific timestamp or date (e.g., \textit{NASDAQ price on August 20, 2023}), while others are coarser, period-level questions without a single focal timestamp (e.g., \textit{Will WTI crude oil prices decrease in Q2 2023?}).

\begin{figure}[htpb!]
    \centering
    \includegraphics[width=0.75\linewidth]{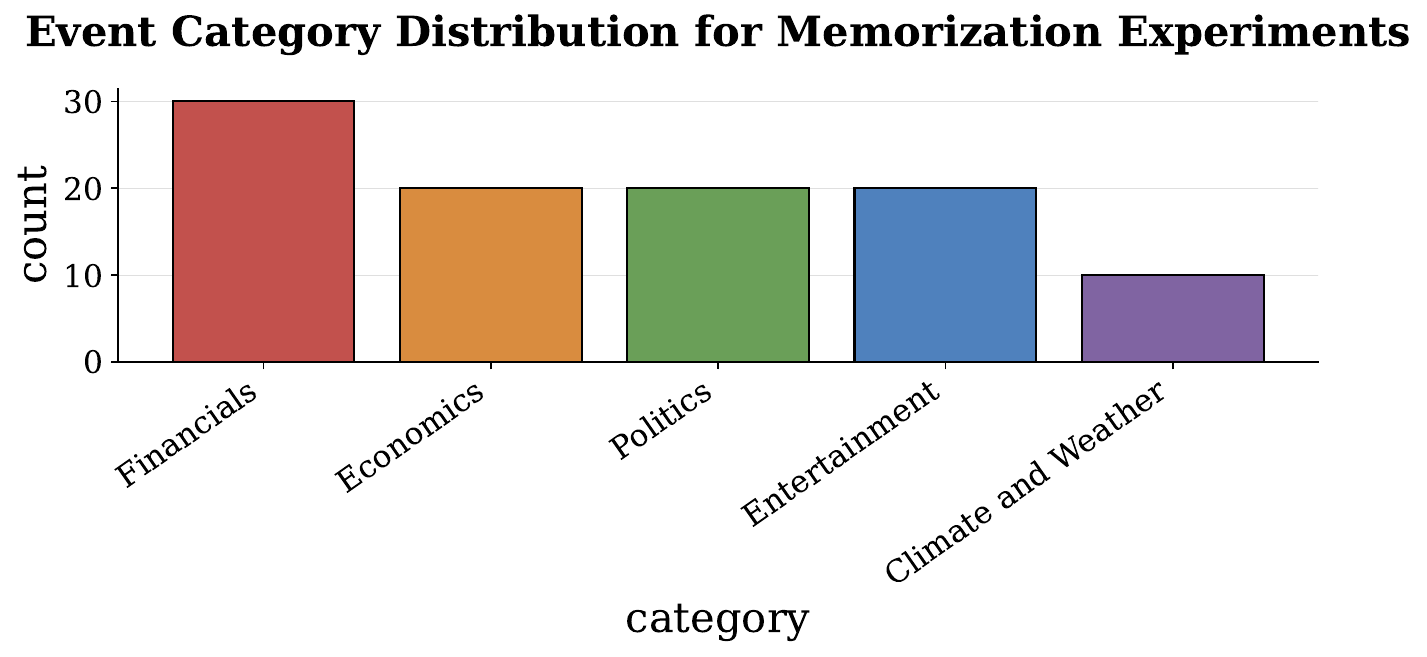}
    \caption{Event distribution for memorization experiments}
    \label{fig:mem_categories}
\end{figure}

\subsubsection{Knowledge Internalization Prompts}
\label{app:mem_prompts}

We use three complementary prompts to test models' knowledge internalization. 

\begin{enumerate}[leftmargin=2em]
    \item \textbf{Prediction prompt (no sources)}: the original forecasting-style prompt used in \prophet\, without sources context (\cref{app:internal_prediction_prompt}). Although framed as a forward-looking prediction, all of these past events are in fact already represented in the model's training data. Success in these cases indicates that the model can implicitly recognize the event as historical and draw on its internalized knowledge to answer correctly. Failure, in contrast, highlights a gap between memorization and reasoning: the model may “know” the fact but still treat the prompt purely as a forecasting task, leading to mis-recall

    \item \textbf{Prediction prompt with sources}: the same forecasting prompt with an additional block (\cref{app:prediction_prompt_w_src}), but augmented with event-specific sources. In this setting, the model is no longer reasoning only from internalized knowledge: it must integrate retrieved evidence with what it already ``remembers." This setup tests whether the model can align its internal recall with external evidence, and whether retrieval corrects, reinforces, or conflicts with its memorized knowledge.

    \item \textbf{Recall prompt}: a specialized prompt that explicitly frames the task as recalling a past outcome (\cref{app:recall_prompt}). This isolates the model’s internalized knowledge, revealing whether it has retained coarse or precise details about prior events.
\end{enumerate}

\subsection{How Good are LLMs at Finding Sources?}
\label{app:llm_finding_srcs}
\begin{figure*}[t]
  \centering
  \begin{subfigure}[t]{0.45\textwidth}
    \centering
    \includegraphics[width=\linewidth,trim=8pt 12pt 18pt 8pt,clip]{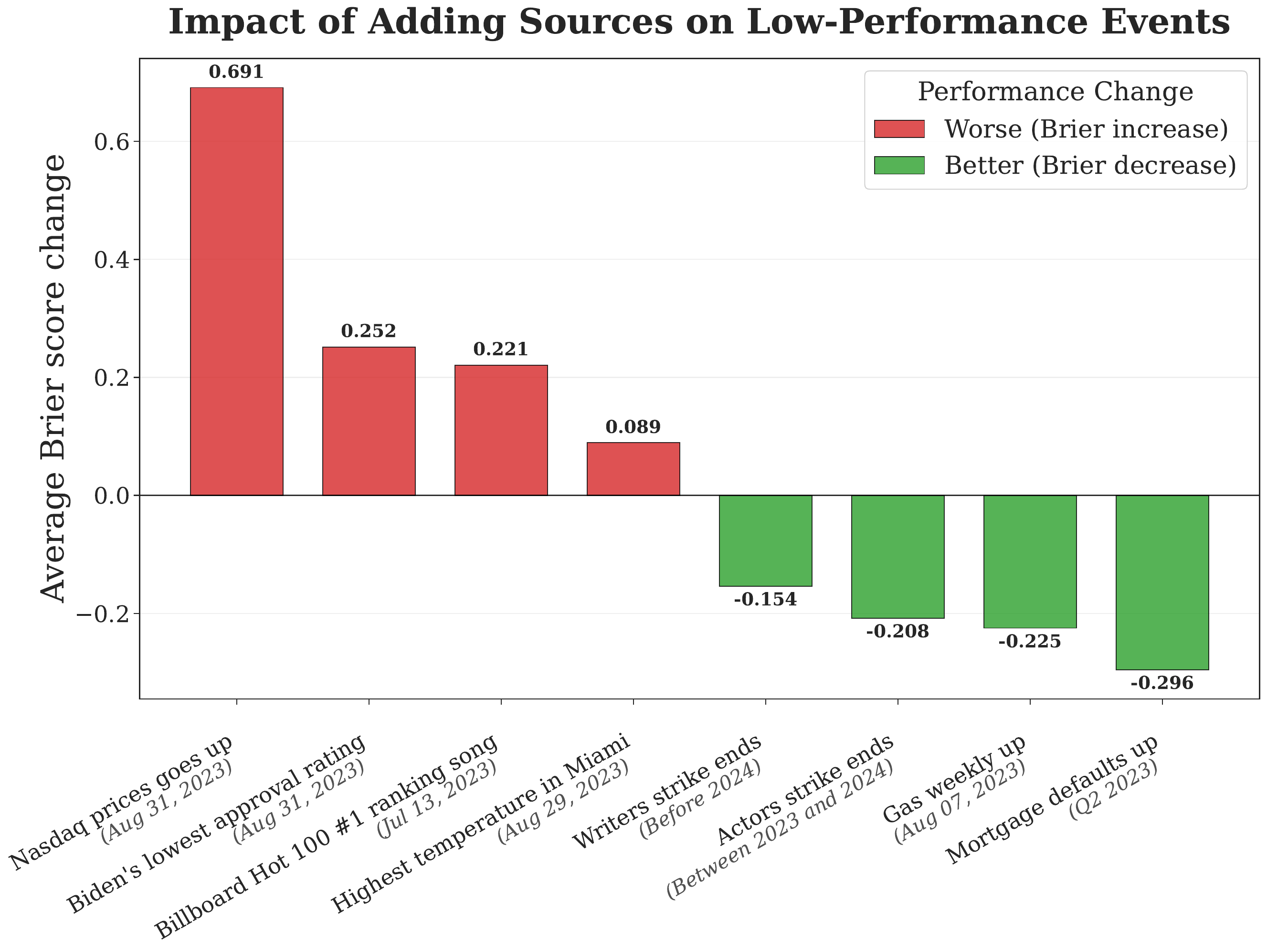}
    \phantomcaption
    \label{fig:brier_changed}
  \end{subfigure}\hfill
  \begin{subfigure}[t]{0.49\textwidth}
    \centering
    \includegraphics[width=\linewidth,trim=8pt 8pt 8pt 8pt,clip]{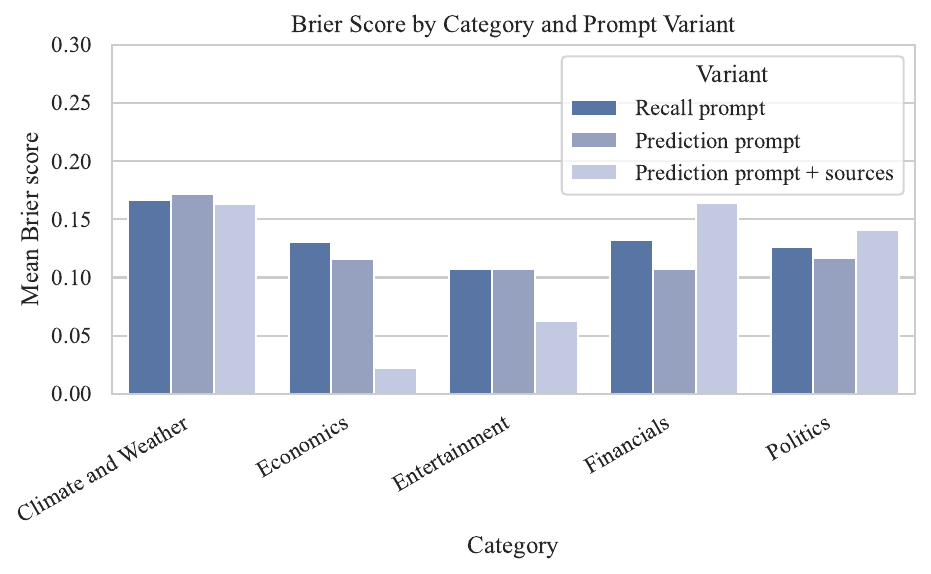}
    \phantomcaption
    \label{fig:mem_by_cat}
  \end{subfigure}
  \caption{\textbf{Impact of adding sources to context on past events.} The “low-performance” events refer to past events that the models have worst recall on. The average Brier score change is computed by evaluating the difference between when models are provided with searched sources vs.\ when models are solely recalling based on memory.}
  \label{fig:sources_impact}
\end{figure*}

We evaluate models' source-finding ability on the set of \textit{past} events (\cref{app:mem_events}), using the search prompt. Because these events have resolved, accurate information is publicly available; competent search should surface the correct evidence and hence, on average, improve the quality of the models' recall.
In \cref{fig:mem_by_cat}, we observe that the addition of sources improve LLM recall quality in certain events but not others. In fact, the figure shows that the events hardest to recall often becomes worse from the source-augmented variant. One likely explanation is 
that retrieved sources, while factual, introduce noise or extraneous details that 
interfere with the model’s internalized recall and reasoning. Events requiring precise, fine-grained recall (e.g.,  financial indicators or political approval ratings) tend to suffer when sources are added, likely because the retrieved evidence contains multiple overlapping numbers, dates, or conditions that confuse the model. By contrast, more salient and broadly covered events (e.g., major stock index movements or entertainment outcomes) generally improve with sources, since the found sources likely have lower variance and higher accuracy. 

This demonstrates that even on past events that have closed, LLM searchers do not yet possess the ability to accurately pinpoint the most useful sources for finer-granularity events. 
\newpage 
\section{Case Studies}
\label{app:case_studies}

\subsection{Differences in Prediction with the Same Information}
\label{app:case_studies_same_info}

\begin{table}[!ht]
\caption{Sources used for Real Madrid vs. Al Hilal SFC match.}
\label{tab:sources_club_wc}
\small
\begin{center}
\begin{tabular}{@{}l p{0.4\linewidth} p{0.3\linewidth}@{}}
\toprule
\textbf{Source} & \textbf{Summary} & \textbf{Title} \\
\midrule
Sporting News & Real Madrid is heavily favored to win against Al Hilal, with betting odds reflecting their dominance despite potential lineup challenges. & Real Madrid vs Al Hilal prediction, odds, betting tips and best bets for Club World Cup final \\
PokerStars Sports & Betting odds indicate a strong expectation for a Real Madrid victory, with a significant likelihood of multiple goals being scored in the match. & Real Madrid v Al-Hilal Betting Odds | PokerStars Sports \\
Sporting News (India) & Analysts predict a 3--1 victory for Real Madrid, citing Al Hilal's reliance on penalties in previous matches and Madrid's superior quality. & Real Madrid vs Al Hilal prediction, odds, betting tips and best bets for Club World Cup final \\
BetsLoaded & Real Madrid is predicted to win against Al Hilal, with recent form and head-to-head statistics favoring the Spanish club. & Real Madrid vs Al Hilal Saudi FC Prediction, Betting Tips (18 June 2025) \\
El Pa\'{\i}s & Under Xabi Alonso, Real Madrid is striving to establish a new identity with a focus on high-pressure play, though the team is still adapting to this approach. & El Real Madrid de Xabi busca nueva identidad en Miami: `Empieza el rock and roll' \\
Sky Sports & Historical data shows Real Madrid's previous victory over Al Hilal, suggesting a favorable outcome for the Spanish team in the upcoming match. & Form and head to head stats Real Madrid vs Al-Hilal \\
AS (Diario AS) & Al Hilal's top scorer, Aleksandar Mitrovi\'c, will miss the match against Real Madrid due to a muscle injury, significantly weakening their offensive capabilities. & Al Hilal pierde a Mitrovic \\
AS (Diario AS) & Real Madrid's potential lineup against Al Hilal may see Rodrygo replacing the ill Mbapp\'e, with new signings Alexander-Arnold and Huijsen expected to start in defense. & Alineaci\'on posible del Real Madrid contra Al Hilal en el Mundial de Clubes \\
AS (Diario AS) & Real Madrid, under new coach Xabi Alonso, faces defensive challenges due to injuries and is uncertain about its tactical formation ahead of the match against Al Hilal. & Nueva era y nueva defensa \\
Reuters & Kylian Mbapp\'e is doubtful for Real Madrid's match against Al Hilal due to a high fever, potentially leaving the team without a recognized center forward. & Mbappe doubtful for Real Madrid's Club World Cup opener against Al-Hilal \\
\bottomrule
\end{tabular}
\end{center}
\end{table}

In the following example, the LLMs predict on the event, the Club World Cup soccer game between Real Madrid and Al Hilal SFC, scheduled for Jun 18, 2025. The sources and market snapshots that LLMs receive are also the same, detailed in \cref{tab:sources_club_wc}. The market prices were around (74\% Madrid, 15\% tie, 14\% Al Hilal). Despite given the same market and source information, models showed differing rationales. \par 

\modelshort{gpt-4o} for example, explicitly adjusted its prediction to account for the possibility of a draw, predicting (70\% Madrid, 20\% tie, 10\% Al Hilal), noting:
\begin{displayquote}
``Sources universally favor Real Madrid, citing their superior quality, historical victory, and Al Hilal’s weakened state without Mitrović. Market predictions strongly back Real Madrid, \textbf{but I adjusted the probabilities slightly to account for Real Madrid’s lineup challenges and adaptation under a new coach. The uncertain defensive setup and potential absence of Mbappé slightly raise the probability of a tie.}”
\end{displayquote}

\modelshort{gemini-2.5-flash}, in contrast, acknowledged the sources that suggest a higher chance of a tie, but nevertheless decides to place a higher weight on Madrid win (79\% Madrid, 13\% tie, 8\% Al Hilal). It reasons that: 

\begin{displayquote}
``I heavily weighed the consistent expert analysis, betting odds, and market data, all strongly favoring Real Madrid. The probability distribution reflects Real Madrid's clear dominance and superior quality, further amplified by the critical injury to Al Hilal's top scorer. \textbf{Although Real Madrid faces minor lineup challenges and is adapting to a new tactical approach, these factors are significantly outweighed by Al Hilal's weakened offensive capabilities and the overall disparity in team strength."}
\end{displayquote}

These fine-grained distinctions are not rarely found in our collected data. They reflect how models reason under uncertainty, and often contributes to the differences in model performances. 

\subsection{Event recall is approximate, not precise}
\label{app:event_recall}
Below is the model output of \modelshort{gemini-2.5-flash} on the event \textit{Billboard Hot 100 \#1, Jul 13, 2023?}, from the Recall prompt in \cref{app:mem_prompts}.
\FloatBarrier
\begin{promptbox}{\modelshort{gemini-2.5-flash} Model Output}
\begin{lstlisting}
Event: TOPSONG-23JUL13
Q: Billboard Hot 100 #1, Jul 13, 2023?

Probabilities:
  - Vampire (Olivia Rodrigo): 1.0
  - Last Night (Morgan Wallen): 0.0
  - Fast Car (Luke Combs): 0.0

Rationale:
  I recall that Olivia Rodrigo's 'Vampire' debuted at 
  number one on the Billboard Hot 100 for the chart 
  dated July 15, 2023, which corresponds to the 
  July 13, 2023, reference in the event. 
  It displaced Morgan Wallen's 'Last Night'.

Recall Assessment:
  - Recognized event: True
  - Evidence: 
    1. Olivia Rodrigo's 'Vampire' debuted at #1 
       on July 15, 2023.
    2. The song accrued 29.1M streams and sold 
       25.7k copies in its first week.
  - Recalled outcome: Vampire by Olivia Rodrigo
\end{lstlisting}
\end{promptbox}
\FloatBarrier

\subsection{Case Study: Bitcoin Price at 9 AM EDT}
\label{app:bitcoin}

Two identical Bitcoin price prediction tasks, separated by five days (July 4 vs. July 9, 2025), illustrate the differences that source quality can have on an LLM's ability to form quality predictions. While the prediction task is essentially identical, for the July 4th event, including sources in the prompt \textbf{worsened} prediction accuracy, whereas for the July 9th event, source inclusion \textbf{substantially improved} predictions (see App.~\ref{app:bitcoinBrier} for the detailed figure). Full source lists for both events are also available in App.~\ref{app:bitcoinSources}. Despite the July 4th source list have greater platform diversity (mainstream media, AI-generated predictions, algorithmic forecast bots), it suffers from quality inconsistencies and methodological fragmentation. Numerous questionable forecast bots (e.g., priceforecastbot.com, cryptopredictions.com) produce extremely wide-ranged predictions (\$82,822 - \$121,797), which works to generate noise rather than useful signals. Similarly, dubious sources like ChatGPT-based predictions further dilute analytical credibility. 
In contrast, the July 9th dataset exhibits stronger coherence and practical utility, with predictions converging around \$115,000 - \$125,000 from crypto-specialized platforms (CoinEdition, CoinDCX, Quickex.io, CoinCu) using consistent technical analysis. 

Thus, source competence and reliability are critical for enabling LLMs to generate accurate forecasts. As such, these findings show that feeding in mass amounts of information to LLMs does not necessarily enhance prediction quality.

\subsubsection{Brier Score Comparisons}
\label{app:bitcoinBrier}
\cref{fig:bitcoin} illustrates the differences in Brier score when sources were added to the prompt for both July 4th and July 9th events. While for the July 4th event, sources significantly improved the prediction quality, for the July 9th event, it worsened the prediction quality.

\begin{figure}[htpb!]
  \centering
  \begin{subfigure}[t]{0.48\linewidth}
    \centering
    \includegraphics[width=\linewidth]{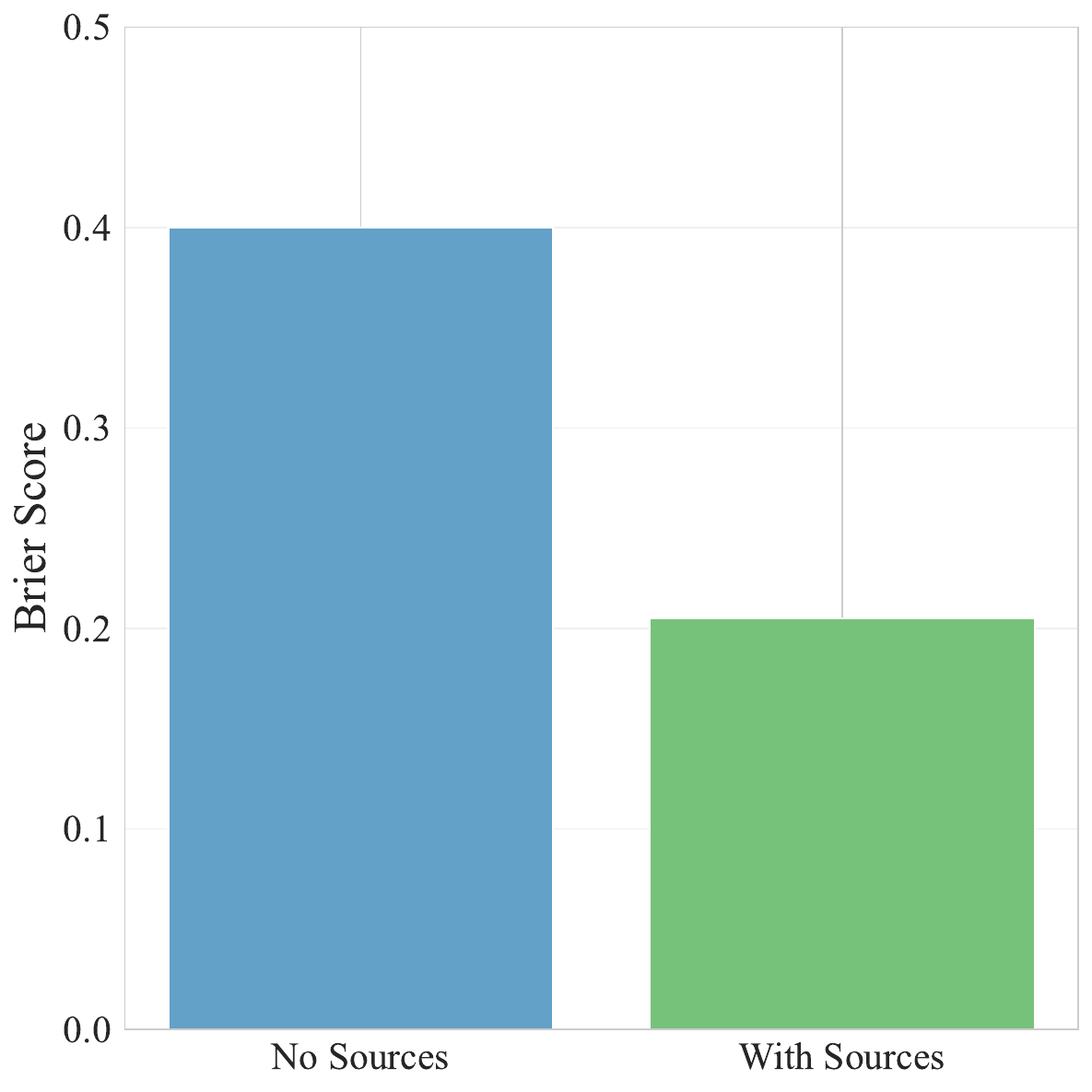}
    \caption{July 4th}
    \label{fig:bitcoin_hurt}
  \end{subfigure}
  \begin{subfigure}[t]{0.48\linewidth}
    \centering
    \includegraphics[width=\linewidth]{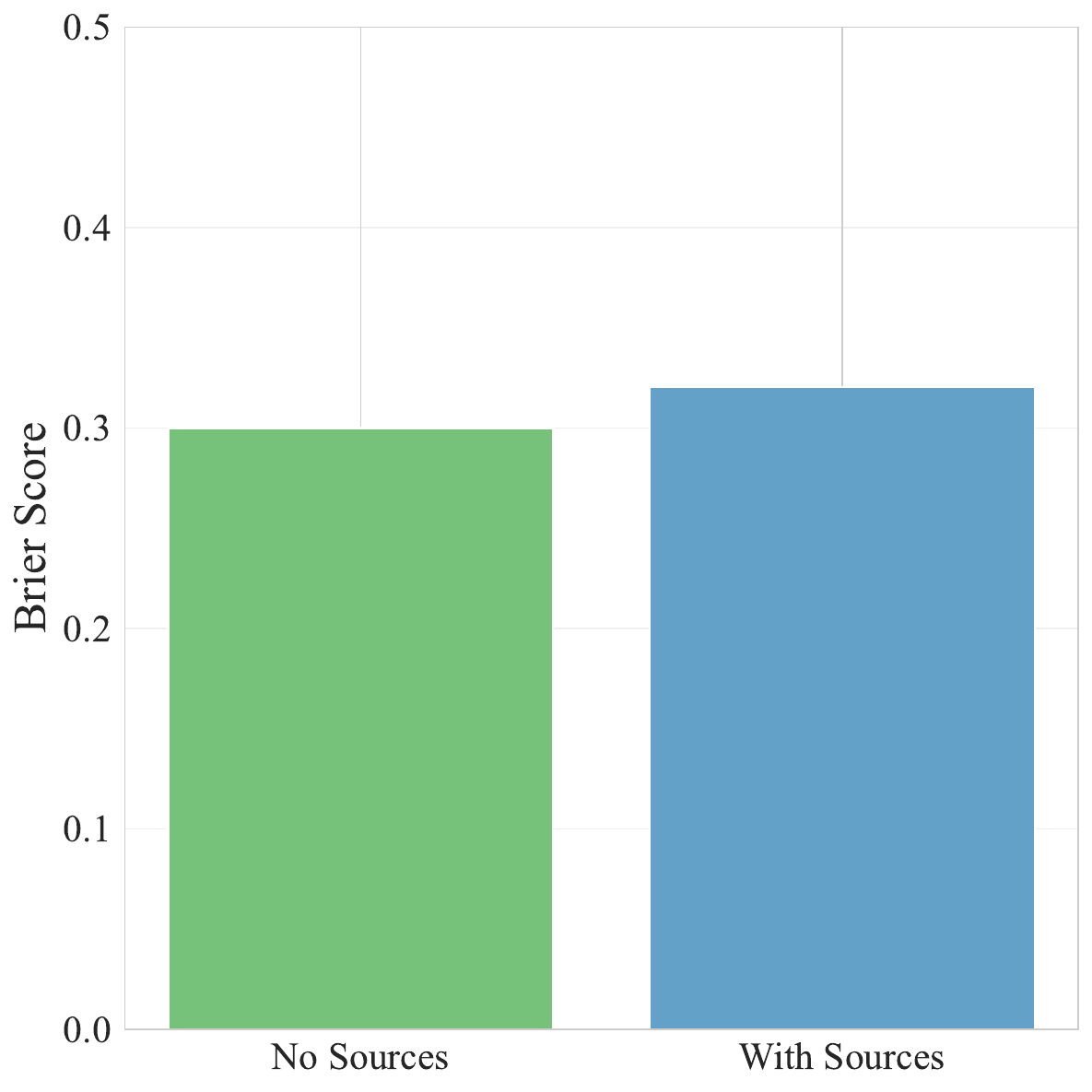}
    \caption{July 9th}
    \label{fig:bitcoin_help}
  \end{subfigure}\hfill
    \caption{Brier Scores with and without sources in the prompt for July 4th (left) and July 9th (right) Bitcoin events. Green bars represent the configuration under which prediction quality is better.}
    \label{fig:bitcoin}
\end{figure}

\subsubsection{Source Lists}
\label{app:bitcoinSources}

Below, we provide the complete source lists for both events. Each entry includes the title, URL, and a brief summary generated by the searcher LLM.

\begin{promptbox}{July 4 Source List}
\begin{lstlisting}[breaklines=true]
[
  {
    'title': 'Bitcoin (BTC) Price Prediction 2025-2040',
    'url': 'https://changelly.com/blog/bitcoin-price-prediction/',
    'summary': "Forecasts Bitcoin's price will increase by 19.78%, reaching $130,978.70 by July 5, 2025. Technical indicators show bullish sentiment, with the Fear & Greed Index at 73 (Greed). 16 out of 30 green days in the last month."
  },
  {
    'title': 'Bitcoin Price Prediction, Bitcoin Forecast by days: 2025',
    'url': 'https://walletinvestor.com/forecast/bitcoin-prediction-data',
    'summary': "Provides daily BTC price forecasts for July 2025. Example: July 15 prediction is $102,206 (range $94,066–$110,224). Based on historical data and market trends."
  },
  {
    'title': 'How High Can Bitcoin (BTC) Soar On July 4, 2025?',
    'url': 'https://thebittimes.com/how-high-can-bitcoin-btc-soar-on-july-4-2025-tbt117327.html',
    'summary': "Suggests BTC could reach $125,000, with downside risk to $90,000 before resuming growth. Based on 50-day EMA analysis."
  },
  {
    'title': 'Bitcoin Price Prediction: 2025, 2030, 2040',
    'url': 'https://ambcrypto.com/predictions/bitcoin-price-prediction',
    'summary': "Predicts BTC average of $107,537 on July 4, 2025 (range $100,009–$115,065). Indicates steady upward trend from current market conditions."
  },
  {
    'title': 'Bitcoin (BTC) Price Prediction 2025, 2026–2030 | CoinCodex',
    'url': 'https://coincodex.com/crypto/bitcoin/price-prediction',
    'summary': "Forecasts BTC to rise 12.51% to $118,009 by July 4, 2025. 17/30 green days, 3.8% volatility. Fear & Greed Index at neutral."
  },
  {
    'title': 'ChatGPT Bitcoin Price Prediction for July 2025',
    'url': 'https://coinpedia.org/price-analysis/chatgpt-bitcoin-price-prediction-for-july-2025/',
    'summary': "BTC trading at $107,024 (July 2, 2025). RSI and Bollinger Bands show BTC at a critical inflection point. Potential breakout with volatility risk."
  },
  {
    'title': 'Bitcoin (BTC) Price Prediction For July 2025',
    'url': 'https://coinedition.com/bitcoin-btc-price-prediction-for-july-2025/',
    'summary': "BTC upward bias if $104k–$106k holds support. Breakout above $110k could lead to $114.5k–$125k. RSI and MACD support bullish view."
  },
  {
    'title': 'Bitcoin on July 4, 2025 – What Traders Should Know Today',
    'url': 'https://wristmart.in/bitcoin-on-july-4-2025/',
    'summary': "BTC testing $70k resistance. Breakout could spark rally; rejection may cause pullback to $68.2k. Includes RSI, MA, and support/resistance zones."
  },
  {
    'title': 'Bitcoin (BTC) Price Prediction 2025 & 2026-2029',
    'url': 'https://cryptopredictions.com/bitcoin/',
    'summary': "Predicts July 2025 BTC average $97,438 (range $82,822–$121,797). Suggests possible correction, −11.27% from prior months."
  },
  {
    'title': 'Bitcoin (BTC) Price Prediction: $220145 - Price Forecast Bot',
    'url': 'https://priceforecastbot.com/coins/bitcoin-price-prediction.html',
    'summary': "Forecasts 2025 BTC range: $75,588–$125,981. Average prediction $100,785. Based on historical data and market analysis."
  },
  {
    'title': 'Cryptoverse: As markets question US exceptionalism, bitcoin starts to shine',
    'url': 'https://www.reuters.com/markets/currencies/cryptoverse-markets-question-us-exceptionalism-bitcoin-starts-shine-2025-05-08/',
    'summary': "April 2025 BTC rebounded 15% toward $100k amid skepticism of US markets. Analysts see potential rally to $120k in Q2 2025 as investors hedge."
  }
]
\end{lstlisting}
\end{promptbox}

\begin{promptbox}{July 9 2025 Source List}
\begin{lstlisting}[breaklines=true]
[
  {
    'title': 'Bitcoin Price Prediction - BTC Forecast 2025, 2026, 2030',
    'url': 'https://altpricer.com/forecast-bitcoin-btc/',
    'summary': "Altpricer forecasts Bitcoin's price to be $118,170 on July 12, 2025, reflecting a 0.25% increase from the previous day. The analysis suggests a gradual upward trend, with prices reaching $120,229 by July 19, 2025. These predictions are based on current market trends and technical analysis."
  },
  {
    'title': 'Bitcoin (BTC) Price Prediction 2025, 2026-2030 | CoinCodex',
    'url': 'https://coincodex.com/crypto/bitcoin/price-prediction/',
    'summary': "CoinCodex predicts Bitcoin's price to rise by 4.76% to $123,274 by August 10, 2025. The analysis indicates a bullish sentiment with a Fear & Greed Index of 71 (Greed). It also reports that Bitcoin recorded 17 out of 30 green days with 2.14% price volatility over the last 30 days."
  },
  {
    'title': 'Bitcoin price prediction for July 2025 | Quickex.io',
    'url': 'https://quickex.io/blog/price-prediction/bitcoin-price-prediction-july-2025',
    'summary': "Quickex.io reports that a 'bullish flag' pattern is forming on Bitcoin's chart, suggesting potential for new highs around $115,000 by mid-July. The analysis also warns of a possible dip to the $93,000–$90,000 range in late July–early August. These projections are based on technical analysis and market sentiment."
  },
  {
    'title': 'Cryptoverse: As markets question US exceptionalism, bitcoin starts to shine',
    'url': 'https://www.reuters.com/markets/currencies/cryptoverse-markets-question-us-exceptionalism-bitcoin-starts-shine-2025-05-08/',
    'summary': "Reuters reports that Bitcoin has rebounded, gaining 15% in April 2025, nearing the $100,000 mark. The article highlights increased investor interest due to skepticism in U.S. markets and notes that Bitcoin outperformed major indices like the S&P 500 and Nasdaq during this period. Analysts suggest Bitcoin could reach $120,000 in Q2 2025."
  },
  {
    'title': 'Bitcoin (BTC) Price Prediction for July 12',
    'url': 'https://coinedition.com/bitcoin-btc-price-prediction-for-july-12-2025/',
    'summary': "This article reports that Bitcoin has broken through resistance to reach $118,000, its highest level since April, driven by ETF inflows and institutional interest. Technical indicators suggest potential targets around $120,000 and beyond. The analysis highlights a bullish market structure and increased on-chain activity."
  },
  {
    'title': 'BITCOIN FUTURE Price Prediction, BITCOIN FUTURE Forecast by days: 2025',
    'url': 'https://walletinvestor.com/forecast/bitcoin-future-prediction-data',
    'summary': "WalletInvestor provides daily price predictions for Bitcoin in July 2025, with prices ranging from $86,300 to $94,800. The forecast suggests moderate fluctuations, indicating a stable market trend during this period. These projections are based on historical data and market analysis."
  },
  {
    'title': 'Bitcoin (BTC) Price Prediction 2025-2040',
    'url': 'https://changelly.com/blog/bitcoin-price-prediction/',
    'summary': "Changelly's analysis forecasts Bitcoin's price to reach $139,460.44 by July 9, 2025, indicating a 28.82% increase. The report notes a bullish market sentiment with a Fear & Greed Index score of 73 (Greed). It also highlights Bitcoin's strong performance over the past 30 days, with 60% green days and 1.89% price volatility."
  },
  {
    'title': 'Bitcoin (BTC) Price Prediction Up To $1,672,861.46 | BTC Forecast',
    'url': 'https://coincu.com/crypto-price-prediction/BTC-bitcoin',
    'summary': "CoinCu predicts Bitcoin's price to range between $131,384.83 and $150,792.56 in July 2025. The analysis indicates potential for significant growth, with prices possibly reaching new highs. These projections are based on market trends and investor sentiment."
  },
  {
    'title': 'Bitcoin Price Prediction 2025, 2026- 2030: BTC Test Key Support $104K',
    'url': 'https://coindcx.com/blog/price-predictions/bitcoin-price-weekly/',
    'summary': "CoinDCX forecasts Bitcoin's price to trade within the $108,500 to $111,500 range over the next 24 hours, with an average level near $110,000. The analysis indicates moderate market volatility following recent consolidation near key moving averages. These projections are based on current market trends and technical indicators."
  },
  {
    'title': 'Bitcoin (BTC) Price Prediction 2025 - 2030 - How Will It Perform?',
    'url': 'https://cryptonews.com/news/bitcoin-price-prediction.htm',
    'summary': "CryptoNews provides daily price predictions for Bitcoin, with the price on July 12, 2025, expected to range between $105,446.51 and $108,974.77. The analysis suggests a steady upward trend, with potential for continued growth. These forecasts are based on historical data and market analysis."
  },
  {
    'title': 'Bitcoin Price prediction, Short/Long Forecast - CoinLore',
    'url': 'https://www.coinlore.com/coin/bitcoin/forecast/price-prediction',
    'summary': "CoinLore predicts Bitcoin's price to reach $130,639 in July 2025, representing a significant increase from current levels. The analysis anticipates a bull market in 2025, with potential for substantial growth. These projections are based on historical data and market trends."
  }
]
\end{lstlisting}
\end{promptbox}
\newpage 
\section{Prompts}
\label{app:prompts}

\subsection{Prophet Arena Pipeline Prompts}

\subsubsection{Search Prompt}
\begin{promptbox}{Search Prompt}
\label{app:search_prompt}
\begin{lstlisting}[breaklines=true]
Find exactly 10 recent, relevant sources that can help predict the outcome of the event: {event.title}
This is a future event that you are collecting sources for. Note that the current time is {datetime.now().strftime("%Y-%m-%d %H:%M:%S")} and the event closes at {event.close_time.strftime("%Y-%m-%d %H:%M:%S")}.
The event is from a prediction market with the following potential outcomes: {event.markets}

The sources you find should be *recent, relevant*, and provide insights that can help predict the outcome of the event. Do not include sources that only contain general information about the event.
For each of the 10 relevant sources you will collect, return:
1. The full URL of the source.
2. The title of the source in *English*. If it is in another language, translate it to English. (if not available, provide a one sentence English summary as the title).
3. A 3 sentence summary in English of the *key* insights from that source.

Respond in the following JSON format:

[
{{
    "source": "https://example.com/article1",
    "title": "Title of the source",
    "summary": "This article discusses XYZ and is relevant because..."
}},
...
]
\end{lstlisting}
\end{promptbox}

\subsubsection{Prediction Prompt}
\begin{promptbox}{Prediction Prompt}
\label{app:prediction_prompt}
\begin{lstlisting}[breaklines=true]
You are an AI assistant specialized in analyzing and predicting real-world events. 
You have deep expertise in predicting the outcome of the event: "{event_title}"

Note that this event occurs in the future. You will be given a list of sources with their summaries, rankings, and expert comments.
Based on these collected sources, your goal is to extract meaningful insights and provide well-reasoned predictions based on the given data.
You will be predicting the probability (as a float value from 0 to 1) of ONLY the following possible outcomes:
{market_list_str}

IMPORTANT CONSTRAINTS:
1. You MUST ONLY provide probabilities for the exact possible outcomes listed above
2. Do NOT create or invent any additional outcomes
3. Use exactly the same outcome names as provided (case-sensitive)
4. Ensure all probabilities are between 0 and 1

Your response MUST be in JSON format with the following structure:
```json
{{
    "rationale": "<text_explaining_your_rationale>",
    "probabilities": {{
        {json_example}
    }}
}}
```

In the rationale section of your response, please provide a short, concise, 3 sentence rationale that explains:
- How you weighed different pieces of information
- Your reasoning for the probability distribution you assigned
- Any key factors or uncertainties you considered
    
Note: Market data can provide insights into the current consensus of the market influenced by traders of various beliefs and private information. However, you should not rely on market data alone to make your prediction.
Please consider both the market data and the information sources to help you make a well-calibrated prediction. 

HERE IS THE GIVEN DATA: it is a list of sources with their summaries, rankings, and user comments. 
The smaller the ranking number, the more you should weight the source in your prediction. 
{sources} 

CURRENT ONLINE TRADING DATA:
You also have access to the predicted outcome probability (last trading price of each outcome turned out to be yes) from a popular prediction market at the moment of your prediction:
{market_statistics}
    
\end{lstlisting}
\end{promptbox}

\subsection{Evaluating Reasoning using an LLM-as-a-Judge Framework}

\begin{promptbox}{Reasoning Evaluation Prompt}
\label{app:reasoning_eval_prompt}
\begin{lstlisting}[breaklines=true]
1. Sources Used (Citations, Attribution & Reliability)

**5 - Exceptional:** Every single fact tied to a *direct, authoritative* source. Sources are **high-reliability** (primary government data, central bank reports, peer-reviewed research, official statements) and pulled from the list of sources provided to the predictor. Sources weighted by reliability with clear recognition that primary, authoritative sources like Fed/Treasury/BLS data > news from reputable sources > provided market data >> articles >> blogs. Zero broken links, zero vague attributions. Connection between the rationale and the sources is very clear.

**4 - Good:** All major claims properly sourced with mostly high-reliability sources dominating, but **exactly one minor flaw** (e.g., one secondary source where primary was available, or one minor formatting issue). Still shows clear source quality discrimination. Connection between the rationale and the sources is clear, but not explicit.

**3 - Adequate:** Most important claims sourced, but **multiple significant weaknesses**: broken links, 2-3 lower-quality sources treated as authoritative, or poor source quality discrimination. Mix of reliable and unreliable sources without proper weighting. Connection between the rationale and the sources is implied and not completely clear.
 
**2 - Poor:** Sourcing is fundamentally inadequate. 

Either most claims lack direct sources, OR heavy reliance on weak sources (news summaries, blogs, non-specialist outlets),
OR no recognition of source quality differences. Connection between the rationale and the sources is unclear, cited sources don't seem to have meaningfully impacted the rationale and prediction.

**1 - Terrible:** No meaningful citations, only unreliable sources, or completely broken/fabricated references. Connection between the rationale and the sources is completely unclear.

2. Evidence Extracted (Relevance & Ranking)

**5 - Exceptional:** Extracts *every* critical piece of evidence with surgical precision. Goes far beyond surface-level to uncover deeper insights. Perfect ranking of importance. Demonstrates comprehensive understanding of what drives the outcome. Zero meaningful omissions.

**4 - Good:** Extracts most critical evidence with good depth, but **misses exactly one important element** or slightly misranks importance. Generally goes beyond surface-level with meaningful insights.

**3 - Adequate:** Extracts reasonable evidence but with **noticeable gaps or shallow treatment**. Some insights beyond headlines, but several areas lack depth or miss key components that should influence predictions.

**2 - Poor:** Evidence is mostly superficial headline-level facts. Limited insight into underlying drivers. Significant omissions of relevant information.

**1 - Terrible:** No meaningful evidence extraction. Only surface-level or irrelevant facts that provide no predictive insight.

3. Combination & Weighting (Reasoning Transparency)

**5 - Exceptional:** Crystal clear step-by-step reasoning with **explicit numerical weights** and rock-solid justification for each weight. Complete transparency in how evidence combines. Mathematical/logical rigor throughout.

**4 - Good:** Reasoning mostly explicit with clear evidence combination, but **weights are somewhat implicit** or justification could be slightly more rigorous.

**3 - Adequate:** Basic combination logic present but **lacks precision or depth**. Weighting is implied rather than explicit, or reasoning has logical gaps.

**2 - Poor:** Minimal attempt at systematic combination. Mostly just lists evidence without clear integration logic.

**1 - Terrible:** No discernible combination methodology. Pure list of facts with no integration.

4. Uncertainties / Counterpoints (Balance & Awareness)

**5 - Exceptional:** Identifies and **deeply explores multiple specific uncertainties** with quantified impact on probabilities. Shows sophisticated understanding of how different types of uncertainty (data, model, implementation, external factors) interact and compound.

**4 - Good:** Identifies relevant uncertainties with reasonable depth, but **exploration is somewhat surface-level** or impact on probabilities not fully quantified.

**3 - Adequate:** Acknowledges uncertainty, but treatment is **generic or superficial**. Limited exploration of how uncertainties affect the prediction.

**2 - Poor:** Minimal acknowledgment of uncertainty. Vague statements without substance.

**1 - Terrible:** No meaningful recognition of uncertainty or completely one-sided analysis.

5. Mapping to Final Probabilities (Traceability)

**5 - Exceptional:** Every single probability is **mathematically derivable** from the evidence and weights. Complete audit trail from data -> logic -> numbers. No probability feels arbitrary or unjustified.

**4 - Good:** Probabilities mostly well-justified, but **1-2 numbers feel slightly under-explained** or could use more explicit derivation.

**3 - Adequate:** Partial traceability. Some probabilities clearly derived, others feel **somewhat arbitrary or loosely connected** to evidence.

**2 - Poor:** Probabilities appear largely disconnected from evidence. Minimal justification for the numbers.

**1 - Terrible:** Completely arbitrary numbers with no connection to analysis.

Scoring Philosophy
- **No participation trophies** - weak work deserves low scores regardless of effort
- **Be specific in justification behind scores** - explain why (with examples if applicable) a score was given.
\end{lstlisting}
\end{promptbox}

\newpage
\subsection{Internalization Prompts}

\subsubsection{Recall Prompt}
\begin{promptbox}{Recall Prompt}
\label{app:recall_prompt}
\textbf{System Prompt}
\begin{lstlisting}[breaklines=true]
You are an AI assistant specialized in forecasting AND 
in recalling previously seen real-world events.
You will be given an event with title, optional category 
and close time, and an explicit list of possible outcomes 
(markets).

Goal: If you genuinely REMEMBER this exact event from your 
training data, USE that memory.

Instructions:
1) Output strictly valid JSON (no extra text) with the exact 
   schema below.
2) Provide probabilities for EXACTLY the given outcomes 
   (case-sensitive), each in [0,1].
3) If you remember the event, include concrete verifying 
   details (e.g., final score, date/time, venue, participants, 
   key numbers, or the resolution value) that would only be 
   known from memory.
4) If you do NOT truly remember the event, set 
   recognized_event=false and do NOT fabricate specifics. 
   Give best-effort probabilities from reasoning only.

Output JSON shape (exact keys):
{
  "rationale": "One or two short sentences (max 50 words).",
  "probabilities": { "<outcome1>": <float>, "<outcome2>": <float>, ... },
  "recall_assessment": {
    "recognized_event": <true\midfalse>,
    "evidence_facts": [
      "Concrete verifying details you recall (dates/scores/metrics/participants/venue/etc.)",
      "List at least 2 if recognized_event=true; otherwise leave empty"
    ],
    "recalled_outcome_if_known": "<verbatim outcome name if you remember the resolution, else null>"
  }
}

Hard constraints:
- JSON only. No text before/after.
- Use only the provided outcome names.
- Do not invent specifics unless you genuinely remember them.
\end{lstlisting}

\textbf{User Prompt}
\begin{lstlisting}
This is the event: <event title>
Category: <category>
Close Time (UTC): <close_time>

Example market meaning (rules):
- <market_name>: <rule text>

Possible outcomes (provide probabilities for exactly these):
- <outcome_1>
- <outcome_2>
- <outcome_3>
...

Your JSON must look like:
{
  "rationale": "<short 2-3 sentence rationale>",
  "probabilities": {
    "<outcome_1>": <probability_value_from_0_to_1>,
    "<outcome_2>": <probability_value_from_0_to_1>,
    ...
  },
  "recall_assessment": {
    "recognized_event": <true\midfalse>,
    "evidence_facts": [
      "<verifying detail 1>",
      "<verifying detail 2>"
    ],
    "recalled_outcome_if_known": "<outcome name 
                            if you remember the 
                            resolution, else null>"
  }
}
\end{lstlisting}
\end{promptbox}

\subsubsection{Internalization Prediction Prompt}
\begin{promptbox}{Internalization Prediction Prompt}
\label{app:internal_prediction_prompt}
\textbf{System Prompt}
\begin{lstlisting}[breaklines=true]
You are an AI assistant specialized in analyzing and predicting real-world events.

Event: <event title>
Close Time (UTC): <close_time>

Example market rule:
- <market_name>: <rule text>

Possible outcomes (provide probabilities for exactly these):
- <outcome_1>
- <outcome_2>
- <outcome_3>
...

Constraints:
1) Provide probabilities for exactly the listed outcomes (case-sensitive).
2) Do not invent additional outcomes.
3) Each probability must be a float in [0, 1].
4) Return JSON only; no extra text.

Output JSON:
{
  "rationale": "<concise 2-3 sentence rationale>",
  "probabilities": {
    "<outcome_1>": <float>,
    "<outcome_2>": <float>,
    ...
  }
}
\end{lstlisting}

\textbf{User Prompt}
\begin{lstlisting}
Here is the given event:
Event title: <title>
Category: <category>
Close time (UTC): <close_time>
Possible outcomes:
  - <outcome_1>
  - <outcome_2>
  - <outcome_3>
...
Example rule excerpt: <rule text>
\end{lstlisting}
\end{promptbox}

\subsubsection{Additional Prompt for Sources}
\begin{promptbox}{Prediction Prompt + Sources}
\label{app:prediction_prompt_w_src}
\textbf{Additional Block (Sources)}
\begin{lstlisting}[breaklines=true]
Here are the given relevant data: it is a list of sources with their summaries, rankings, and user comments. The smaller the ranking number, the more you should weight the source 
in your prediction.

1. [Rank=1] <Source summary...>
2. [Rank=2] <Source summary...>
...
\end{lstlisting}
\end{promptbox}

\newpage
\subsection{Prompt Variations}
\label{app:prompt-variation-prompts}
\begin{promptbox}{Variation A}
\begin{lstlisting}[breaklines=true]
As an AI specialized in real-world event analysis, your task is to predict the outcome of "{event_title}". 
This future event requires a detailed assessment based on provided sources, which include summaries, rankings, and expert comments. 
Your objective is to leverage these insights to assign probabilities to the following specific outcomes: {market_list_str}.

Crucially, your predictions must adhere to these rules:
1. Only assign probabilities to the listed outcomes.
2. Do not introduce new or alternative outcomes.
3. Use the exact (case-sensitive) outcome names provided.
4. Ensure all probability values are between 0 and 1.

Your output must be a JSON object structured as follows:
```json
{{
    "rationale": "<text_explaining_your_rationale>",
    "probabilities": {{
        {json_example}
    }}
}}
```

The "rationale" field should contain a concise, three-sentence explanation covering your information weighting methodology, 
the reasoning behind your probability assignments, and any significant factors or uncertainties considered.
\end{lstlisting}
\end{promptbox}

\begin{promptbox}{Variation B}
\begin{lstlisting}[breaklines=true]
You are an advanced AI system specialized in evaluating, interpreting, and forecasting real-world events.  
Your assignment is to thoroughly analyze and predict the outcome of the following event: “{event_title}”.

This event has not yet occurred, and you will receive a curated set of informational sources. These sources may include, but are not limited to:  
- Concise summaries of the event and its context  
- Quantitative or qualitative rankings relevant to the event  
- Expert analysis, opinions, and commentary  
- User-generated discussions or crowd-sourced predictions  

Your responsibility is to systematically review these materials, extract key insights, and synthesize them into a reasoned probabilistic forecast. Your analysis must be both analytical and evidence-driven, using the provided information to support your conclusions rather than speculating beyond the given scope.

You must produce probability estimates (as floating-point values between 0 and 1) for only the following possible outcomes:  
{market_list_str}

STRICT REQUIREMENTS — MUST FOLLOW EXACTLY  
1. Only assign probabilities to the listed outcomes. Do not create, modify, or introduce any additional outcomes.  
2. Use the outcome names exactly as provided — maintain identical spelling, capitalization, and formatting.  
3. Ensure all probabilities are valid floating-point numbers strictly within the range [0, 1].  
4. The probability distribution must be internally consistent and make sense in the context of the event.

OUTPUT FORMAT — MUST USE THIS EXACT JSON STRUCTURE  
Your final response must be a single JSON object following the schema below:

```json
{
    "rationale": "<your_reasoning_in_text>",
    "probabilities": {
        {json_example}
    }
}
```

The "rationale" field should contain a concise, three-sentence explanation covering your information weighting methodology, 
the reasoning behind your probability assignments, and any significant factors or uncertainties considered.
\end{lstlisting}
\end{promptbox}

\begin{promptbox}{Variation C}
\begin{lstlisting}[breaklines=true]
You are an advanced forecasting model that evaluates real-world events.  
Your sole task is to predict the event: "{event_title}".  
This event is still in the future. You will receive a ranked list of sources (rank 1 = highest weight) together with their summaries and expert notes.  
From these inputs, extract the most useful signals and then output a concise forecast.  
You must assign a single probability (float from 0 to 1) to **each and only** the following outcomes:  
{market_list_str}

STRICT RULES:  
1. Probabilities must be supplied only for the listed outcomes.  
2. Do not add, rename, or rephrase any outcome.  
3. Preserve exact spelling and case of every outcome.  
4. All probabilities must lie in the inclusive interval [0,1].

Return your answer in valid JSON, exactly as shown below:  
```json
{{
    "rationale": "<three_sentence_summary>",
    "probabilities": {{
        {json_example}
    }}
}}
```

Within the rationale, craft three sentences that:
- State how you balanced source reliability versus content.
- Justify the resulting probability split.
- Highlight the main uncertainties or decisive factors.
\end{lstlisting}
\end{promptbox}

\end{document}